\RequirePackage{fix-cm}
\documentclass[smallextended]{svjour3}       
\smartqed  
\usepackage{graphicx}

\usepackage{braket,amsfonts,dsfont}

\usepackage{booktabs}
\usepackage{threeparttable}
\usepackage{mathtools}
\usepackage{tablefootnote}

\usepackage{amssymb}

\usepackage{graphicx,epstopdf}

\usepackage{amsopn}

\usepackage{xspace}
\usepackage{bold-extra}
\usepackage[most]{tcolorbox}
\usepackage{url}

\colorlet{texcscolor}{blue!50!black}
\colorlet{texemcolor}{red!70!black}
\colorlet{texpreamble}{red!70!black}
\colorlet{codebackground}{black!25!white!25}

\usepackage{bm}

\usepackage{algorithm}
\usepackage[noend]{algpseudocode}
\usepackage{enumitem}

\usepackage{wrapfig}

\def\va{{\bm{a}}}
\def\vb{{\bm{b}}}
\def\vc{{\bm{c}}}

\def\ve{{\bm{e}}}

\def\vh{{\bm{h}}}
\def\vi{{\bm{i}}}

\def\vk{{\bm{k}}}

\def\vm{{\bm{m}}}

\def\vo{{\bm{o}}}
\def\vp{{\bm{p}}}
\def\vq{{\bm{q}}}
\def\vr{{\bm{r}}}
\def\vs{{\bm{s}}}

\def\vu{{\bm{u}}}
\def\vv{{\bm{v}}}

\def\vx{{\bm{x}}}
\def\vy{{\bm{y}}}
\def\vz{{\bm{z}}}

\def\mA{{\bm{A}}}

\def\mD{{\bm{D}}}

\def\mF{{\bm{F}}}

\def\mH{{\bm{H}}}
\def\mI{{\bm{I}}}
\def\mJ{{\bm{J}}}
\def\mK{{\bm{K}}}

\def\mM{{\bm{M}}}

\def\mP{{\bm{P}}}
\def\mQ{{\bm{Q}}}

\def\mT{{\bm{T}}}
\def\mU{{\bm{U}}}
\def\mV{{\bm{V}}}
\def\mW{{\bm{W}}}
\def\mX{{\bm{X}}}

\def \Ib {{\mathbf{Ib}}}

\def \RR {{\mathbb{R}}}

\def \Ib {{\mathbf{I}}}

\def \Ub {{\mathbf{U}}}

\newcommand{\norm}[1]{\left\lVert#1\right\rVert}

\newcommand{\lstm}{\textsc{lstm}}

\usepackage{soul}

\begin{document}

\title{How Does Momentum Benefit Deep Neural Networks Architecture Design? A Few Case Studies
}


\author{
        Bao Wang \and 
        Hedi Xia \and
        Tan Nguyen \and
        Stanley Osher
}


\institute{
Bao Wang \at
        Department of Mathematics\\
        Scientific Computing and Imaging Institute, University of Utah \\
              \email{wangbaonj@gmail.com}
        \and
Hedi Xia \at
        Department of Mathematics, UCLA\\
              \email{hedixia@math.ucla.edu}
        \and
Tan Nguyen \at
        Department of Mathematics, UCLA\\
              \email{tannguyen89@math.ucla.edu}
        \and
Stanley Osher \at
        Department of Mathematics, UCLA \\
              \email{sjo@math.ucla.edu}           
}

\date{Received: date / Accepted: date}

\maketitle

\begin{abstract}
We present and review an algorithmic and theoretical framework for improving neural network architecture design via momentum. As case studies, we consider how momentum can improve the architecture design for recurrent neural networks (RNNs), neural ordinary differential equations (ODEs), and transformers. We show that integrating momentum into neural network architectures has several remarkable theoretical and empirical benefits, including 1) integrating momentum into RNNs and neural ODEs can overcome the vanishing gradient issues in training RNNs and neural ODEs, resulting in effective learning long-term dependencies. 2) momentum in neural ODEs can reduce the stiffness of the ODE dynamics, which significantly enhances the computational efficiency in training and testing. 3) momentum can improve the efficiency and accuracy of transformers.



\end{abstract}

\section{Introduction}\label{section-introduction}


Deep learning has radically advanced artificial intelligence \cite{lecun2015deep}, which has achieved state-of-the-art performance in various applications, including computer vision \cite{dosovitskiy2020image,touvron2020deit}, natural language processing \cite{devlin2018bert,NEURIPS2020_1457c0d6}, and control \cite{silver2017mastering}. Nevertheless, deep neural networks (DNNs) designs are mostly heuristic and the resulting architectures have many well-known issues: 1) convolutional neural networks (CNNs) are not robust to unperceptible 
adversarial attacks \cite{szegedy2013intriguing}; 2) recurrent neural networks (RNNs) cannot learn long-term dependencies effectively due to vanishing gradients \cite{pascanu2013difficulty}; 3) training neural ordinary differential equations (ODEs) can take an excessive number of function evaluations (NFEs) \cite{chen2018neural}; and 4) training transformers suffers from quadratic computational time and memory costs \cite{vaswani2017attention}. See Sections~\ref{sec-RNN}, \ref{sec-neural-ODE}, and \ref{sec-transformers}, respectively, for the details of these problems 2)-4).

Addressing the above grand challenges is at the forefront of deep learning research. 1) Adversarial defense \cite{Madry:2018} has been proposed to train robust neural networks against adversarial attacks; a survey of adversarial defense algorithms is available, see, e.g., \cite{chakraborty2018adversarial}. Training ResNets has also been interpreted as solving a control problem of the transport equation \cite{BaoWang:2018NIPS,BaoWang:2019NIPS}, resulting in PDE-motivated adversarial defenses \cite{BaoWang:2019NIPS,1930-8337_2021_1_129,wang_osher_2021}. 2) Learning long-term dependencies with improved RNNs has been an active research area for several decades; celebrated works include long short-term memory \cite{hochreiter1997long} and gated recurrent unit \cite{chung2015gated}. 3) Several algorithms and techniques have been proposed to reduce NFEs in training neural ODEs, including input augmentation \cite{NEURIPS2019_21be9a4b}, regularizing solvers and learning dynamics 
\cite{pmlr-v119-finlay20a,NEURIPS2020_2e255d2d,NEURIPS2020_a9e18cb5,pmlr-v139-pal21a},
high-order ODE \cite{norcliffe2020_sonode}, data control \cite{massaroli2020dissecting}, and depth-variance \cite{massaroli2020dissecting}. 4) Transformers are the current state-of-the-art machine learning (ML) models for sequential learning \cite{vaswani2017attention}, which processes the input sequence concurrently and can learn long-term dependencies effectively. However, transformers suffer from quadratic computational time and memory costs with respect to the input sequence length; see Section~\ref{sec-transformers} for details. In response, efficient attention has been proposed leveraging sparse and low-rank approximation of the attention matrix \cite{j.2018generating,pmlr-v80-parmar18a,beltagy2020longformer,ainslie-etal-2020-etc,zaheer2021big,wang2020linformer,katharopoulos2020transformers,choromanski2021rethinking}, locality-sensitive hashing \cite{Kitaev2020Reformer}, clustered attention \cite{vyas2020fast}, and decomposed near-field and far-field attention \cite{nguyen2021fmmformer}.


\subsection{Background: Momentum acceleration for gradient descent
}\label{subsec:review:momentum}

Let us recall heavy-ball momentum, a.k.a. classical momentum \cite{polyak1964some}, for accelerating gradient descent in solving $\min_{\vx\in \RR^d}F(\vx)$. Starting from $\vx^0$ and $\vx^1$, the heavy-ball method iterates as follows
\begin{equation}\label{eq:HB1}
\vx^{k+1} = \vx^k-s\nabla F(\vx^k) + \beta(\vx^k-\vx^{k-1}),
\end{equation}
where $s>0$ is the step size and $0\leq \beta <1$ is the momentum parameter. By introducing the momentum state $\vm$, we can rewrite the HB method as
\begin{equation}\label{eq:HB2}
\begin{aligned}
{\vm}^{k+1} = \beta{\vm}^k+\nabla F(\vx^k);\quad
{\vx}^{k+1} = {\vx}^k-s{\vm}^{k+1}.
\end{aligned}
\end{equation}
In contrast, gradient descent updates at each step as follows
\begin{equation}\label{eq:GD}
\vx^{k+1} = \vx^k-s \nabla F(\vx^k).
\end{equation}

\subsection{Contribution}
This paper aims to present and review an algorithmic and theoretical framework for improving neural network architecture design via momentum, a well-established first-order optimization tool \cite{polyak1964some,nesterov1983method}. In particular, we focus on leveraging momentum to design new RNNs and neural ODEs to accelerate their training and testing and improve learning long-term dependencies with theoretical guarantees. Moreover, we present a new efficient attention mechanism with momentum augmentation, which significantly improves computational efficiency over transformers \cite{vaswani2017attention} and accuracy over linear transformers \cite{katharopoulos2020transformers}. Finally, we present some perspectives of how momentum can further improve neural networks design and solve existing grand challenges.

\subsection{Notations}
We denote scalars by lower- or upper-case letters. We also denote vectors and matrices by lower- and upper-case boldface letters, respectively. For a vector $\vx = (x_1, \cdots, x_d)^\top\in \mathbb{R}^d$, where $(x_1,\cdots,x_d)^\top$ denotes the transpose of the vector $(x_1,\cdots,x_d)$, we use $\|\vx\| = {(\sum_{i=1}^d |x_i|^2)^{1/2}}$ to denote its $\ell_2$ norm. We denote the vector whose entries are all 0s as $\mathbf{0}$. For a matrix $\mA$, we use $\mA^\top$,  $\mA^{-1}$, and $\|\mA\|$ to denote its transpose, inverse, and spectral norm, respectively. We use $\mI$ to denote the identity matrix, whose dimension can be determined in its context.
For a function $f(\vx): \mathbb{R}^d \rightarrow \mathbb{R}$, we denote its gradient as $\nabla f(\vx)$. Given two sequences $\{a_n\}$ and $\{b_n\}$, we write $a_n=\mathcal{O}(b_n)$ if there exists a positive constant $0<C<+\infty$ such that $a_n \leq C b_n$.

\subsection{Organization}

We organize this paper as follows: In Section~\ref{sec-RNN}, we present RNN models and their difficulties in learning long-term dependencies. We also show how to integrate momentum into RNNs to accelerate training RNNs and enhance RNNs' capability in learning long-term dependencies. In Section~\ref{sec-neural-ODE}, we show how the ODE limit of momentum can improve neural ODEs in terms of training and test efficiency and learning long-term dependencies. In Section~\ref{sec-transformers}, we show how momentum can be integrated into efficient transformers and improve their accuracy. We conclude and present potential new directions in Section~\ref{sec-conclusion}.


\section{Recurrent Neural Networks}\label{sec-RNN}
In this section, we present how to leverage momentum to improve RNN architecture design. The main results have been presented at NeurIPS 2020 \cite{MomentumRNN}.

\subsection{Recap on RNNs and LSTM}
Recurrent cells are the building blocks of RNNs. A recurrent cell employs a cyclic connection to update the current hidden state ($\vh_t$) using the past hidden state ($\vh_{t-1}$) and the current input data ($\vx_t$) \cite{elman1990finding}; the dependence of $\vh_t$ on $\vh_{t-1}$ and  $\vx_t$ in a recurrent cell can be written as
\begin{equation}\label{eq:RNN:Cell}
\vh_t = \sigma(\mU\vh_{t-1} + \mW\vx_t + \vb),\ \vx_t \in \RR^d,\ \mbox{and}\  \vh_{t-1}, \vh_t \in \RR^h,\ \ t=1, 2, \cdots, T,
\end{equation}
where  $\mU \in \RR^{h\times h}, \mW \in \RR^{h\times d}$, and $\vb\in \RR^h$ are trainable parameters; $\sigma(\cdot)$ is a nonlinear activation function, e.g., sigmoid or hyperbolic tangent; see Fig.~\ref{fig:momentumrnn-vs-rnn} (left) for an illustration of the RNN cell. Error backpropagation through time is used to train RNN, but it tends to result in exploding or vanishing gradients \cite{bengio1994learning}. Thus RNNs may fail to learn long term dependencies. 

LSTM cells augment the recurrent cell with ``gates'' \cite{hochreiter1997long} and results in
\begin{equation}\label{eq:LSTM:cell}
\begin{aligned}
\vi_t &= \sigma(\mU_{ih}\vh_{t-1} + \mW_{ix}\vx_t + \vb_i),\ \ &(\vi_t: \mbox{input gate})\\
\widetilde{\vc}_t &= \tanh{(\mU_{\widetilde{c}h}\vh_{t-1} + \mW_{\widetilde{c}x}\vx_t + \vb_{\widetilde{c}})},\ \ &(\widetilde{\vc}_t: \mbox{cell input})\\
\vc_t &= \vc_{t-1} + \vi_t \odot \widetilde{\vc}_t,\ \ &(\vc_t: \mbox{cell state})\\
\vo_t &= \sigma(\mU_{oh}\vh_{t-1} + \mW_{ox}\vx_t + \vb_o ),\ \ &(\vo_t: \mbox{output gate})\\
\vh_t &= \vo_t\odot \tanh{\vc_t},\ \ &(\vh_t: \mbox{hidden state})
\end{aligned}
\end{equation}
where $\mU_*\in \RR^{h\times h}$, $\mW_*\in \RR^{h\times d}$, and $\vb_* \in \RR^h$ are learnable parameters, and $\odot$ denotes the Hadamard product. The input gate decides what new information to be stored in the cell state, and the output gate decides what information to output based on the cell state value. The gating mechanism in LSTMs can lead to the issue of saturation \cite{van2018unreasonable,chandar2019towards}.

\begin{figure}[t]
\centering
\includegraphics[width=1.0\linewidth]{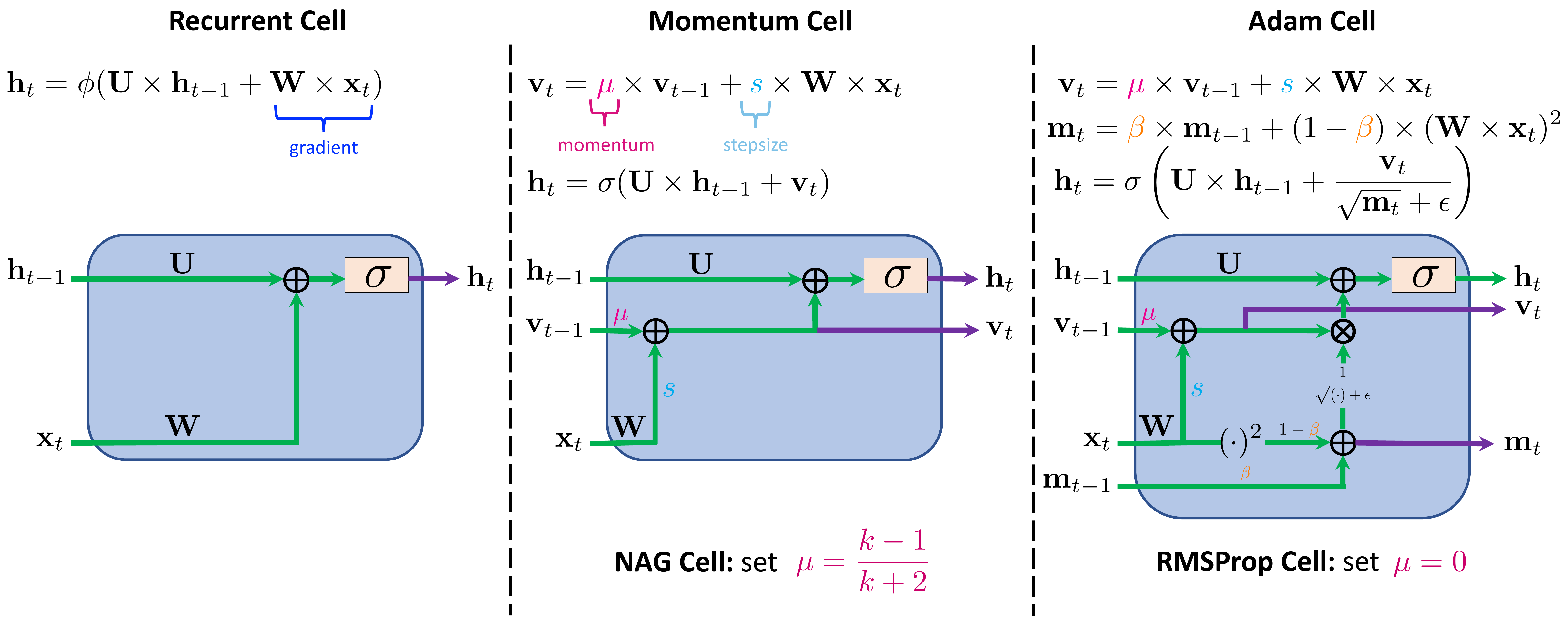}
\vskip -0.35cm
\caption{Illustration of the recurrent cell (left), Momentum/NAG cell (middle), and Adam/RMSProp cell (right). We draw a connection between the dynamics of hidden states in the recurrent cell and GD. We then introduce momentum to recurrent cell as an analogy of the momentum accelerated GD.
}
\label{fig:momentumrnn-vs-rnn}
\end{figure}

\subsection{Gradient descent analogy for RNN and MomentumRNN}\label{subsec:momentum-cell}
Now, we are going to establish a connection between RNN and gradient descent, and further leverage momentum to improve RNNs. Let $\widetilde{\mW} = [\mW, \vb]$ and $\widetilde{\vx}_t = [\vx_t, 1]^\top$ in \eqref{eq:RNN:Cell}, then we have $\vh_t = \sigma(\mU\vh_{t-1} + \widetilde{\mW}\widetilde{\vx}_t)$. For the ease of notation, without ambiguity we denote $\mW := \widetilde{\mW}$ and $\vx_t := \widetilde{\vx}_t$. Then the recurrent cell can be reformulated as
\begin{equation}
\label{eq:basic:cell}
\vh_t = \sigma(\mU \vh_{t-1} + \mW \vx_t).
\end{equation}
Moreover, let $\phi(\cdot) := \sigma(\mU(\cdot))$ and $\vu_t := \mU^{-1}\mW\vx_t$, we can rewrite \eqref{eq:basic:cell} as
\begin{equation}
\label{eq:recurent-cell:PGD}
\vh_t = \phi(\vh_{t-1} + \vu_t).
\end{equation}
If we regard $-\vu_t$ as the ``gradient'' at the $t$-th iteration, then we can consider \eqref{eq:recurent-cell:PGD} as the dynamical system which updates the hidden state by the gradient and then transforms the updated hidden state by the nonlinear activation function $\phi$. We propose the following accelerated dynamical system to accelerate the dynamics of  \eqref{eq:recurent-cell:PGD}, which is principled by the accelerated gradient descent theory (see subsection~\ref{subsec:review:momentum}):
\begin{align}\label{eq:momentum:PGD}
\vp_{t} = \mu\vp_{t-1} - s \vu_{t};\ \ \vh_{t} =\phi(\vh_{t-1} - \vp_{t}),
\end{align}
where $\mu \geq 0, s >0$ are two hyperparameters, which are the analogies of the momentum coefficient and step size in the momentum-accelerated GD, respectively. Let $\vv_t := -\mU\vp_t$, we arrive at the following dynamical system:
\begin{align}\label{eq:momentum:cell}
\vv_{t} = \mu\vv_{t-1} + s \mW\vx_{t};\ \ \vh_{t} = \sigma(\mU\vh_{t-1} + \vv_{t}).
\end{align}
The architecture of the momentum cell that corresponds to the dynamical system \eqref{eq:momentum:cell} is plotted in Fig.~\ref{fig:momentumrnn-vs-rnn} (middle). Compared with the recurrent cell, the momentum cell introduces an auxiliary momentum state in each update and scales the dynamical system with two positive hyperparameters $\mu$ and $s$. 

\begin{remark}
\label{rm:different-parameterization}
Different parameterizations of \eqref{eq:momentum:PGD} can result in different momentum cell architectures. For instance, if we let $\vv_t = -\vp_t$, we end up with the following dynamical system:
\begin{align}\label{eq:momentum:cell2}
\vv_{t} = \mu\vv_{t-1} + s \widehat{\mW}\vx_{t};\ \ \vh_{t} = \sigma(\mU\vh_{t-1} + \mU\vv_{t}),
\end{align}
where $\widehat{\mW} := \mU^{-1}\mW$ is the trainable weight matrix. Even though \eqref{eq:momentum:cell} and \eqref{eq:momentum:cell2} are mathematically equivalent, the training procedure might cause the MomentumRNNs that are derived from different parameterizations to have different performances.
\end{remark}

\begin{remark}
We put the nonlinear activation in the second equation of \eqref{eq:momentum:PGD} to ensure that the value of $\vh_t$ is in the same range as the original recurrent cell. 
\end{remark}

\begin{remark}
The derivation above also applies to the dynamical systems in the LSTM cells, and we can design the MomentumLSTM in the same way as designing the MomentumRNN.
\end{remark}

\subsection{Analysis of the vanishing gradient: Momentum cell vs. Recurrent cell}\label{subsec:gradient}
Let $\vh_T$ and $\vh_t$ be the state vectors at the time step $T$ and $t$, respectively, and we suppose $T\gg t$. Furthermore, assume that $\mathcal{L}$ is the objective to minimize, then
\begin{equation}
\label{eq:gradient:rnn}
{\small \frac{\partial \mathcal L}{\partial \vh_t} = \frac{\partial \mathcal{L}}{\partial\vh_T}\cdot\frac{\partial \vh_T}{\partial \vh_t} = \frac{\partial \mathcal L}{\partial \vh_T}\cdot \prod_{k=t}^{T-1}\frac{\partial \vh_{k+1}}{\partial \vh_k} = \frac{\partial \mathcal L}{\partial \vh_T}\cdot \prod_{k=t}^{T-1}(\mD_k\mU^\top),}
\end{equation}
where {\small $\mU^\top$} is the transpose of $\mU$ and {\small $\mD_k = {\rm diag}(\sigma'(\mU\vh_k + \mW\vx_{k+1}))$} is a diagonal matrix with $\sigma'(\mU\vh_k + \mW\vx_{k+1})$ being its diagonal entries. 
{\small $\|\prod_{k=t}^{T-1}(\mD_k\mU^\top)\|_2$} tends to either vanish or explode \cite{bengio1994learning}. We can use regularization or gradient clipping to mitigate the exploding gradient, leaving vanishing gradient as the major obstacle to training RNN to learn long-term dependency~\cite{pascanu2013difficulty}. We can rewrite \eqref{eq:momentum:cell} as
\vskip -0.5cm
\begin{equation}
\label{eq:momentum:cell:one:eq}
{\small \vh_t = \sigma\left(\mU(\vh_{t-1} - \mu\vh_{t-2}) + \mu\sigma^{-1}(\vh_{t-1}) + s\mW\vx_t\right),}
\end{equation}
where $\sigma^{-1}(\cdot)$ is the inverse function of $\sigma(\cdot)$. We compute $\partial \mathcal L/\partial \vh_t$ as follows
\begin{equation}
\label{eq:gradient:momentum}
{\small \frac{\partial \mathcal L}{\partial \vh_t} = \frac{\partial \mathcal{L}}{\partial\vh_T}\cdot\frac{\partial \vh_T}{\partial \vh_t} = \frac{\partial \mathcal L}{\partial \vh_T}\cdot \prod_{k=t}^{T-1}\frac{\partial \vh_{k+1}}{\partial \vh_k} = \frac{\partial \mathcal L}{\partial \vh_T}\cdot \prod_{k=t}^{T-1}\widehat{\mD}_k[\mU^\top + \mu\boldsymbol\Sigma_k],}
\end{equation}
where {\small $\widehat{\mD}_k = {\rm diag}(\sigma'(\mU(\vh_{k} - \mu\vh_{k-1}) + \mu\sigma^{-1}(\vh_{k}) + s\mW\vx_{k+1} ))$} and {\small $\boldsymbol\Sigma = {\rm diag}((\sigma^{-1})'(\vh_k))$}. For mostly used $\sigma$, e.g., sigmoid and tanh,  $(\sigma^{-1}(\cdot))' > 1$ and $\mu\boldsymbol\Sigma_k$ dominates $\mU^\top$.\footnote{In the vanishing gradient scenario, $\|\mU\|_2$ is small; also it can be controlled by regularizing the loss function.} Therefore, with an appropriate choice of $\mu$, the momentum cell can alleviate vanishing gradient and accelerate training. 

We empirically corroborate that momentum cells can alleviate vanishing gradients by training a MomentumRNN and its corresponding RNN on the PMNIST classification task and plot $\|\partial \mathcal L/\partial \vh_t\|_{2}$ for each time step $t$. Figure~\ref{fig:stability} confirms that unlike in RNN, the gradients in MomentumRNN do not vanish. 

\begin{figure}[t!]
\centering
\includegraphics[width=0.99\linewidth]{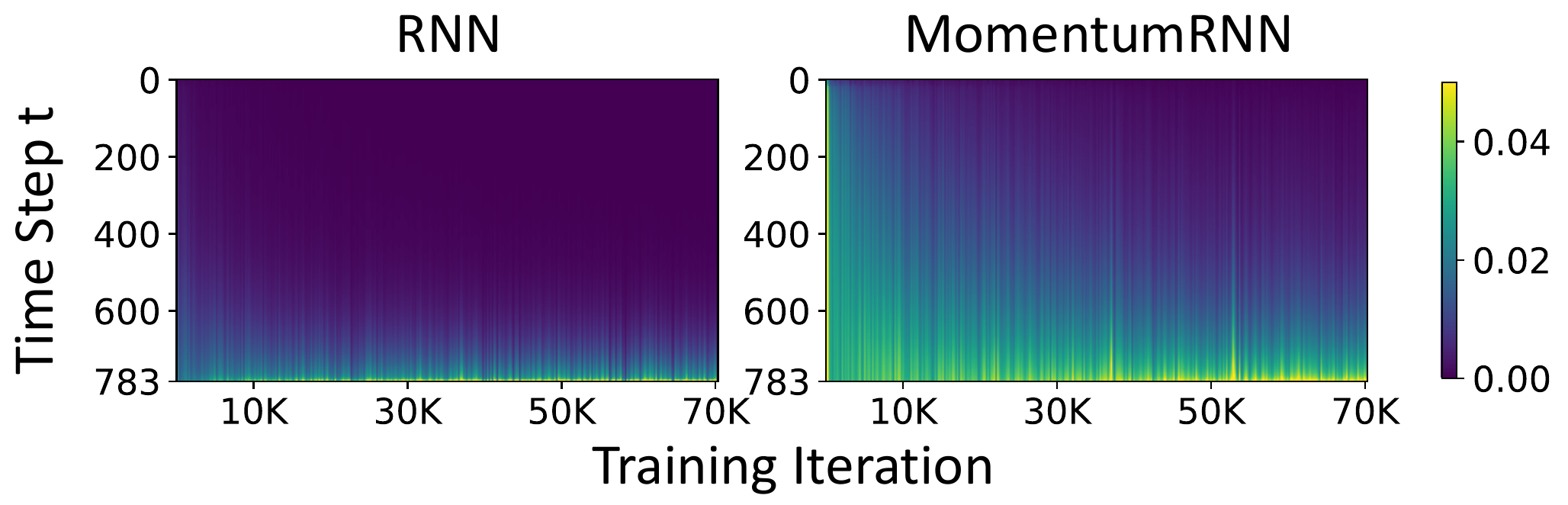}
\vskip -0.2cm
\caption{$\ell_2$ norm of the gradients of the loss $\mathcal{L}$ w.r.t. the state vector $\vh_t$ at each time step $t$ for RNN (left) and MomentumRNN (right). MomentumRNN does not suffer from vanishing gradients.}
\label{fig:stability}
\end{figure}

\subsection{Beyond MomentumRNN: NAG and Adam principled RNNs}\label{sec:Adam:cell} 
There are several other advanced formalisms of momentum existing in optimization, which can be leveraged for RNN architecture design. In this subsection, we present two additional variants of MomentumRNN that are derived from the Nesterov accelerated gradient (NAG)-style momentum with restart \cite{nesterov1983method,wang2020scheduled} and Adam \cite{kingma2014adam}.

{\bf NAG Principled RNNs.} The momentum-accelerated GD can be further accelerated by  replacing the constant momentum coefficient $\mu$ in~\eqref{eq:momentum:cell} with the NAG-style momentum, i.e. setting $\mu$ to $(t-1)/(t+2)$ at the $t$-th iteration. Furthermore, we can accelerate NAG by resetting the momentum to 0 after every $F$ iterations, i.e. $\mu = (t \mod F)/((t\mod F) + 3)$, which is the NAG-style momentum with a scheduled restart of the appropriately selected frequency $F$ \cite{wang2020scheduled}. For convex optimization, NAG has a convergence rate $O(1/t^2)$, which is significantly faster than GD or GD with constant momentum whose convergence rate is $O(1/t)$. Scheduled restart not only accelerates NAG to a linear convergence rate $O(\alpha^{-t}) (0<\alpha <1)$ under mild extra assumptions but also stabilizes the NAG iteration \cite{wang2020scheduled}.
We call the MomentumRNN with the NAG-style momentum and scheduled restart momentum the NAG-based RNN and the scheduled restart RNN (SRRNN), respectively.

{\bf Adam Principled RNNs.} Adam~\cite{kingma2014adam} leverages the moving average of historical gradients and entry-wise squared gradients to accelerate the stochastic gradient dynamics. We use Adam to accelerate \eqref{eq:recurent-cell:PGD} and end up with the following iteration
\begin{equation}
\label{eq:Adam:cell}
\begin{aligned}
\vp_t &= \mu \vp_{t-1} + (1-\mu)\vu_t\\
\vm_t &= \beta\vm_{t-1} + (1-\beta)\vu_t\odot\vu_t\\
\vh_t &= \phi(\vh_{t-1} - s \frac{\vp_t}{\sqrt{\vr_t} +\epsilon}),
\end{aligned}
\end{equation}
where $\mu, s, \beta > 0$ are hyperparameters, $\epsilon$ is a small constant and chosen to be $10^{-8}$ by default, and $\odot$/$\sqrt{\cdot}$ denotes the entrywise product/square root\footnote{In contrast to Adam, we do not normalize $\vp_t$ and $\vm_t$ since they can be absorbed in the weight matrices.}. Again, let $\vv_t = -\mU\vp_t$, we rewrite \eqref{eq:Adam:cell} as follows
\begin{equation*}
\label{eq:Adam:momentum:cell}
\begin{aligned}
\vv_t &= \mu \vv_{t-1} + (1-\mu)\mW\vx_t\\
\vm_t &= \beta\vm_{t-1} + (1-\beta)\vu_t\odot\vu_t\\
\vh_t &= \sigma(\mU\vh_{t-1} + s \frac{\vv_t}{\sqrt{\vm_t} +\epsilon}).
\end{aligned}
\end{equation*}
As before, here $\vu_t := \mU^{-1}\mW\vx_t$. Computing $\Ub^{-1}$ is expensive. Our experiments suggest that replacing $\vu_t\odot\vu_t$ by $\mW\vx_t \odot \mW\vx_t$ is sufficient and more efficient to compute. In our implementation, we also relax $\vv_t = \mu \vv_{t-1} + (1-\mu)\mW\vx_t$ to $\vv_t = \mu \vv_{t-1} + s\mW\vx_t$  that follows the momentum in the MomentumRNN \eqref{eq:momentum:cell} for better performance. Therefore, we propose the \emph{AdamRNN} that is given by
\begin{equation}
\label{eq:Adam:momentum:cell:final}
\begin{aligned}
\vv_t &= \mu \vv_{t-1} + s\mW\vx_t\\
\vm_t &= \beta\vm_{t-1} + (1-\beta)(\mW\vx_t\odot\mW\vx_t)\\ \vh_t &= \sigma(\mU\vh_{t-1} + \frac{\vv_t}{\sqrt{\vm_t} +\epsilon}).
\end{aligned}
\end{equation}
In AdamRNN, if $\mu$ is set to 0, we achieve another new RNN, which obeys the RMSProp gradient update rule~\cite{Tieleman2012}; which we call \emph{RMSPropRNN}.

\begin{remark}
Both AdamRNN and RMSPropRNN can also be derived by letting $\vv_t = -\vp_t$ and $\widehat{\mW} := \mU^{-1}\mW$ as in Remark~\ref{rm:different-parameterization}. This parameterization yields the following formulation for AdamRNN
\begin{equation*}
\label{eq:Adam:momentum:cell:final-2}
\begin{aligned}
\vv_t &= \mu \vv_{t-1} + s\widehat{\mW}\vx_t\\
\vm_t &= \beta\vm_{t-1} + (1-\beta)(\widehat{\mW}\vx_t\odot\widehat{\mW}\vx_t)\\
\vh_t &= \sigma(\mU\vh_{t-1} + \frac{\mU\vv_t}{\sqrt{\vm_t} +\epsilon}).
\end{aligned}
\end{equation*}
Here, we simply need to learn $\widehat{\mW}$ and $\mU$ without any relaxation. In contrast, we relaxed $\mU^{-1}$ to an identity matrix in~\eqref{eq:Adam:momentum:cell:final}. Our experiments suggest that both parameterizations yield similar results.
\end{remark}

\subsection{Experimental results}\label{sec:experiments-momentum-rnn}
In this subsection, we evaluate the effectiveness of our momentum approach in designing RNNs in terms of convergence speed and accuracy. We compare the performance of the MomentumLSTM with the baseline LSTM~\cite{hochreiter1997long} in the following tasks: 1) the object classification task on pixel-permuted MNIST~\cite{le2015simple}, 2) the speech prediction task on the TIMIT dataset~\cite{arjovsky2016unitary,pmlr-v80-helfrich18a,wisdom2016full,mhammedi2017efficient,pmlr-v48-henaff16}, 3) the celebrated copying and adding tasks \cite{hochreiter1997long,arjovsky2016unitary}, and 4) the language modeling task on the Penn TreeBank (PTB) dataset~\cite{mikolov2010recurrent}. These four tasks are among standard benchmarks to measure the performance of RNNs and their ability to handle long-term dependencies. Also, these tasks cover different data modalities -- image, speech, and text data -- as well as a variety of model sizes, ranging from thousands to millions of parameters with one (MNIST and TIMIT tasks) or multiple (PTB task) recurrent cells in concatenation. Our experimental results confirm that MomentumLSTM converges faster and yields better test accuracy than the baseline LSTM across tasks and settings. We also discuss the AdamLSTM, RMSPropLSTM, and scheduled restart LSTM (SRLSTM) and show their advantage over MomentumLSTM in specific tasks. 
All of our results are averaged over 5 runs with different seeds. For MNIST and TIMIT experiments, we use the baseline codebase provided by~\cite{dtrivgithub}. For PTB experiments, we use the baseline codebase provided by~\cite{ptbgithub}.

\subsubsection{Pixel-by-Pixel MNIST}
\label{sec:mnist-exp}
In this task, we classify image samples of hand-written digits from the MNIST dataset~\cite{lecun2010mnist} into one of the ten classes. Following the implementation of~\cite{le2015simple}, we flatten the image of original size 28 $\times$ 28 pixels and feed it into the model as a sequence of length 784. In the unpermuted task (MNIST), the sequence of pixels is processed row-by-row. In the permuted task (PMNIST), a fixed permutation is selected at the beginning of the experiments and then applied to both training and test sequences. We summarize the results in Table~\ref{tab:mnist}. Our experiments show that \emph{MomentumLSTM achieves better test accuracy than the baseline LSTM in both MNIST and PMNIST digit classification tasks} using different numbers of hidden units (i.e. $N=128, 256$). Especially, the improvement is significant on the PMNIST task, which is designed to test the performance of RNNs in the context of long-term memory. Furthermore, we notice that \emph{MomentumLSTM converges faster than LSTM} in all settings. Figure~\ref{fig:loss-vs-iters} 
corroborates this observation when using $N=256$ hidden units.

\begin{table}[!t]
\vspace{-0.02in}
    \caption{Best test accuracy at the MNIST and PMNIST tasks (\%). We use the baseline results reported in \cite{pmlr-v80-helfrich18a},~\cite{wisdom2016full},~\cite{vorontsov2017orthogonality}. Our proposed models outperform the baseline LSTM. Among the models using $N=256$ hidden units, RMSPropLSTM yields the best results. 
    }
\vspace{-0.05in}
\label{tab:mnist}
\begin{center}
\begin{footnotesize}
\begin{sc}
\begin{tabular}{lclcc}
    \toprule
    Model & n & \# params & MNIST & PMNIST \\
    \midrule
    LSTM & $128$ & $\approx 68K$ & $98.70$\cite{pmlr-v80-helfrich18a},$97.30$  \cite{vorontsov2017orthogonality} & $92.00$  \cite{pmlr-v80-helfrich18a},$92.62$ \cite{vorontsov2017orthogonality} \\
    LSTM & $256$ & $\approx 270K$ & $98.90$ \cite{pmlr-v80-helfrich18a}, $98.50$ \cite{wisdom2016full} & $92.29$  \cite{pmlr-v80-helfrich18a}, $92.10$  \cite{wisdom2016full} \\
    \midrule
    MomentumLSTM & $128$ & $\approx 68K$ & $\bf{99.04 \pm 0.04}$ & $\bf{93.40 \pm 0.25}$\\
    MomentumLSTM & $256$ & $\approx 270K$ & $\bf{99.08 \pm 0.05}$ & $\bf{94.72 \pm 0.16}$\\
    \midrule
    AdamLSTM & $256$ & $\approx 270K$ & $99.09 \pm 0.03$ & $95.05 \pm 0.37$\\
    RMSPropLSTM & $256$ & $\approx 270K$ & $\bf{99.15 \pm 0.06}$ & $\bf{95.38 \pm 0.19}$\\
    SRLSTM & $256$ & $\approx 270K$ & $ 99.01 \pm 0.07 $ & $93.82 \pm 1.85$\\
    \bottomrule
\end{tabular}
\end{sc}
\end{footnotesize}
\end{center}
\end{table}

\subsubsection{TIMIT speech dataset}
\label{sec:timit-exp}
We study how MomentumLSTM performs on audio data with speech prediction experiments on the TIMIT speech dataset~\cite{garofolo1993timit}, which is a collection of real-world speech recordings. As first proposed by~\cite{wisdom2016full}, the recordings are downsampled to 8kHz and then transformed into log-magnitudes via a short-time Fourier transform (STFT). The task accounts for predicting the next log-magnitude given the previous ones. We use the standard train/validation/test separation in~\cite{wisdom2016full,lezcano2019cheap,casado2019trivializations}, thereby having 3640 utterances for the training set with a validation set of size 192 and a test set of size 400.

The results for this TIMIT speech prediction are shown in Table~\ref{tab:timit}. Results are reported on the test set using the model parameters
that yield the best validation loss. Again, we see the advantage of MomentumLSTM over the baseline LSTM. In particular, MomentumLSTM yields much better prediction accuracy and faster convergence speed compared to
LSTM. Figure~\ref{fig:loss-vs-iters} 
shows the convergence of MomentumLSTM vs. LSTM when using $N=158$ hidden units.

\begin{figure}[!t]
\centering
\includegraphics[width=0.9\linewidth]{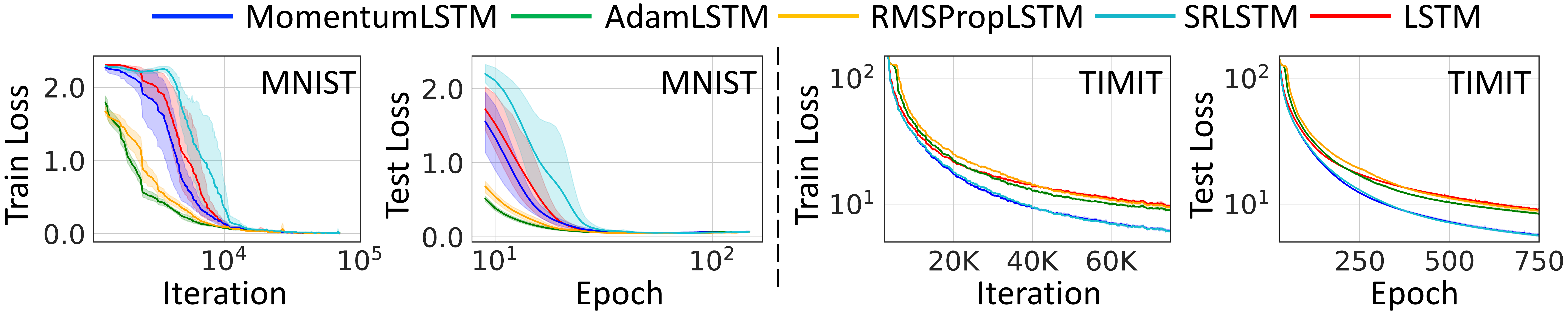}
\vskip -0.3cm
\caption{Train and test loss  of MomentumLSTM, AdamLSTM, RMSPropLSTM, SRLSTM, and LSTM for MNIST (left two panels) and TIMIT (right two panels) tasks. MomentumLSTM converges faster than LSTM in both tasks. For MNIST, AdamLSTM and RMSPropLSTM converge fastest. For TIMIT, MomentumLSTM and SRLSTM converge fastest.}
\label{fig:loss-vs-iters}
\end{figure}

\begin{table}[!t]
\vspace{-0.1in}
    \caption{Test and validation MSEs at the end of the epoch with the lowest validation
MSE for the TIMIT task. All of our proposed models outperform the baseline LSTM. Among models using $N=158$ hidden units, SRLSTM performs the best.} 
\label{tab:timit}
\begin{center}
\begin{footnotesize}
\begin{sc}
\begin{tabular}{lclcc}
    \toprule
    Model & n & \# params & Val. MSE & Test MSE \\
    \midrule
    LSTM & $84$ & $\approx 83K$ & $14.87 \pm 0.15$ & $14.94 \pm 0.15$  \\
    LSTM & $120$ & $\approx 135K$ & $11.77 \pm 0.14$  & $11.83 \pm 0.12$  \\
    LSTM & $158$ & $\approx 200K$ & $9.33 \pm 0.14$  & $9.37 \pm 0.14$  \\
    \midrule
    MomentumLSTM & $84$ & $\approx 83K$ & $\bf{10.90 \pm 0.19}$ & $\bf{10.98 \pm 0.18}$\\
    MomentumLSTM & $120$ & $\approx 135K$ & $\bf{8.00 \pm 0.30}$ & $\bf{8.04 \pm 0.30}$\\
    MomentumLSTM & $158$ & $\approx 200K$ & $\bf{5.86 \pm 0.14}$ & $\bf{5.87 \pm 0.15}$\\
    \midrule
    AdamLSTM & $158$ & $\approx 200K$ & $8.66 \pm 0.15$ & $8.69 \pm 0.14$\\
    RMSPropLSTM & $158$ & $\approx 200K$ & $9.13 \pm 0.33$ & $9.17 \pm 0.33$\\
    SRLSTM & $158$ & $\approx 200K$ & $\bf{5.81 \pm 0.10}$ & $\bf{5.83 \pm 0.10}$\\
    \bottomrule
\end{tabular}
\end{sc}
\end{footnotesize}
\end{center}
\vskip -0.05in
\end{table}

\subsubsection{Copying and adding tasks}
\label{sec:copying-adding-maintext}
Two other important tasks for measuring the ability of a model to learn long-term dependency are the copying and adding tasks \cite{hochreiter1997long,arjovsky2016unitary}. In both copying and adding tasks, avoiding vanishing/exploding gradients becomes more relevant when the input sequence length increases. We compare the performance of MomentumLSTM over LSTM on these tasks. We also examine the performance of AdamLSTM, RMSPropLSTM, and SRLSTM on the same tasks. We 
summarize our results in Figure~\ref{fig:copy-add-task}. 
In copying task for sequences of length 2K, MomentumLSTM obtains slightly better final training loss than the baseline LSTM (0.009 vs.\ 0.01). In adding task for sequence of length 750, both models achieve similar training loss of 0.162. However, AdamLSTM and RMSPropLSTM significantly outperform the baseline LSTM.
\begin{figure}[t!]
\centering
\includegraphics[width=0.65\linewidth]{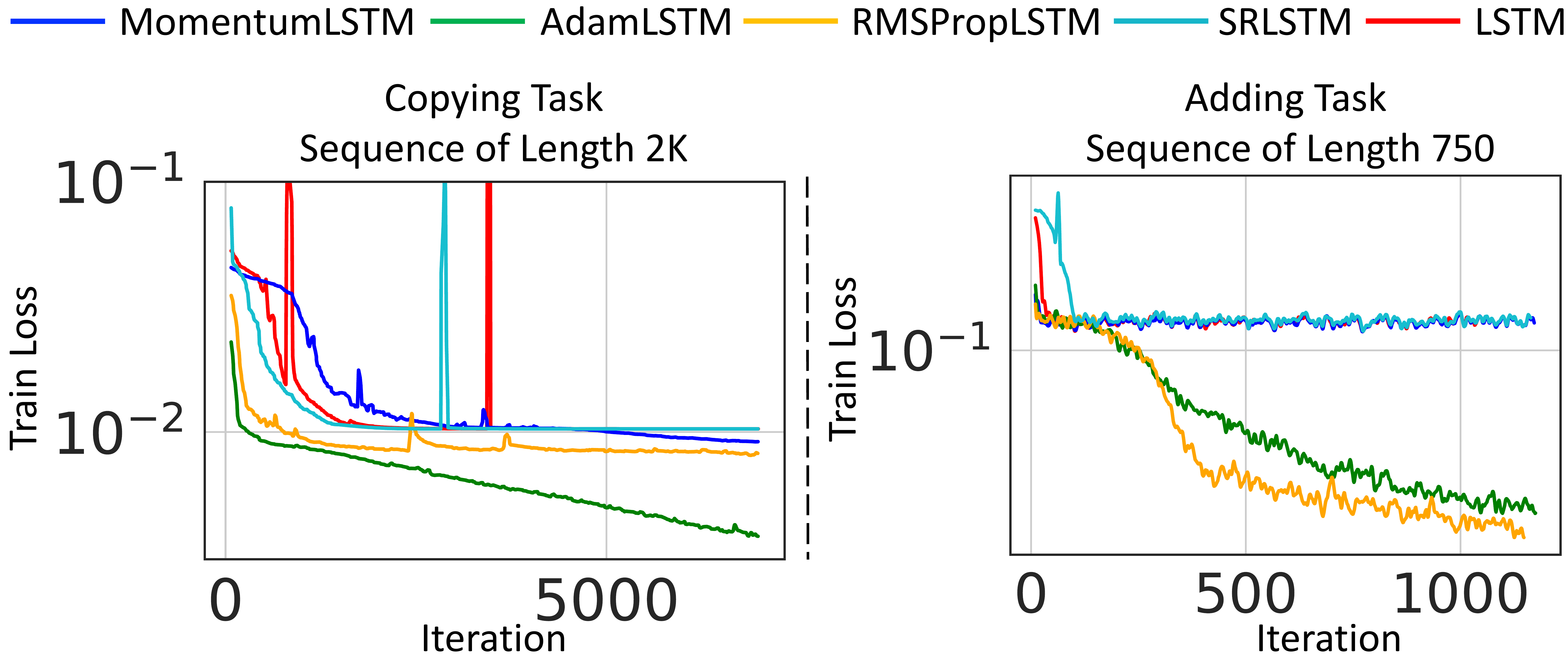}
\vskip -0.2cm
\caption{Train loss vs. iteration for (left) copying task with sequence length 2K and (right) adding task with sequence length 750. AdamLSTM and RMSPropLSTM converge faster and to better final losses than other models. 
}
\label{fig:copy-add-task}
\end{figure}

\subsubsection{Word-level Penn TreeBank}
To study the advantage of MomentumLSTM over LSTM on text data, we perform language
modeling on a preprocessed version of the  
PTB dataset~\cite{mikolov2010recurrent}, 
which has been a standard benchmark for evaluating language models. 
Unlike the baselines used in 
\begin{wrapfigure}{r}{.5\textwidth}
\vspace{0.3cm}
\includegraphics[width=\linewidth]{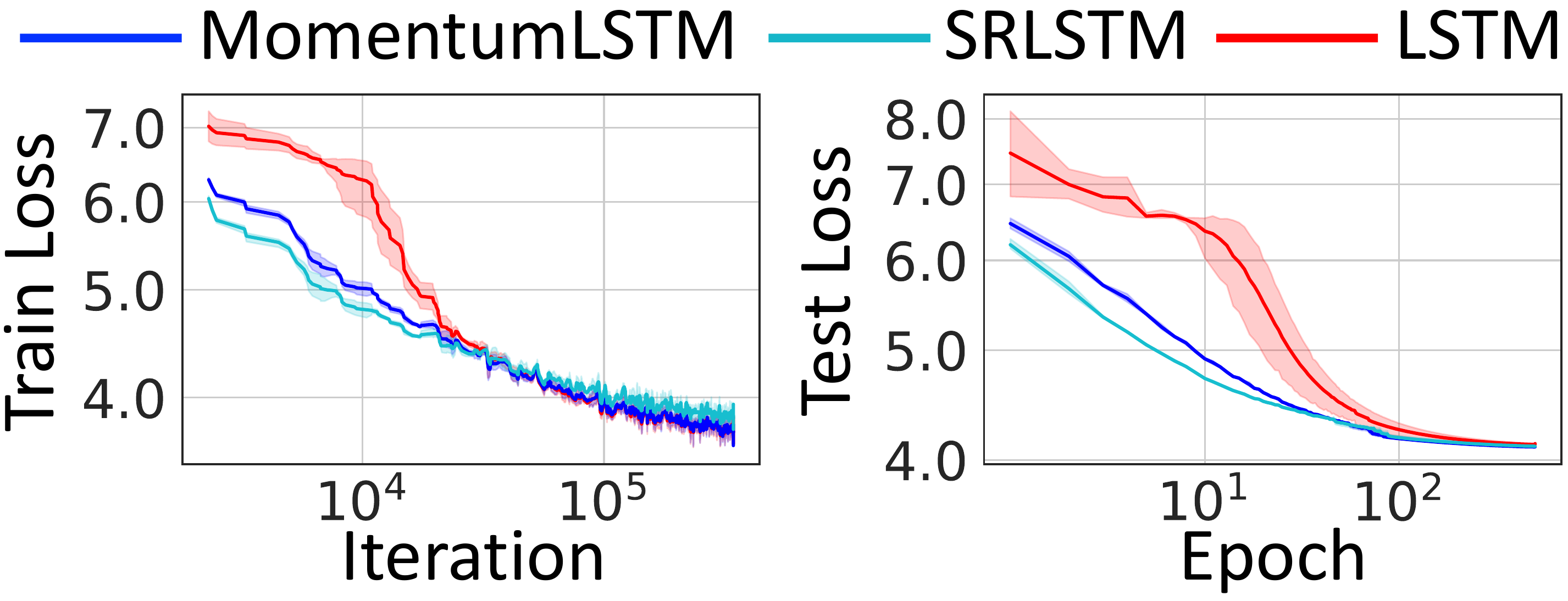}
\vspace{-0.2in}
\caption{Train (left) and test loss (right) of MomentumLSTM (blue), SRLSTM (cyan), and LSTM (red) for the Penn Treebank language modeling tasks at word level.}\label{fig:train-loss-vs-iters-ptb}\vspace{-0.5cm}
\end{wrapfigure}
the (P)MNIST and TIMIT experiments which 
contain one LSTM cell, in this PTB experiment, we use a three-layer LSTM model, 
which contains three concatenated LSTM cells, as the baseline. The size of this model in terms of the number of parameters is also much larger than those in the (P)MNIST and TIMIT experiments. Table~\ref{tab:ptb-word-level} shows the test and validation perplexity (PPL) using the model parameters that yield the best validation loss. Again, MomentumLSTM achieves better perplexities and converges faster than the baseline LSTM (see Figure~\ref{fig:train-loss-vs-iters-ptb}).
\begin{table}[t!]
    \caption{Model test perplexity at the end of the epoch with the lowest validation perplexity for the Penn Treebank language modeling task (word level).} 
\vspace{-0.05in}
\label{tab:ptb-word-level}
\begin{center}
\begin{small}
\begin{sc}
\begin{tabular}{l l c c}
    \toprule
    Model & \# params & Val. PPL & Test PPL \\
    \midrule
    \lstm & $\approx 24M$ & $61.96 \pm 0.83$ & $59.71 \pm 0.99$  \\
    \midrule
    MomentumLSTM & $\approx 24M$ & $\bf{60.71 \pm 0.24}$ & $\bf{58.62 \pm 0.22}$\\
    \midrule
    SRLSTM & $\approx 24M$ & $61.12 \pm 0.68$ & $58.83 \pm 0.62$\\
    \bottomrule
\end{tabular}
\end{sc}
\end{small}
\end{center}
\vskip -0.05in
\end{table}

\subsubsection{NAG and Adam principled RNNs}
Finally, we evaluate AdamLSTM, RMSPropLSTM and SRLSTM on all tasks. 
For (P)MNIST and TIMIT tasks, we summarize the test accuracy of the trained models in Tables~\ref{tab:mnist} and~\ref{tab:timit} and provide the plots of train and test losses in Figure~\ref{fig:loss-vs-iters}. We observe that though AdamLSTM and RMSPropLSTM work better than the MomentumLSTM at (P)MNIST task, they yield worse results at the TIMIT task. Interestingly, SRLSTM shows an opposite behavior - better than MomentunLSTM at TIMIT task but worse at (P)MNIST task. For the copying and adding tasks, Figure~\ref{fig:copy-add-task} shows that AdamLSTM and RMSPropLSTM converge faster and to better final training loss than other models in both tasks. Finally, for the PTB task, both MomentumLSTM and SRLSTM outperform the baseline LSTM (see Figure~\ref{fig:train-loss-vs-iters-ptb} and Table~\ref{tab:ptb-word-level}). However, in this task, AdamLSTM and RMSPropLSTM yields slightly worse performance than the baseline LSTM. In particular, test PPL for AdamLSTM and RMSPropLSTM are $61.11 \pm 0.31$, and $64.53 \pm 0.20$, respectively, which are higher than the test PPL for LSTM ($59.71 \pm 0.99$). We observe that there is no model that win in all tasks. This is somewhat expected, given the connection between our model and its analogy to optimization algorithm. An optimizer needs to be chosen for each particular task, and so is for our MomentumRNN. All of our models outperform the baseline LSTM.

\section{Neural ODEs}\label{sec-neural-ODE}
In this section, we derive the continuous limit of heavy-ball momentum and then present a new class of neural ODEs, named heavy-ball neural ODEs (HBNODEs), which have two properties that imply practical advantages over NODEs: 1) The adjoint state of an HBNODE also satisfies an HBNODE, accelerating both forward and backward ODE solvers, thus significantly reducing the NFEs and improving the utility of trained models. 2) The spectrum of HBNODEs is well structured, enabling effective learning of long-term dependencies from complex sequential data. We verify the advantages of HBNODEs over NODEs on benchmark tasks, including image classification, learning complex dynamics, and sequential modeling. Our method requires remarkably fewer forward and backward NFEs, is more accurate, and learns long-term dependencies more effectively than the other ODE-based neural network models. Part of the results in this section has been accepted for publication at NeurIPS 2021 \cite{HBNODE2021}.

\subsection{Recap on neural ODEs}
Neural ODEs (NODEs) are a family of continuous-depth machine learning models whose forward and backward propagations rely on solving an ODE and its adjoint equation \cite{chen2018neural}. NODEs model the dynamics of hidden features $\vh(t)\in \RR^N$ using an ODE, which is parametrized by a neural network $f(\vh(t),t,\theta)\in \RR^N$ with learnable parameters $\theta$, i.e.,
\begin{equation}\label{eq:NODE}
\frac{d\vh(t)}{dt} = f(\vh(t),t,\theta).
\end{equation}
Starting from the input $\vh(t_0)$, NODEs obtain the output $\vh(T)$ by solving \eqref{eq:NODE} for $t_0 \le t\le T$ with the initial value $\vh(t_0)$, using a black-box numerical ODE solver. The NFEs that the black-box ODE solver requires in a single forward pass is an analogue for the continuous-depth models \cite{chen2018neural} to the depth of networks in ResNets \cite{He2015}. The loss between 
$\vh(T)$ and the ground truth is denoted by $\mathcal{L}(\vh(T))$; we update parameters $\theta$ using the following gradient \cite{pontryagin2018mathematical}
\begin{equation}\label{eq:dL-dtheta}
\frac{d\mathcal{L}(\vh(T))}{d\theta} = \int_{t_0}^T \va(t)\frac{\partial f(\vh(t),t,\theta)}{\partial\theta}dt,
\end{equation}
where $\va(t):=\partial\mathcal{L}/\partial\vh(t)$ 
satisfies the following adjoint equation
\begin{equation}\label{eq:adjoint}
\frac{d\va(t)}{dt} = -\va(t)\frac{\partial f(\vh(t),t,\theta)}{\partial\vh}.
\end{equation}
NODEs are flexible in learning from irregularly-sampled sequential data and particularly suitable for learning  complex dynamical systems \cite{chen2018neural,NEURIPS2019_42a6845a,NEURIPS2019_227f6afd,norcliffe2020_sonode,NEURIPS2020_e562cd9c,KidgerMFL20}, which can be trained by efficient algorithms \cite{Quaglino2020SNODE:,NEURIPS2020_c24c6525,zhuang2021mali}.
The drawback of NODEs is also prominent. In many ML tasks, NODEs require very high NFEs in both training and inference, especially in high accuracy settings where a lower tolerance is needed. The NFEs increase rapidly with training;
high NFEs reduce computational speed and accuracy of NODEs and can lead to blow-ups in the worst-case scenario \cite{grathwohl2018scalable,NEURIPS2019_21be9a4b,massaroli2020dissecting,norcliffe2020_sonode}. As an illustration, we train NODEs for CIFAR10 classification using the same model and experimental settings as in \cite{NEURIPS2019_21be9a4b}, except using a tolerance of {$10^{-5}$}; Fig.~\ref{fig:node-hbnode-cifar10} shows both forward and backward NFEs and the training time of different ODE-based models; we see that NFEs and computational times increase
very rapidly for NODE, ANODE \cite{NEURIPS2019_21be9a4b}, and SONODE \cite{norcliffe2020_sonode}. More results on the large NFE and degrading utility issues for different benchmark experiments are available in Section~\ref{sec:experiments}. Another issue is that NODEs often fail to effectively learn long-term dependencies in sequential data \cite{lechner2020learning}, discussed in subsection~\ref{sec:lowerbounds}.

\begin{figure}
\begin{center}
\begin{tabular}{cccc}
\hskip -0.3cm
\includegraphics[width=0.21\columnwidth]{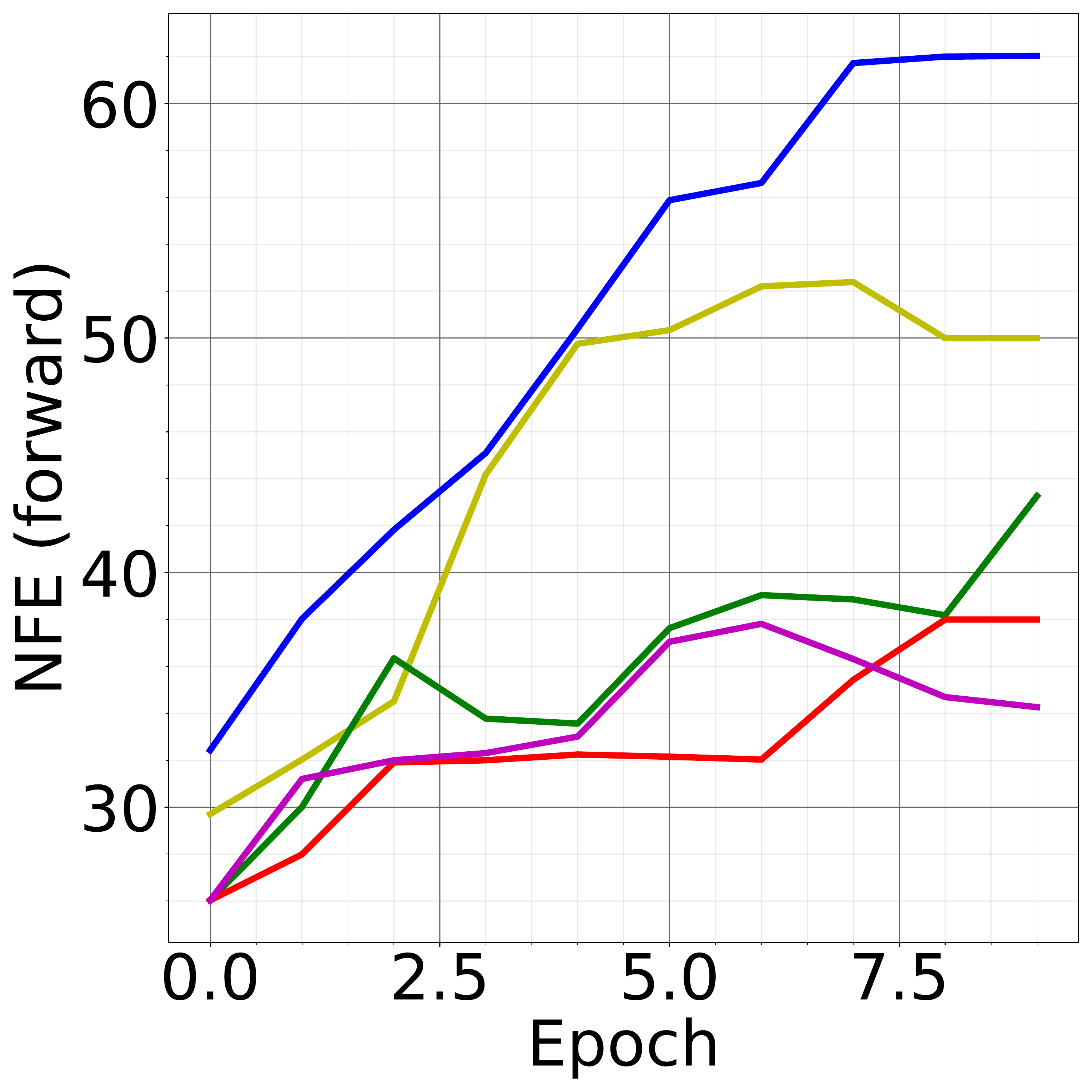}&
\hskip -0.3cm
\includegraphics[width=0.21\columnwidth]{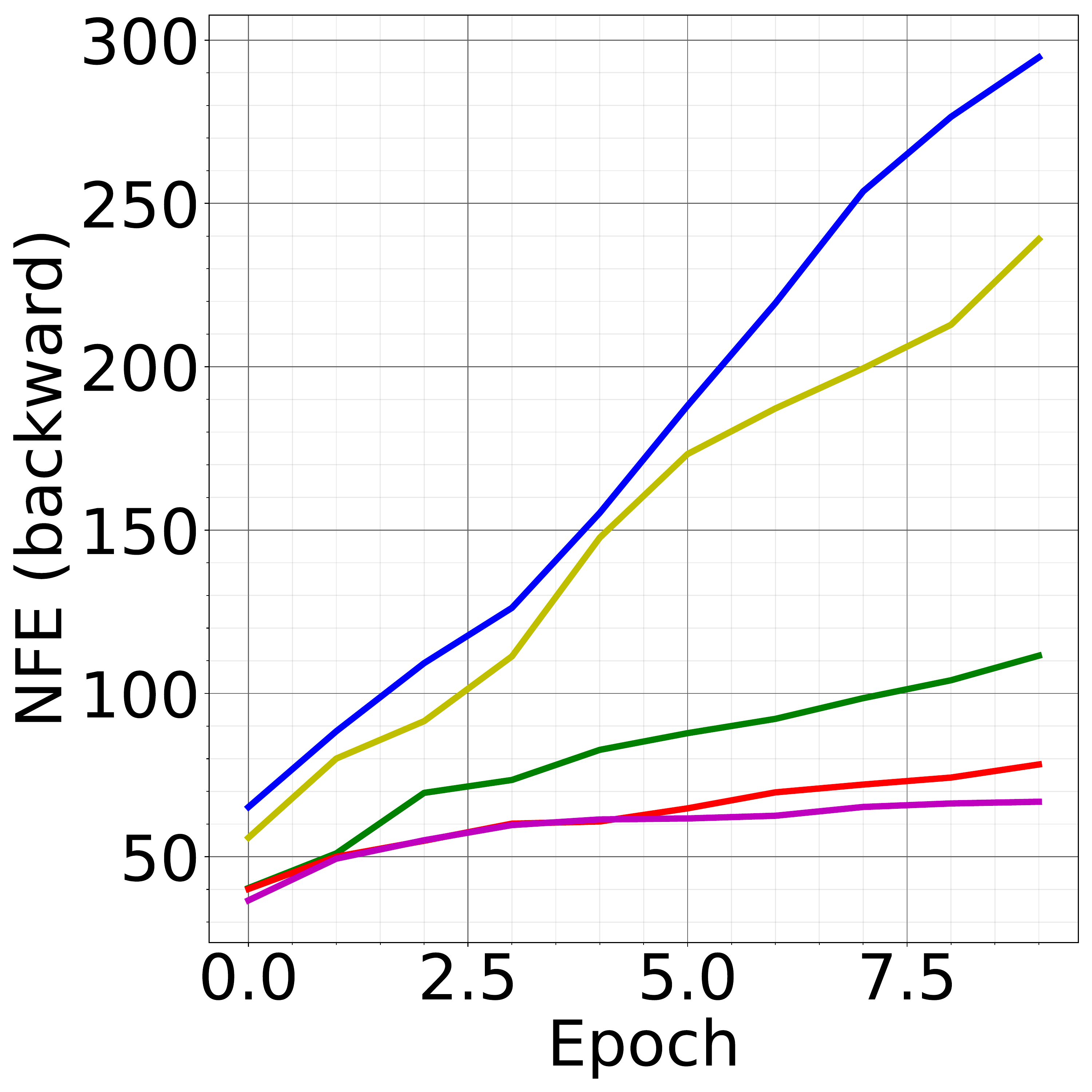}& 
\hskip -0.3cm
\includegraphics[width=0.21\columnwidth]{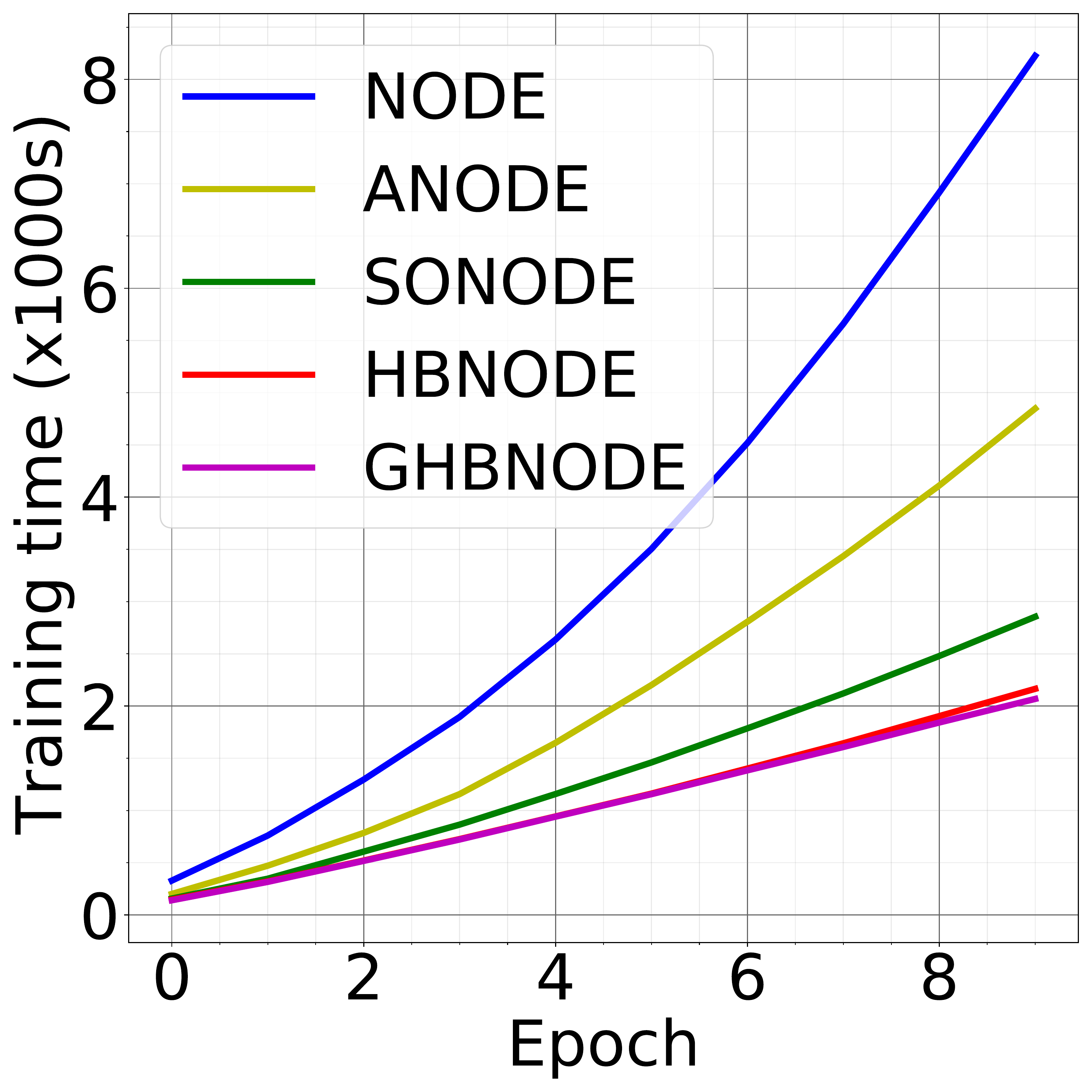} &
\hskip -0.3cm
\includegraphics[width=0.21\columnwidth]{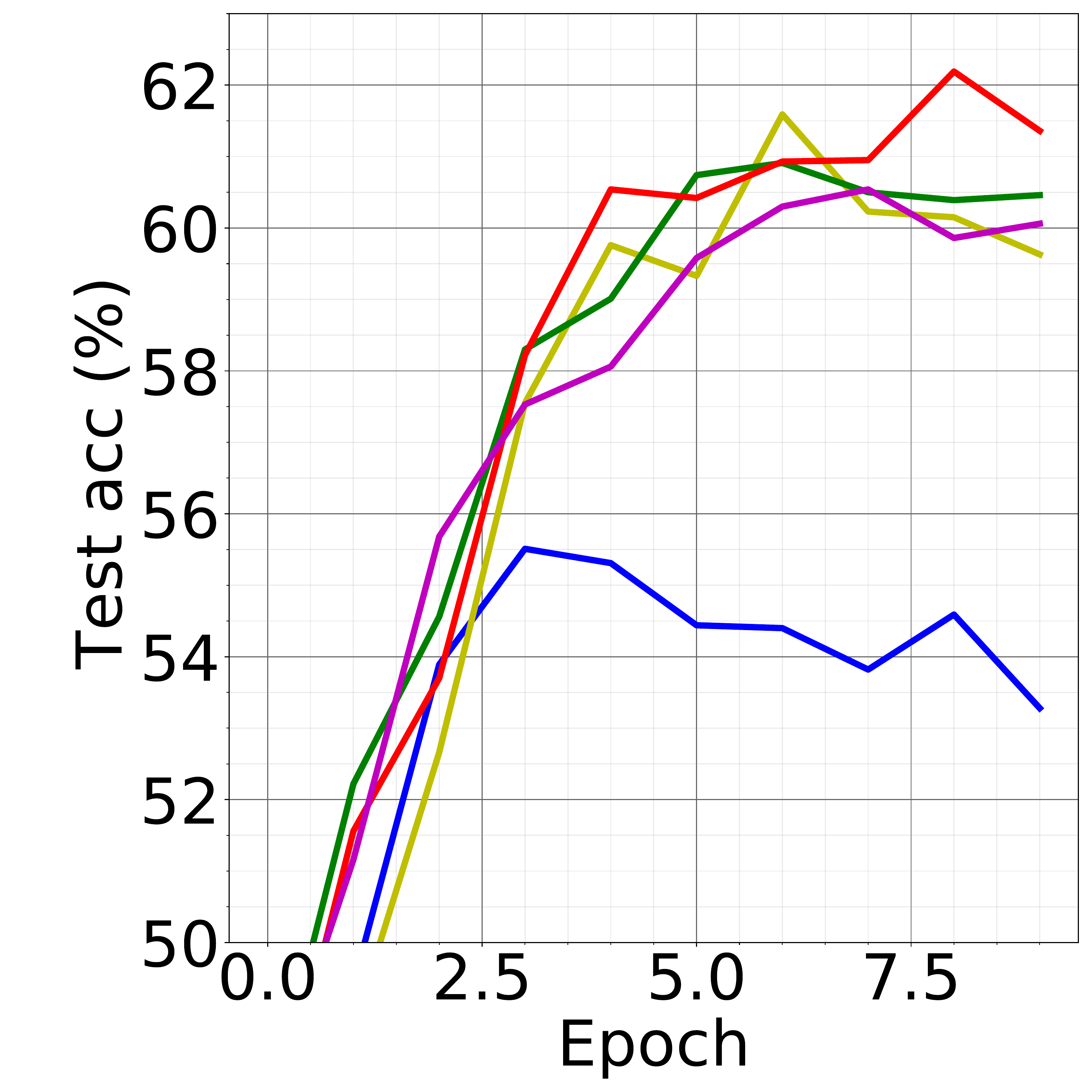}\\
  \end{tabular}
  \end{center}
  \vskip -0.2in
  \caption{Contrasting NODE, ANODE, SONODE, HBNODE, and GHBNODE for CIFAR10 classification in NFEs, training time, and test accuracy. (Tolerance: {$10^{-5}$}, see subsection~\ref{sec:image-classification} for experimental details.)}\label{fig:node-hbnode-cifar10}
\end{figure}

\subsection{Heavy-ball neural ODEs}\label{sec:HBNODE}
\subsubsection{Heavy-ball ODE}
We derive the HBODE from the heavy-ball momentum method. 
For any fixed step size $s$, let 
$
\vm^k := ({\vx^{k+1}-\vx^k})/{\sqrt{s}},
$ 
and let $\beta := 1-\gamma \sqrt{s}$, where $\gamma\geq 0$ is another hyperparameter. Then we can rewrite \eqref{eq:HB1} as 
\begin{equation}\label{eq:Appendix-HB3}
\vm^{k+1}=(1-\gamma\sqrt{s})\vm^k -\sqrt{s}\nabla F(\vx^k);\ \vx^{k+1} = \vx^k + \sqrt{s}\vm^{k+1}.
\end{equation}
Let $s\to 0$ in \eqref{eq:Appendix-HB3}; we obtain the following system of first-order ODEs, 
\begin{equation}\label{eq:Appendix-HBNODE1}
\frac{d\vx(t)}{dt}= \vm(t);\ \frac{d\vm(t)}{dt} = -\gamma \vm(t) - \nabla F(\vx(t)).
\end{equation}
This can be further rewritten as a second-order heavy-ball ODE (HBODE), which also models a damped oscillator,
\begin{equation}\label{eq:Appendix-HBNODE2}
\frac{d^2\vx(t)}{dt^2} + \gamma \frac{d\vx(t)}{dt} = -\nabla F(\vx(t)).
\end{equation}

\subsubsection{Heavy-ball neural ODEs}
{Similar to NODE, we parameterize $-\nabla F$ in \eqref{eq:Appendix-HBNODE2} using a neural network $f(\vh(t),t,\theta)$,} resulting in the following HBNODE with initial position $\vh(t_0)$ and momentum $\vm(t_0):=d\vh/dt(t_0)$, 
\begin{equation}\label{eq:HBNODE-2nd}
\frac{d^2\vh(t)}{dt^2} + \gamma \frac{d\vh(t)}{dt} = f(\vh(t),t,\theta),
\end{equation}
where $\gamma\geq 0$ is the damping parameter, which can be set as a tunable or a learnable hyperparmater with positivity constraint. In the trainable case, we use $\gamma = \epsilon \cdot \text{sigmoid}(\omega)$ for a trainable $\omega\in \mathbb{R}$ and a fixed tunable upper bound $\epsilon$ (we set $\epsilon=1$ below).
According to \eqref{eq:Appendix-HBNODE1}, HBNODE \eqref{eq:HBNODE-2nd} is equivalent to 
\begin{equation}\label{eq:HBNODE-1st}
\frac{d\vh(t)}{dt} = \vm(t); 
\quad \frac{d\vm(t)}{dt} = -\gamma \vm(t) + f(\vh(t),t,\theta).
\end{equation}
Equation \eqref{eq:HBNODE-2nd} (or equivalently, the system \eqref{eq:HBNODE-1st}) defines the forward ODE for the HBNODE, {and} we can use either the first-order (Prop.~\ref{prop:adjoint-HBNODE-1st}) or the second-order (Prop.~\ref{prop:adjoint-HBNODE}) adjoint sensitivity method to update the parameter $\theta$ \cite{norcliffe2020_sonode}.

\begin{proposition}[Adjoint equation for HBNODE]\label{prop:adjoint-HBNODE}
The adjoint state $\va(t):=\partial\mathcal{L}/\partial\vh(t)$ for the HBNODE \eqref{eq:HBNODE-2nd} satisfies the following HBODE with the same damping parameter $\gamma$ as that in \eqref{eq:HBNODE-2nd},
\begin{equation}\label{eq:adjoint-HBNODE-2nd}
\frac{d^2\va(t)}{dt^2} - \gamma\frac{d\va(t)}{dt} = \va(t) \frac{\partial f}{\partial\vh}(\vh(t),t,\theta). 
\end{equation}
\end{proposition}
\begin{proof}
Consider the following coupled first-order ODE system
\begin{equation}
\frac{\partial }{\partial t}\begin{bmatrix}
\vh \\ \vv
\end{bmatrix} = \begin{bmatrix}
\vv \\ f(\vh(t), \vv(t), t, \theta)
\end{bmatrix},
\quad 
\begin{bmatrix}
\vh \\ \vv
\end{bmatrix}(t_0) = \begin{bmatrix}
\vh_{t_0} \\ \vv_{t_0}
\end{bmatrix}.
\end{equation}

Denote $\vz = \begin{bmatrix}\vh\\\vv\end{bmatrix}$ and final state as 
\begin{equation}
\begin{bmatrix}
\vh(T) \\ \vv(T)
\end{bmatrix}
= \begin{bmatrix}
\vh_{T} \\ \vv_{T}
\end{bmatrix} = \vz_T.
\end{equation}

Then the adjoint equation is given by
\begin{equation}
\frac{\partial \mA(t)}{\partial t} = -\mA(t) \begin{bmatrix}
{\bf 0} & \mI \\
\frac{\partial f}{\partial \vh} & \frac{\partial f}{\partial \vv}
\end{bmatrix},
\quad 
\mA(T) = -\mI,
\quad 
\va(t) = -\frac{d\mathcal{L}}{d\vz_T} \mA(t).
\end{equation}
By rewriting $\mA = \begin{bmatrix}
\mA_\vh & \mA_\vv
\end{bmatrix}$, we have the following differential equations

\begin{equation}
\frac{\partial \mA_\vh(t)}{\partial t} = -\mA_\vv(t) \frac{\partial f}{\partial \vh},
\quad 
\frac{\partial \mA_\vv(t)}{\partial t} = -\mA_\vh(t) - \mA_\vv(t) \frac{\partial f}{\partial \vv},
\end{equation}
with initial conditions 
\begin{equation}
\mA_\vh(T) = -\begin{bmatrix}\mI\\{\bf 0}\end{bmatrix},
\quad 
\mA_\vv(T) = -\begin{bmatrix}{\bf 0}\\\mI\end{bmatrix},
\end{equation}
and adjoint states 
\begin{equation}
\va_\vh(t) = \frac{d\mathcal{L}}{d\vz_T} \mA_\vh(t),
\quad 
\va_\vv(t) = \frac{d\mathcal{L}}{d\vz_T} \mA_\vv(t).
\end{equation}
The gradient equations becomes 
\begin{equation}
\frac{d \mathcal{L}}{d \theta} = \int_{t_0}^{T}\va \begin{bmatrix}
{\bf 0} \\ \frac{\partial f}{\partial \theta} 
\end{bmatrix} dt = \int_{t_0}^{T}\va_\vv \frac{\partial f}{\partial \theta} dt,
\quad 
\frac{d \mathcal{L}}{d \vh_{t_0}} = \va_\vh (t_0),
\quad 
\frac{d \mathcal{L}}{d \vv_{t_0}} = \va_\vv (t_0).
\end{equation}
Note $\vh_{t_0}$ is fixed, and thus $\va_\vh$ disappears in gradient computation. Therefore, we are only interested in $\va_\vv$. Thus the adjoint $\mA_\vv$ satisfies the following second-order ODE
\begin{equation}
\frac{\partial^2 \mA_\vv(t)}{\partial t^2} = \mA_\vv(t) \frac{\partial f}{\partial \vh} - \frac{\partial (\mA_\vv(t) \frac{\partial f}{\partial \vv})}{\partial t},
\end{equation}
and thus
\begin{equation}\label{eq:254}
\frac{\partial^2 \va_\vv(t)}{\partial t^2} = \va_\vv(t) \frac{\partial f}{\partial \vh} - \frac{\partial (\va_\vv(t) \frac{\partial f}{\partial \vv})}{\partial t},
\end{equation}
with initial conditions
\begin{equation}\label{eq:255}
\va_\vv (T) = -\frac{d\mathcal{L}}{d\vz} \mA_\vv(T) = \frac{d\mathcal{L}}{d\vv_{T}}, 
\quad 
\frac{\partial \va_\vv(T)}{\partial t} = -\frac{d\mathcal{L}}{d\vh_{T}} - \va_\vv(T) \frac{\partial f}{\partial \vv}(T).
\end{equation}
As HBNODE takes the form 
\begin{equation}
\frac{d^2\vh(t)}{dt^2} + \gamma \frac{d\vh(t)}{dt} = f(\vh(t),t,\theta), 
\end{equation}
which can also be viewed as a SONODE. By applying the adjoint equation \eqref{eq:254}, we arrive at
\begin{equation}
\frac{\partial^2 \va(t)}{\partial t^2} = \va(t) \frac{\partial f}{\partial \vh} +\gamma \frac{\partial \va(t)}{\partial t}.
\end{equation}
As HBNODE only carries its state $\vh$ to the loss $\mathcal{L}$, we have $\frac{d\mathcal{L}}{d\vv_{T}} = 0$, and thus the initial conditions in equation \eqref{eq:255} becomes 
\begin{equation}
\va(T) = {\bf 0}, 
\quad 
\frac{\partial \va(T)}{\partial t} = -\frac{d\mathcal{L}}{d\vh_{T}},
\end{equation}
which concludes the proof of Proposition~\ref{prop:adjoint-HBNODE}.
\end{proof}

\begin{remark}\label{remark-adjoint-HBNODE-2nd}
Note that we solve the adjoint equation \eqref{eq:adjoint-HBNODE-2nd} from time $t=T$ to $t=t_0$ in the backward propagation. By letting $\tau=T-t$ and $\vb(\tau)=\va(T-\tau)$, we can rewrite \eqref{eq:adjoint-HBNODE-2nd} as follows,
\begin{equation}\label{eq:adjoint-HBNODE-2nd-b}
\frac{d^2\vb(\tau)}{d\tau^2} + \gamma \frac{d\vb(\tau)}{d\tau} = \vb(\tau)\frac{\partial f}{\partial\vh}(\vh(T-\tau),T-\tau,\theta).
\end{equation}
Therefore, the adjoint of the HBNODE is also a HBNODE and they have the same damping parameter.
\end{remark}
\begin{proposition}[Adjoint equations for the first-order HBNODE system]\label{prop:adjoint-HBNODE-1st}
The adjoint states $\va_\vh(t)$ $:=\partial\mathcal{L}/\partial\vh(t)$ and  $\va_\vm(t):=\partial\mathcal{L}/\partial\vm(t)$ for the first-order HBNODE system \eqref{eq:HBNODE-1st} satisfy
\begin{equation}\label{eq:adjoint-HBNODE-1st}
\frac{d\va_\vh(t)}{dt} = -\va_\vm(t) \frac{\partial f}{\partial\vh}(\vh(t),t,\theta); \quad \frac{d\va_\vm(t)}{dt} = -\va_\vh(t) + \gamma \va_\vm(t).
\end{equation}
\end{proposition}
\begin{proof}
The coupled form of HBNODE is a coupled first-order ODE 
{system} of the form
\begin{equation}
\frac{\partial }{\partial t}\begin{bmatrix}
\vh \\ \vm
\end{bmatrix} = \begin{bmatrix}
\vm \\ -\gamma \vm + f(\vh(t), t, \theta)
\end{bmatrix},
\quad 
\begin{bmatrix}
\vh \\ \vm
\end{bmatrix}(t_0) = \begin{bmatrix}
\vh_{t_0} \\ \vm_{t_0}
\end{bmatrix}.
\end{equation}
Denote the final state as 
\begin{equation}
\begin{bmatrix}
\vh(T) \\ \vm(T)
\end{bmatrix}
= \begin{bmatrix}
\vh_{T} \\ \vm_{T}
\end{bmatrix} = z.
\end{equation}
Using the conclusions from the proof of Proposition~\ref{prop:adjoint-HBNODE}, we have the adjoint equation 
\begin{equation}
\frac{\partial \mA(t)}{\partial t} = -\mA(t) \begin{bmatrix}
{\bf 0} & \mI \\
\frac{\partial f}{\partial \vh} & -\gamma \mI
\end{bmatrix},
\quad 
\mA(T) = -\mI,
\quad 
\va (t) = -\frac{d\mathcal{L}}{d\vz} \mA(t).
\end{equation}

Let $\begin{bmatrix}\va_\vh & \va_\vm\end{bmatrix} = \va$, by linearity we have 

\begin{equation}
\frac{\partial \begin{bmatrix}\va_\vh & \va_\vm\end{bmatrix}}{\partial t} = -\begin{bmatrix}\va_\vh & \va_\vm\end{bmatrix} \begin{bmatrix}
{\bf 0} & \mI \\
\frac{\partial f}{\partial \vh} & -\gamma \mI
\end{bmatrix},
\quad 
\begin{bmatrix}\va_\vh(T) & \va_\vm(T)\end{bmatrix} = \begin{bmatrix}\frac{d\mathcal{L}}{d\vh_{T}} & \frac{d\mathcal{L}}{d\vm_{T}}\end{bmatrix},
\end{equation}
which gives us the initial conditions at $t=T$, and the simplified first-order ODE system 
\begin{equation}
\frac{\partial \va_\vh}{\partial t} = - \va_\vm\frac{\partial f}{\partial \vh},
\quad 
\frac{\partial \va_\vm}{\partial t} = -\va_\vh + \gamma \va_\vm,
\end{equation}
concluding the proof of Proposition~\ref{prop:adjoint-HBNODE-1st}.
\end{proof}

\medskip

\begin{remark}\label{remark-adjoint-HBNODE-1st}
Let $\tilde{\va}_\vm(t)=d\va_\vm(t)/dt$, then $\va_\vm(t)$ and $\tilde{\va}_\vm(t)$ satisfies the following first-order heavy-ball ODE system 
\begin{equation}\label{eq:Adjoint-firsrt-order-system}
\frac{d\va_\vm(t)}{dt} =  \tilde{\va}_\vm(t); \quad \frac{d\tilde{\va}_\vm(t)}{dt} = \va_\vm(t)\frac{\partial f}{\partial \vh}(\vh(t),t,\theta) + \gamma \tilde{\va}_\vm(t).
\end{equation}
{Note that we solve this system backward in time in back-propagation.}
Moreover, we have $\va_{\vh}(t) = \gamma \va_\vm(t) - \tilde{\va}_\vm(t)$.
\end{remark}
Similar to \cite{norcliffe2020_sonode}, we use the coupled first-order HBNODE system \eqref{eq:HBNODE-1st} and its adjoint first-order HBNODE system \eqref{eq:adjoint-HBNODE-1st} for practical implementation, since the entangled representation permits faster computation \cite{norcliffe2020_sonode} of the gradients of the coupled ODE systems.

\subsubsection{Generalized heavy-ball neural ODEs}\label{sec:GHBNODEs}
\begin{wrapfigure}{r}{.4\textwidth}\vspace{-0.5cm}
\includegraphics[width=\linewidth]{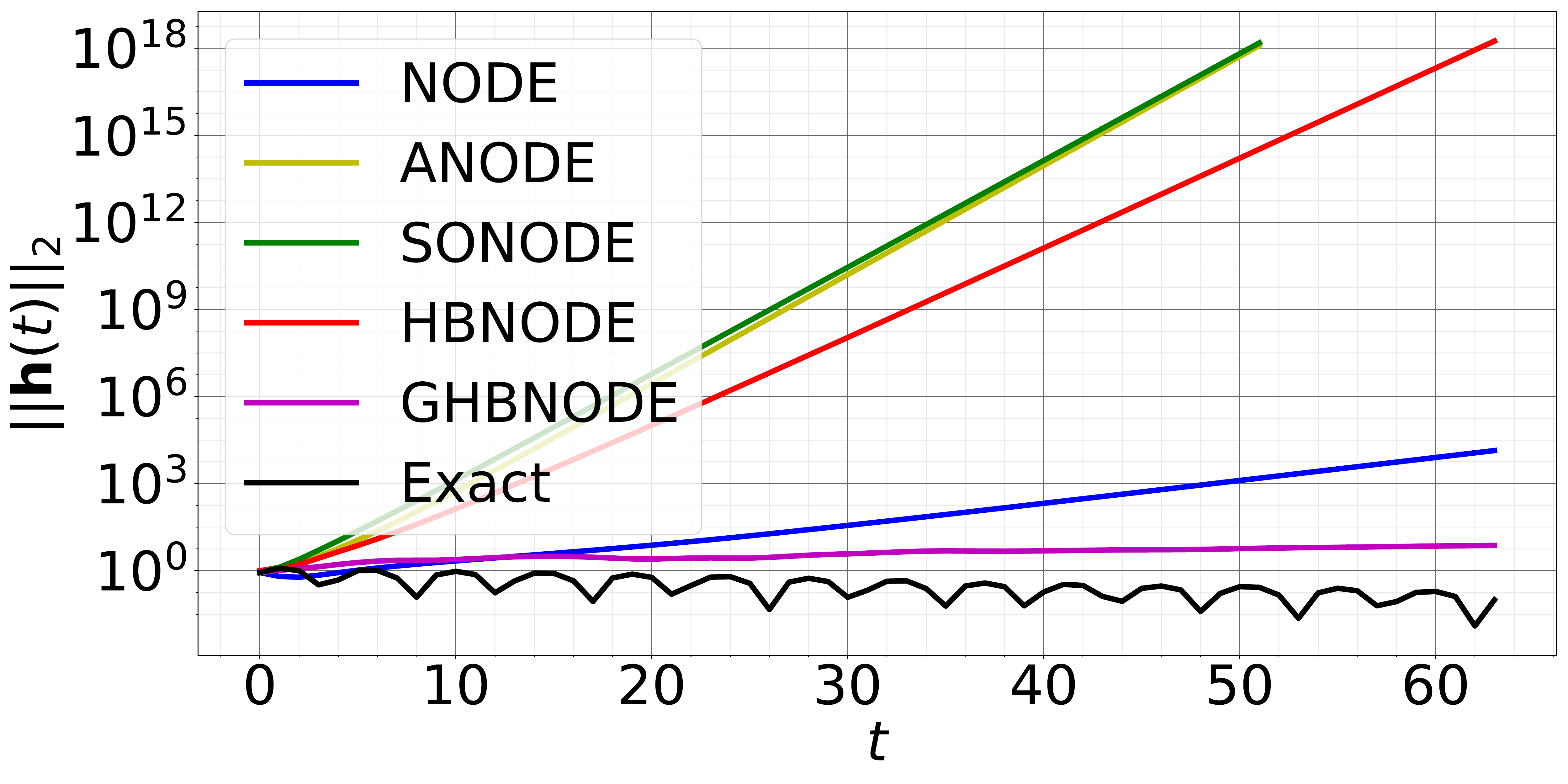}
\vspace{-0.2in}
\caption{Contrasting $\vh(t)$ for different models. $\vh(t)$ in ANODE, SONODE, and HBNODE grows much faster than that in NODE. GHBNODE controls the growth of $\vh(t)$ effectively when $t$ is large.}\label{fig:blow-up}\vspace{-0.5cm}
\end{wrapfigure}
In this part, we propose a generalized version of HBNODE (GHBNODE), see \eqref{eq:GHBNODE}, to mitigate the potential blow-up issue in training ODE-based models. We observe that $\vh(t)$ of ANODEs \cite{NEURIPS2019_21be9a4b}, SONODEs \cite{norcliffe2020_sonode}, and HBNODEs \eqref{eq:HBNODE-1st} usually grows much faster than that of NODEs. The fast growth of $\vh(t)$ can lead to finite-time blow up. As an illustration, we compare the performance of NODE, ANODE, SONODE, HBNODE, and GHBNODE on the Silverbox task as in \cite{norcliffe2020_sonode}. The goal of the task is to learn the voltage of an electronic circuit that resembles a Duffing oscillator, where the input voltage $V_1(t)$ is used to predict the output $V_2(t)$. Similar to the setting in \cite{norcliffe2020_sonode}, we first augment ANODE by 1 dimension with 0-augmentation and augment SONODE, HBNODE, and GHBNODE with a dense network.
We use a simple dense layer to parameterize $f$ for all five models, with an extra input term for $V_1(t)$\footnote{Here, we exclude an $\vh^3$ term that appeared in the original Duffing oscillator model because including it would result in finite-time explosion.}. For both HBNODE and GHBNODE, we set the damping parameter $\gamma$ to be ${\rm sigmoid}(-3)$. For GHBNODE \eqref{eq:GHBNODE} below, we set $\sigma(\cdot)$ to be the \texttt{hardtanh} function with bound $[-5,5]$ and $\xi=\ln(2)$. As shown in Fig.~\ref{fig:blow-up}, compared to the vanilla NODE, the $\ell_2$ norm of $\vh(t)$ grows much faster when a higher order NODE is used, which leads to blow-up during training. Similar issues arise in the time series experiments (see Section~\ref{sec:sequential-modeling}), where SONODE blows up during long term integration in time, and HBNODE suffers from the same issue with some initialization.

To alleviate the problem above, we propose the following GHBNODE
\begin{equation}\label{eq:GHBNODE}
\begin{aligned}
\frac{d\vh(t)}{d t} &= \sigma(\vm(t)),\\
\frac{d\vm(t)}{d t} &= -\gamma\vm(t) + f(\vh(t),t,\theta) - \xi\vh(t), 
\end{aligned}
\end{equation}
where $\sigma(\cdot)$ is a nonlinear activation, which is set as $\tanh$ in our experiments. The positive hyperparameters $\gamma,\xi>0$ are tunable or learnable. In the trainable case, we let $\gamma = \epsilon \cdot \text{sigmoid}(\omega)$ as in HBNODE, and $\xi = \text{softplus}(\chi)$ to ensure that $\gamma,\xi \geq 0$. Here, we integrate two main ideas into the {design of} GHBNODE: (i) We incorporate the gating mechanism used in LSTM \cite{hochreiter1997long} and GRU \cite{cho2014learning}, which can suppress the aggregation of $\vm(t)$;
(ii) Following the idea of skip connection \cite{he2016identity}, we add the term $\xi\vh(t)$ into the governing equation of $\vm(t)$, which benefits training and generalization of GHBNODEs. Fig.~\ref{fig:blow-up} shows that GHBNODE can indeed control the growth of $\vh(t)$ effectively. 
\begin{proposition}[Adjoint equations for GHBNODEs]\label{prop:adjoint-GHBNODE}
The adjoint states $\va_\vh(t):=\partial\mathcal{L}/\partial\vh(t)$, $\va_\vm(t):=\partial\mathcal{L}/\partial\vm(t)$ for the GHBNODE \eqref{eq:GHBNODE} 
satisfy the following first-order ODE system
\begin{equation}\label{eq:adjoint-GHBNODE-2nd}
\begin{aligned}
\frac{\partial \va_\vh(t)}{\partial t} &= - \va_\vm(t)\Big(\frac{\partial f}{\partial \vh}(\vh(t),t,\theta) - \xi\mI\Big)\\
\frac{\partial \va_\vm(t)}{\partial t} &= -\va_\vh(t) \sigma'(\vm(t)) + \gamma \va_\vm(t).
\end{aligned}
\end{equation}
\end{proposition}
\begin{proof}
GHBNODE can be written as the following first-order ODE system
\begin{equation}\label{eq:264}
\frac{\partial }{\partial t}\begin{bmatrix}
\vh \\ \vm
\end{bmatrix} = \begin{bmatrix}
\sigma(\vm) \\ -\gamma \vm + f(\vh(t), t, \theta) - \xi \vh(t)
\end{bmatrix},
\quad 
\begin{bmatrix}
\vh \\ \vm
\end{bmatrix}(t_0) = \begin{bmatrix}
\vh_{t_0} \\ \vm_{t_0}
\end{bmatrix}.
\end{equation}
Denote the final state as $\vz_T:=[\vh_T \vm_T]$.
We have the adjoint equation 
\begin{equation}
\frac{\partial \mA(t)}{\partial t} = -\mA(t) \begin{bmatrix}
{\bf 0} & \sigma'(\vm) \\
\frac{\partial f}{\partial \vh}-\xi\mI & -\gamma \mI
\end{bmatrix},
\quad 
\mA(T) = -\mI,
\quad 
\va (t) = -\frac{d\mathcal{L}}{d\vz_T} \mA(t).
\end{equation}
Let $\begin{bmatrix}\va_\vh & \va_\vm\end{bmatrix} = \va$, by linearity we have 
\begin{equation}
\begin{aligned}
\frac{\partial \begin{bmatrix}\va_\vh & \va_\vm\end{bmatrix}}{\partial t} &= -\begin{bmatrix}\va_\vh & \va_\vm\end{bmatrix} \begin{bmatrix}
{\bf 0} & \sigma'(\vm) \\
\frac{\partial f}{\partial \vh}-\xi\mI & -\gamma \mI
\end{bmatrix},\\
\begin{bmatrix}\va_\vh(T) & \va_\vm(T)\end{bmatrix} &= \begin{bmatrix}\frac{d\mathcal{L}}{d\vh_{T}} & \frac{d\mathcal{L}}{d\vm_{T}}\end{bmatrix},
\end{aligned}
\end{equation}
which gives us the initial conditions at $t=T$, and the simplified first-order ODE system
\begin{equation}
\frac{\partial \va_\vh}{\partial t} = - \va_\vm \Big(\frac{\partial f}{\partial \vh}-\xi\mI\Big),
\quad 
\frac{\partial \va_\vm}{\partial t} = -\va_\vh \sigma'(\vm) + \gamma \va_\vm,
\end{equation}
concluding the proof of Proposition~\ref{prop:adjoint-GHBNODE}.
\end{proof}
Though the adjoint state of the GHBNODE \eqref{eq:adjoint-GHBNODE-2nd} does not satisfy the exact  heavy-ball ODE, based on our empirical study, it also significantly reduces the backward NFEs.

\subsection{Learning long-term dependencies -- Vanishing 
gradient 
}\label{sec:lowerbounds}
As mentioned in Section~\ref{sec-RNN}, the vanishing gradient is the main bottleneck for training RNNs with long-term dependencies. 
As the continuous analogue of RNN, NODEs as well as their hybrid ODE-RNN models, may {also} suffer from vanishing in the adjoint state $\va(t):=\partial\mathcal{L}/\partial\vh(t)$ \cite{lechner2020learning}. When {the} vanishing gradient issue happens, $\va(t)$ goes to ${\bf 0}$ quickly as $T-t$ increases, then $d\mathcal{L}/d\theta$ in \eqref{eq:dL-dtheta} will be independent of these $\va(t)$. We have the following expressions for the adjoint states of the NODE and HBNODE:
\begin{itemize}[leftmargin=*]
\item For NODE, we have 
\begin{equation}\label{eq:NODE-gradient}
\frac{\partial\mathcal{L}}{\partial\vh_t} = \frac{\partial\mathcal{L}}{\partial\vh_T}\frac{\partial\vh_T}{\partial\vh_t} = \frac{\partial\mathcal{L}}{\partial\vh_T}\exp\Big\{-\int_T^t\frac{\partial f}{\partial \vh}(\vh(s),s,\theta)ds\Big\}.
\end{equation}

\item For GHBNODE\footnote{HBNODE can be seen as a special GHBNODE with $\xi=0$ and $\sigma$ be the identity map.
}, 
from \eqref{eq:adjoint-HBNODE-1st} we can derive 
\begin{equation}\label{eq:HBNODE-gradient}
{\small \begin{bmatrix}
\frac{\partial\mathcal{L}}{\partial\vh_t}&\hspace{-0.1in} \frac{\partial\mathcal{L}}{\partial\vm_t} 
\end{bmatrix} = 
\begin{bmatrix}
\frac{\partial\mathcal{L}}{\partial\vh_T}&\hspace{-0.1in}  \frac{\partial\mathcal{L}}{\partial\vm_T} 
\end{bmatrix}
\begin{bmatrix}
\frac{\partial\vh_T}{\partial\vh_t} &\hspace{-0.1in} \frac{\partial\vh_T}{\partial\vm_t}\\
\frac{\partial\vm_T}{\partial\vh_t} &\hspace{-0.1in}
\frac{\partial\vm_T}{\partial\vm_t}\\
\end{bmatrix}=
\begin{bmatrix}
\frac{\partial\mathcal{L}}{\partial\vh_T} \  \frac{\partial\mathcal{L}}{\partial\vm_T} \end{bmatrix}\exp\Big\{-\underbrace{\int_T^t\begin{bmatrix}
{\bf 0} &\hspace{-0.05in} \frac{\partial \sigma}{\partial \vm}\\
\big(\frac{\partial f}{\partial\vh}-\xi\mI\big) &\hspace{-0.05in}
-\gamma\mI
\end{bmatrix}ds}_{:=\mM} \Big\}.}
\end{equation}
\end{itemize}

Note that the matrix exponential is directly related to its eigenvalues. By
Schur decomposition, there exists an orthogonal matrix $\mQ$ and {an} upper triangular matrix $\mU$, where the diagonal entries of $\mU$ are eigenvalues of $\mQ$ ordered by their real parts, such that
\begin{equation}
    -\mM = \mQ\mU\mQ^\top \Longrightarrow 
    \exp\{-\mM\} = \mQ\exp\{\mU\}\mQ^\top. 
\end{equation}
Let $\vv^\top := \begin{bmatrix}\frac{\partial\mathcal{L}}{\partial\vh_T} \ \frac{\partial\mathcal{L}}{\partial\vm_T} \end{bmatrix}\mQ$, then \eqref{eq:HBNODE-gradient} can be rewritten as 
\begin{equation}\label{eq:sgfsfxvb}
\begin{aligned}
\begin{bmatrix}
\frac{\partial\mathcal{L}}{\partial\vh_t} \  \frac{\partial\mathcal{L}}{\partial\vm_t} 
\end{bmatrix} &=
\begin{bmatrix}
\frac{\partial\mathcal{L}}{\partial\vh_T} \  \frac{\partial\mathcal{L}}{\partial\vm_T} \end{bmatrix}\exp\{-\mM\}\\
&= \begin{bmatrix}
\frac{\partial\mathcal{L}}{\partial\vh_T} \  \frac{\partial\mathcal{L}}{\partial\vm_T} \end{bmatrix}\mQ\exp\{\mU\}\mQ^\top = 
\vv^\top \exp\{\mU\}\mQ^\top.
\end{aligned}
\end{equation}
Taking the $\ell_2$ norm in \eqref{eq:sgfsfxvb} and dividing both sides by $\norm{\begin{bmatrix}\frac{\partial\mathcal{L}}{\partial\vh_T} \ \frac{\partial\mathcal{L}}{\partial\vm_T} \end{bmatrix}}_2$, we have
\begin{equation}\label{eq:ratio:HBNODE}
\frac{\norm{\begin{bmatrix}
\frac{\partial\mathcal{L}}{\partial\vh_t} \  \frac{\partial\mathcal{L}}{\partial\vm_t} 
\end{bmatrix}}_2}
{\norm{\begin{bmatrix}\frac{\partial\mathcal{L}}{\partial\vh_T} \ \frac{\partial\mathcal{L}}{\partial\vm_T} \end{bmatrix}}_2}
= 
\frac{\norm{\vv^\top \exp\{\mU\}\mQ^\top}_2}{\norm{\vv^\top\mQ^\top}_2}
= 
\frac{\norm{\vv^\top \exp\{\mU\}}_2}{\norm{\vv}_2}
=
\norm{\ve^\top \exp\{\mU\}}_2,
\end{equation}
i.e., $\norm{\begin{bmatrix}
\frac{\partial\mathcal{L}}{\partial\vh_t} \  \frac{\partial\mathcal{L}}{\partial\vm_t} 
\end{bmatrix}}_2=\norm{\ve^\top \exp\{\mU\}}_2 \norm{\begin{bmatrix}\frac{\partial\mathcal{L}}{\partial\vh_T} \ \frac{\partial\mathcal{L}}{\partial\vm_T} \end{bmatrix}}_2$ where $\ve = {\vv}/{\norm{\vv}_2}$.

\begin{proposition}\label{lemma-eigan-M}
The eigenvalues of $-\mM$ can be paired so that the sum of each pair equals $(t-T)\gamma$. 
\end{proposition} 
\begin{proof}
Let $\mF = \frac{1}{t-T}\int_T^t\frac{\partial f}{\partial \vh}(\vh(s), s, \theta)ds - \xi\mI$, ${\mJ} = \frac{1}{t-T}\int_T^t\frac{\partial \sigma}{\partial \vm}(\vm(s))ds$, and ${\mH} = \frac{1}{t-T} \mM$, then we have the following equation
\begin{equation}
\mH = \frac{1}{t-T}\mM = \begin{bmatrix}0 & {\mJ} \\ {\mF} & -\gamma \mI\end{bmatrix}.
\end{equation}
As $(\lambda+\gamma)\mI$ commutes with any matrix ${\mF}$, the characteristics polynomials of ${\mH}$ and ${\mJ\mF}$ satisfy the relation
\begin{equation}
\begin{aligned}
    ch_{\mH}(\lambda) &= \det(\lambda \mI - {\mH}) = \det \begin{bmatrix}\lambda \mI & -{\mJ} \\ -{\mF} & (\lambda + \gamma) \mI\end{bmatrix}\\
    &= det(\lambda(\lambda + \gamma)\mI - {\mJ\mF}) = -ch_{\mJ\mF}(\lambda(\lambda+\gamma)).
\end{aligned}
\end{equation}
Since the characteristics polynomial of ${\mJ\mF}$ splits in the field $\mathbb{C}$ of complex numbers, i.e. $ch_{\mJ\mF}(x) = \prod_{i=1}^n (x - \lambda_{{\mJ\mF},i})$, we have
\begin{equation}
    ch_{\mH}(\lambda) = -ch_{\mJ\mF}(\lambda(\lambda+\gamma)) = -\prod_{i=1}^n (\lambda(\lambda+\gamma) - \lambda_{{\mJ\mF},i}).
\end{equation}
Therefore, the eigenvalues of ${\mH}$ 
{appear} in $n$ pairs with each pair 
{satisfying}
the 
quadratic equation 
\begin{equation}
    \lambda(\lambda+\gamma) - \lambda_{{\mJ\mF},i} = 0.
\end{equation}
By Vieta's formulas, the sum of these pairs are all $-\gamma$. Therefore, the eigenvalues of $\mM$ comes in $n$ pairs and the sum of each pair is 
$-(t-T)\gamma$, which finishes the proof of Proposition~\ref{lemma-eigan-M}.
\end{proof}
 
For a given constant $a>0$, we can group the upper triangular matrix $\exp\{\mU\}$ as follows
\begin{equation}
\exp\{\mU\}:= \begin{bmatrix}\exp\{\mU_L\} & \mP 
\\ {\bf 0} & \exp\{\mU_V\}\end{bmatrix},
\end{equation}
where the diagonal of $\mU_L$ ($\mU_V$) contains eigenvalues of $-\mM$ that are no less (greater) than $(t-T)a$. Then, we have $\|\ve^\top \exp\{\mU\}\|_2 \geq \|\ve_L^\top \exp\{\mU_L\}\|_2$ where the vector $\ve_L$ denotes the first $m$ columns of $\ve$ with $m$ be the number of columns of $\mU_L$. By choosing $0\leq \gamma \leq 2a$, for every pair of eigenvalues of $-\mM$ there is at least one eigenvalue whose real part is no less than $(t-T)a$. Therefore, $\exp\{\mU_L\}$ decays at a rate at most $(t-T)a$, and the dimension of $\mU_L$ is at least $N\times N$. We avoid exploding gradients by clipping the $\ell_2$ norm of the adjoint states similar to that used for training RNNs.

In contrast, all eigenvalues of the matrix $\int_T^t {\partial f}/{\partial \vh}ds$ in \eqref{eq:NODE-gradient} {for NODE} can be very positive or negative, resulting in exploding or vanishing
gradients. As an illustration, we consider the benchmark Walker2D kinematic simulation task that requires learning long-term dependencies effectively \cite{lechner2020learning,brockman2016openai}. We train ODE-RNN \cite{NEURIPS2019_42a6845a} and (G)HBNODE-RNN on this benchmark dataset, and the detailed experimental settings are provided in Section~\ref{sec:sequential-modeling}. Figure~\ref{fig:gradient-vanishing-ODE} plots $\|\partial\mathcal{L}/\partial\vh_{t}\|_2$ for ODE-RNN and $\|[\partial\mathcal{L}/\partial \vh_{t}\ \partial\mathcal{L}/\partial\vm_{t}] \|_2$
{for (G)HBNODE-RNN}, showing that the adjoint state of ODE-RNN vanishes quickly, while {that of} (G)HBNODE-RNN 
{does not}
vanish even when the gap between $T$ and $t$ is very large.

\begin{figure}
\begin{center}
  \includegraphics[clip, trim=0.1cm 0.1cm 3.9cm 0.1cm,height=3.5cm]{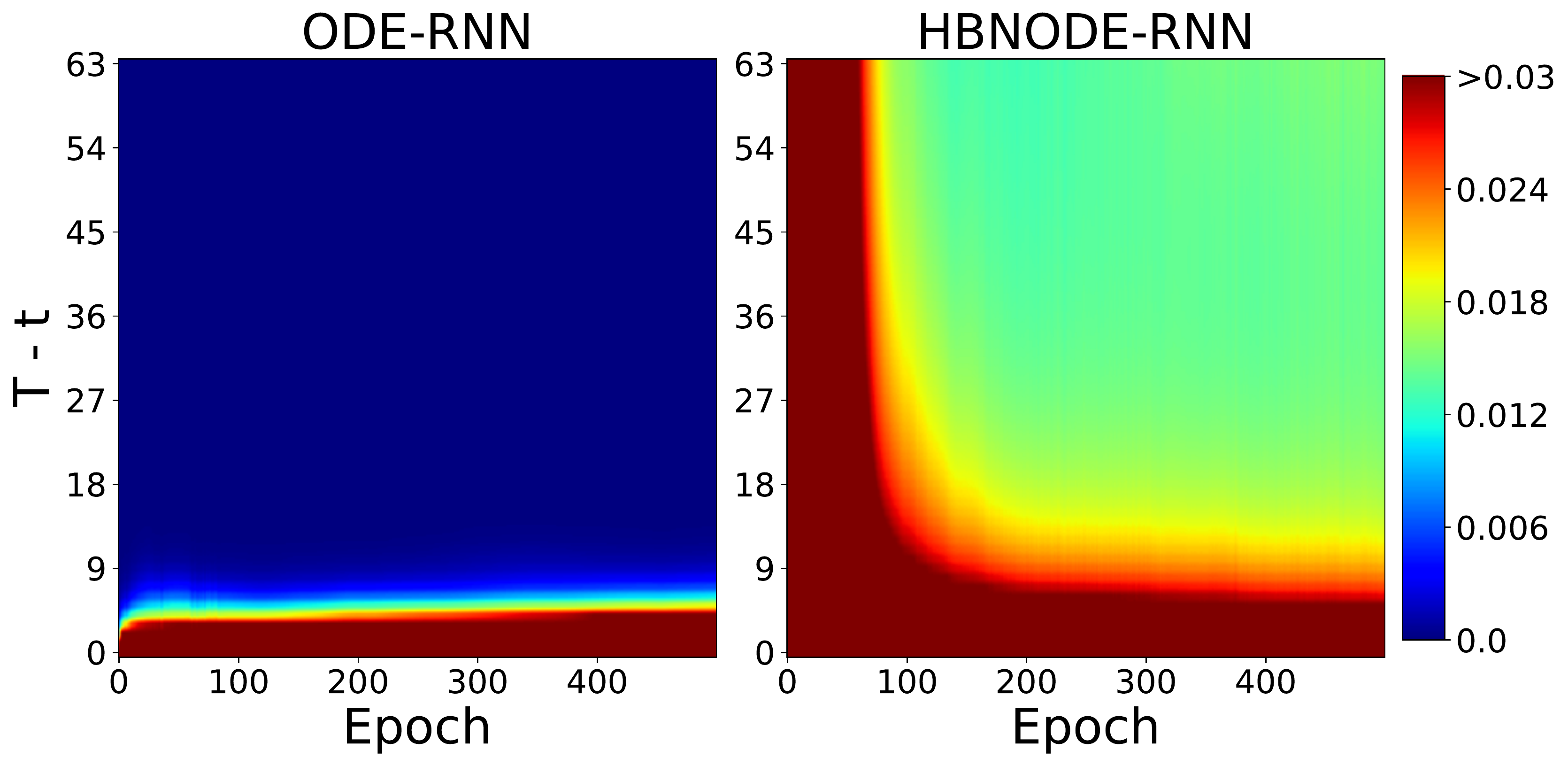}
  \includegraphics[clip, trim=16.5cm 0.1cm 0.1cm 0.1cm,height=3.5cm]{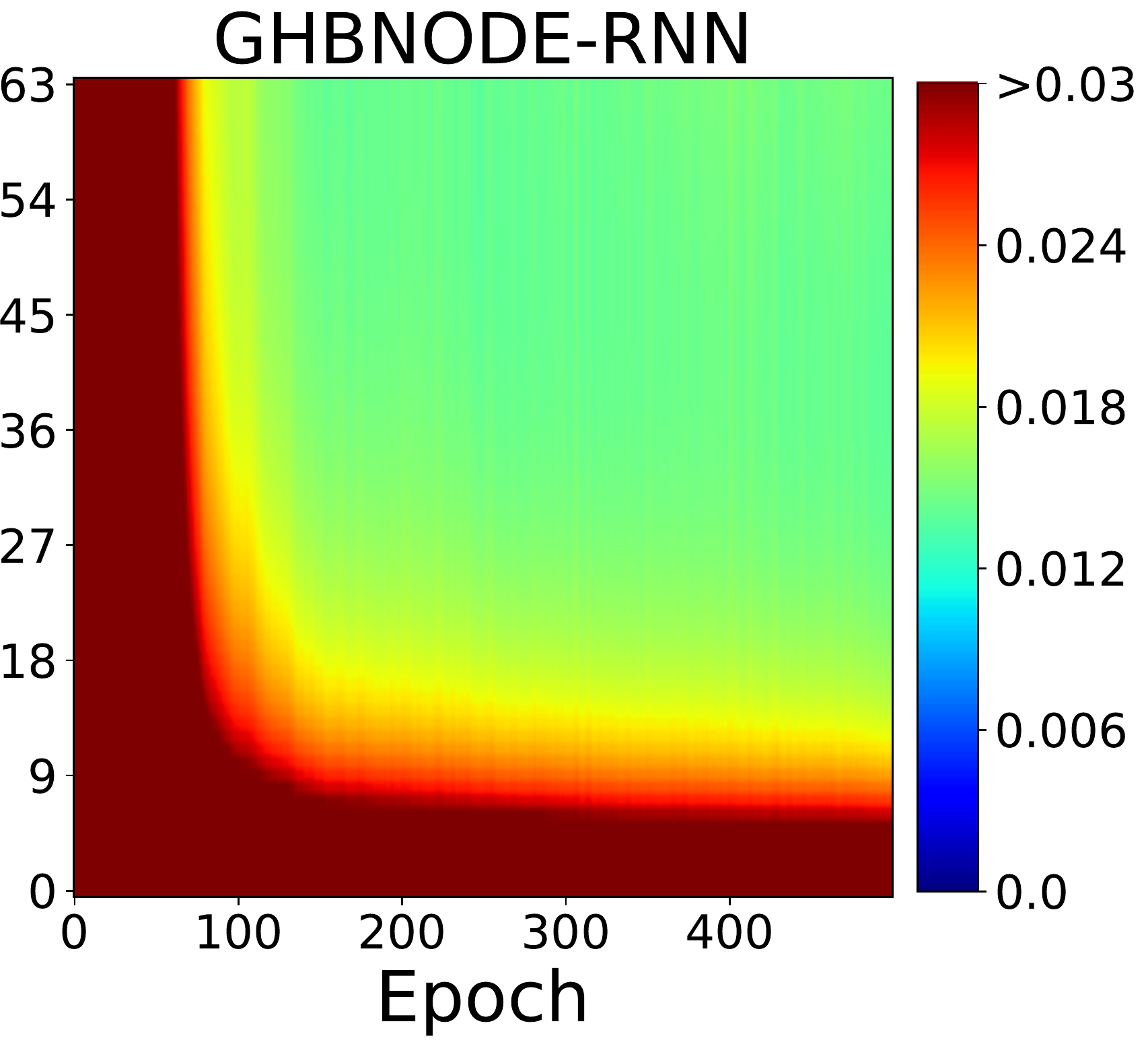}
  \vspace{-2mm}
   \caption{Plot of the the $\ell_2$-norm of the adjoint states for ODE-RNN and (G)HBNODE-RNN back-propagated from the last time stamp. The adjoint state of ODE-RNN vanishes quickly when the gap between the final time $T$ and intermediate time $t$ becomes larger, while the adjoint states of (G)HBNODE-RNN decays much more slowly. This implies that (G)HBNODE-RNN is more effective in learning long-term dependency
   than ODE-RNN. 
  }\label{fig:gradient-vanishing-ODE}
\end{center}
\end{figure}

\subsection{Experimental Results}\label{sec:experiments}
In this section, we compare the performance of the proposed HBNODE and GHBNODE with existing ODE-based models, including NODE \cite{chen2018neural}, ANODE \cite{NEURIPS2019_21be9a4b}, and SONODE \cite{norcliffe2020_sonode} on the benchmark 
point cloud separation, 
image classification, learning dynamical systems, and kinematic simulation. 
For all the experiments, we use Adam \cite{kingma2014adam} as the benchmark optimization solver (the learning rate and batch size for each experiment 
{are}
listed in Table~\ref{Tab:batch-size-learning-rate}).
For HBNODE and GHBNODE, we 
{set}
{$\gamma={\rm sigmoid}(\theta)$, }
{where}
$\theta$ 
{is}
a trainable weight initialized as $\theta=-3$. 

\begin{table}[!ht]
\fontsize{8.0}{8.0}\selectfont
\centering
\begin{threeparttable}
\caption{The batch size and learning rate for different datasets.}\label{Tab:batch-size-learning-rate}
\begin{tabular}{cccccc}
\toprule[1.0pt]
Dataset& Point Cloud& MNIST  & CIFAR10 &  Plane Vibration &  Walker2D\cr
\midrule[0.8pt]
Batch Size  & 50 & 64 & 64 & 64 & 256 \cr
Learning Rate  & 0.01 & 0.001 & 0.001 & 0.0001 & 0.003 \cr
\bottomrule[1.0pt]
\end{tabular}
\end{threeparttable}\vspace{-0.1in}
\end{table}

\begin{figure*}[!ht]\vspace{-2mm}
\centering
\begin{tabular}{cc}
\hskip -0.3cm
\includegraphics[clip, trim=0.1cm 0.1cm 0.1cm 1.5cm,width=0.46\linewidth]{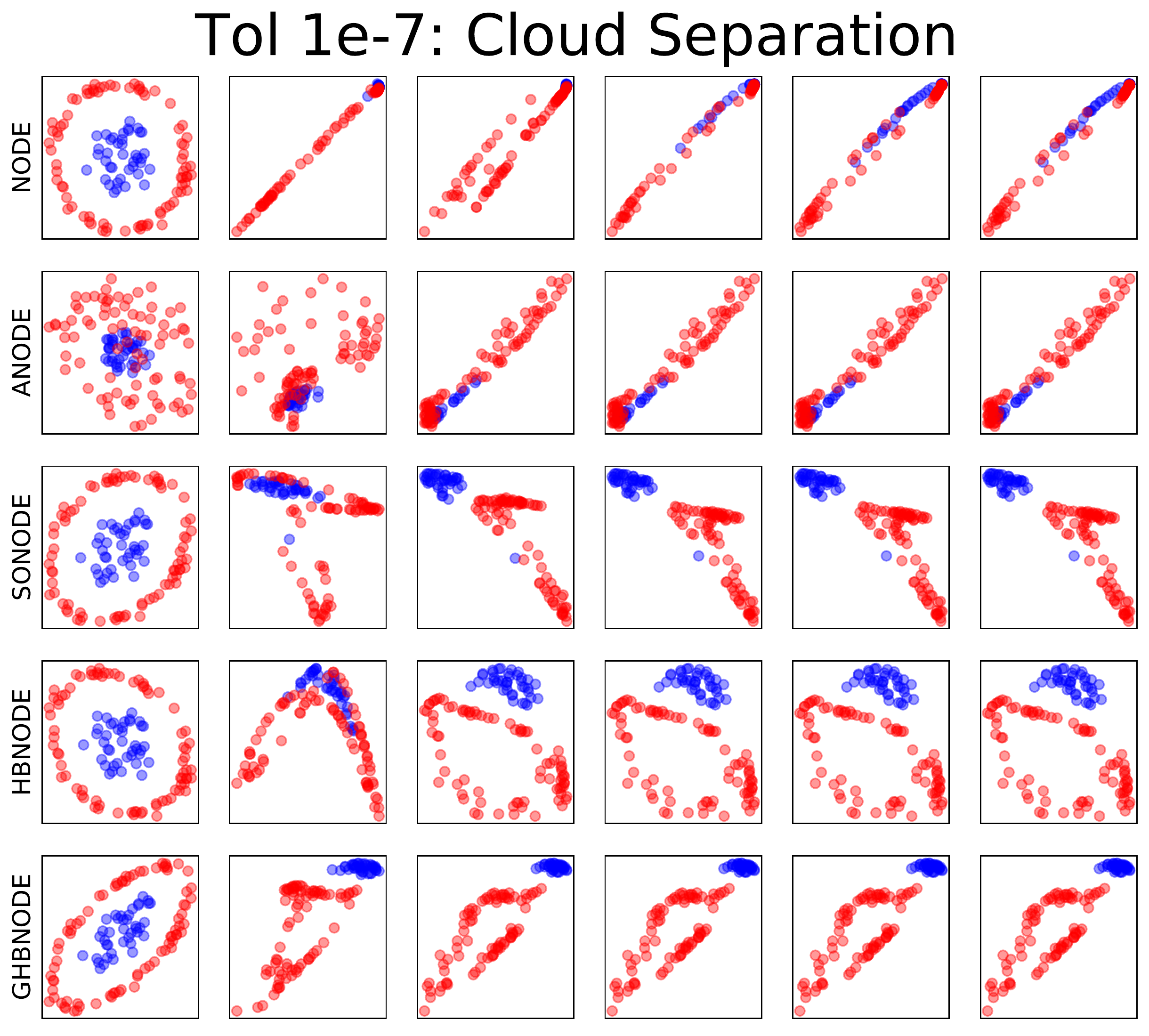}&
\hskip -0.4cm\includegraphics[clip, trim=0.1cm 0.1cm 0.1cm 0.1cm,width=0.45\linewidth]{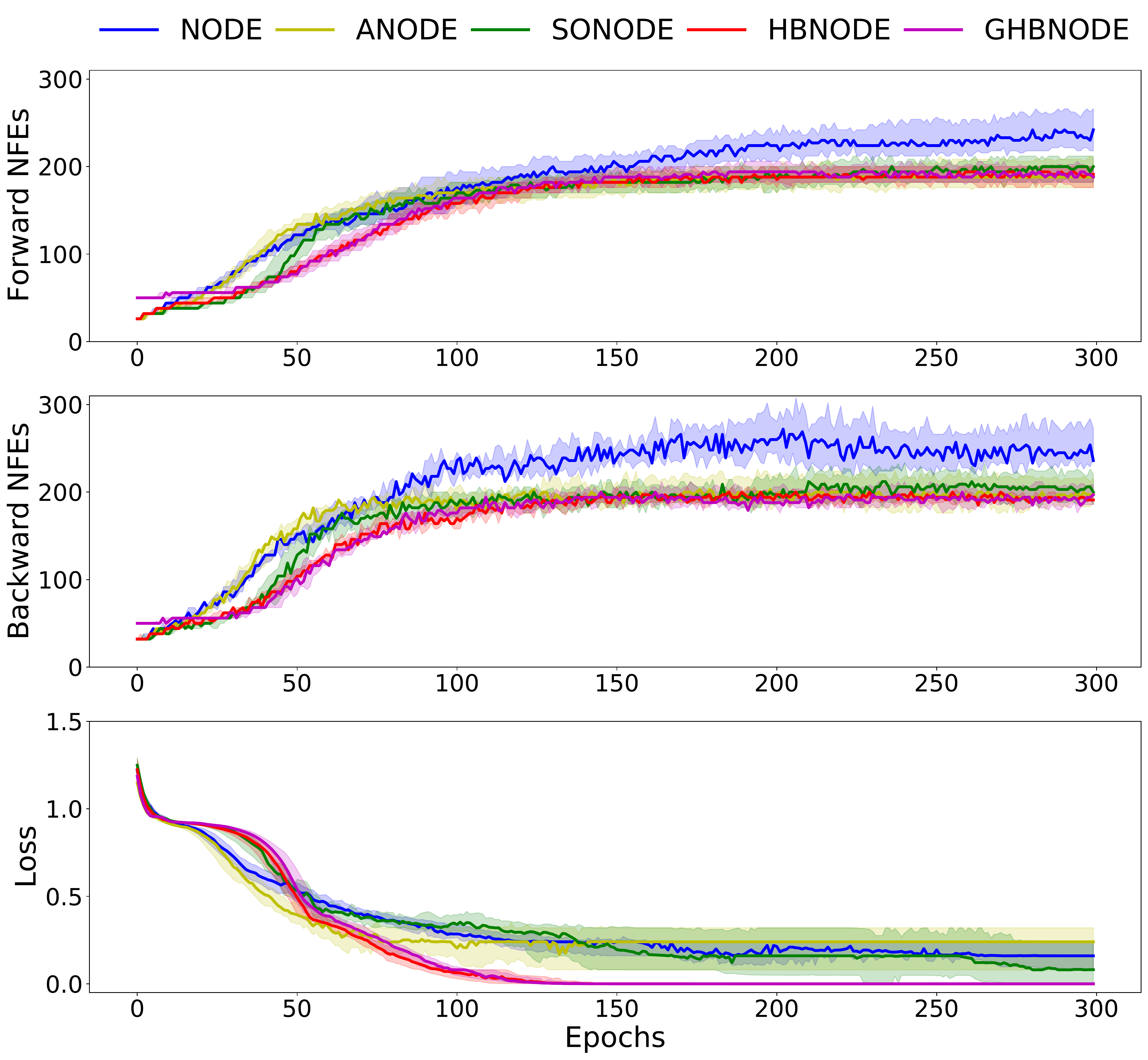}\\
\end{tabular}
\vskip -0.3cm
\caption{Comparison between NODE, ANODE, SONODE, HBNODE, and GHBNODE for two-dimensional point cloud separation. HBNODE and GHBNODE converge better and require less 
{NFEs}
in both forward and backward propagation than the other benchmark models.
}
\label{fig:point:cloud-separation}\vspace{-3mm}
\end{figure*}

\subsubsection{Point cloud separation}\label{sec:point-cloud}
In this subsection, we consider the two-dimensional point cloud separation benchmark.
A total of $120$ points are sampled, 
{in which}
$40$ points are drawn uniformly from the circle $\|\vr\|<0.5$, and $80$ points are drawn uniformly from the annulus $0.85<\|\vr\|<1.0$. This experiment aims to learn effective features to classify these two point clouds. 
{Following \cite{NEURIPS2019_21be9a4b},}
we use a three-layer neural network 
to parameterize the 
{right-hand side}
of each ODE-based model, 
{integrate}
the ODE-based model from $t_0=0$ to $T=1$, and
{pass the integration results to a dense layer to generate the classification results.}
We set the size 
of hidden layers so that the models have similar 
{sizes,}
and the number of 
{parameters}
of NODE, ANODE, SONODE, HBNODE, and GHBNODE are $525$, $567$, $528$, $568$, and $568$, respectively.
To avoid the effects of numerical error of the black-box ODE solver we set tolerance of ODE solver to be $10^{-7}$.
Figure~\ref{fig:point:cloud-separation} plots a randomly selected evolution of the point cloud separation for each model; we also compare the forward and backward NFEs and the training loss of these models 
(100 independent runs). 
HBNODE and GHBNODE improve training as the training loss consistently goes to zero over different runs, while ANODE and SONODE often get stuck at local 
{minima}, 
and NODE cannot separate the point cloud 
since it preserves the topology \cite{NEURIPS2019_21be9a4b}.

\subsubsection{Image classification}\label{sec:image-classification}
We compare the performance of HBNODE and GHBNODE with the existing ODE-based models {on} 
MNIST and CIFAR10 classification {tasks} using the same setting as 
in \cite{NEURIPS2019_21be9a4b}.
We parameterize $f(\vh(t),t,\theta)$ 
using a 3-layer convolutional network for each ODE-based model, and the total number of parameters 
{for} each model is listed in Table~\ref{Tab:num-params-image-classification}.
For a given input image of the size $c\times h\times w$, 
we first augment the number of channel from $c$ to $c+p$ 
with the augmentation dimension $p$ 
{dependent}
on each method\footnote{We set $p=0, 5, 4, 4, 5/0,10,9,9,9$ on MNIST/CIFAR10 
for NODE, ANODE, SONODE, HBNODE, and GHBNODE, respectively.}. Moreover, for SONODE, HBNODE and GHBNODE, we further include velocity or momentum with the same shape as the augmented state. 

\begin{table}[!ht]\vspace{-0.1in}
\fontsize{8.0}{8.0}\selectfont
\centering
\begin{threeparttable}
\caption{The number of parameters for each models for image classification.}\label{Tab:num-params-image-classification}
\begin{tabular}{cccccc}
\toprule[1.0pt]
Model  & NODE  & ANODE  & SONODE &HBNODE & GHBNODE \cr
\midrule[0.8pt]
\#Params (MNIST)    & 85,315 & 85,462 & 86,179 & {85,931} & {85,235} \cr
\#Params (CIFAR10)  & 173,611 & 172,452 & 171,635 & 172,916 & 172,916 \cr
\bottomrule[1.0pt]
\end{tabular}
\end{threeparttable}\vspace{-0.2in}
\end{table}

\begin{figure}
\begin{center}
\begin{tabular}{cccc}
\hskip -0.3cm
\includegraphics[width=0.2\columnwidth]{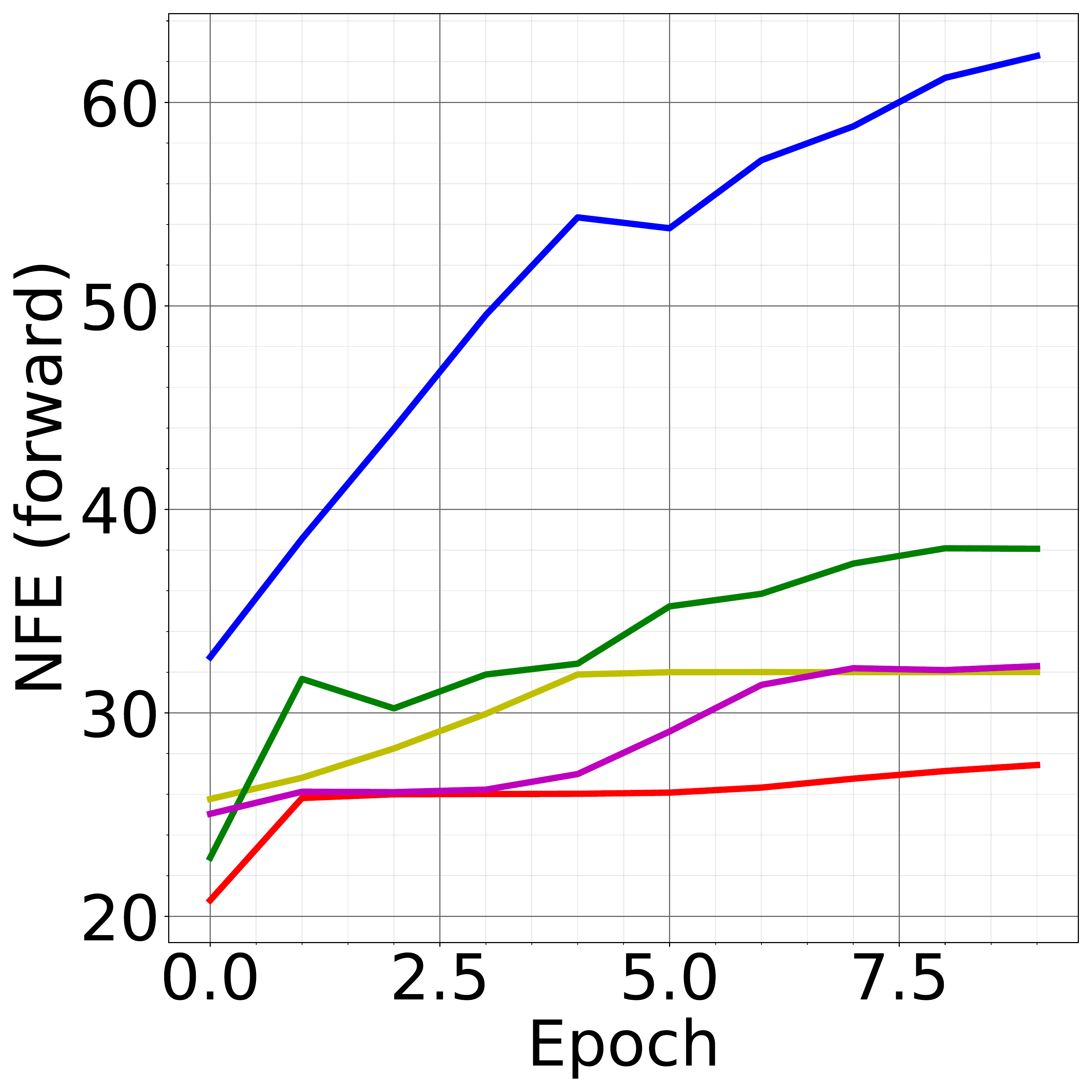}&
\hskip -0.3cm
\includegraphics[width=0.2\columnwidth]{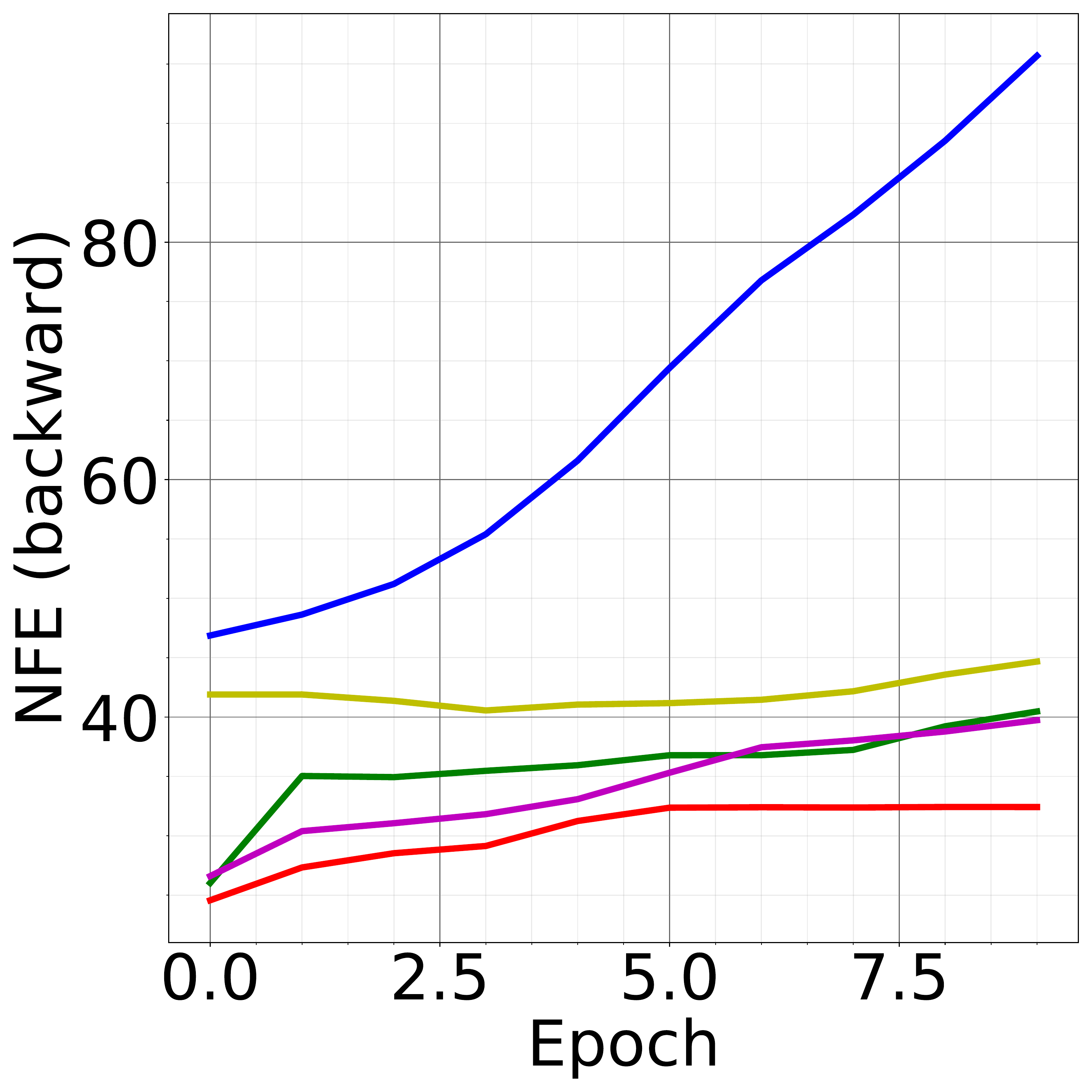}& 
\hskip -0.3cm
\includegraphics[width=0.2\columnwidth]{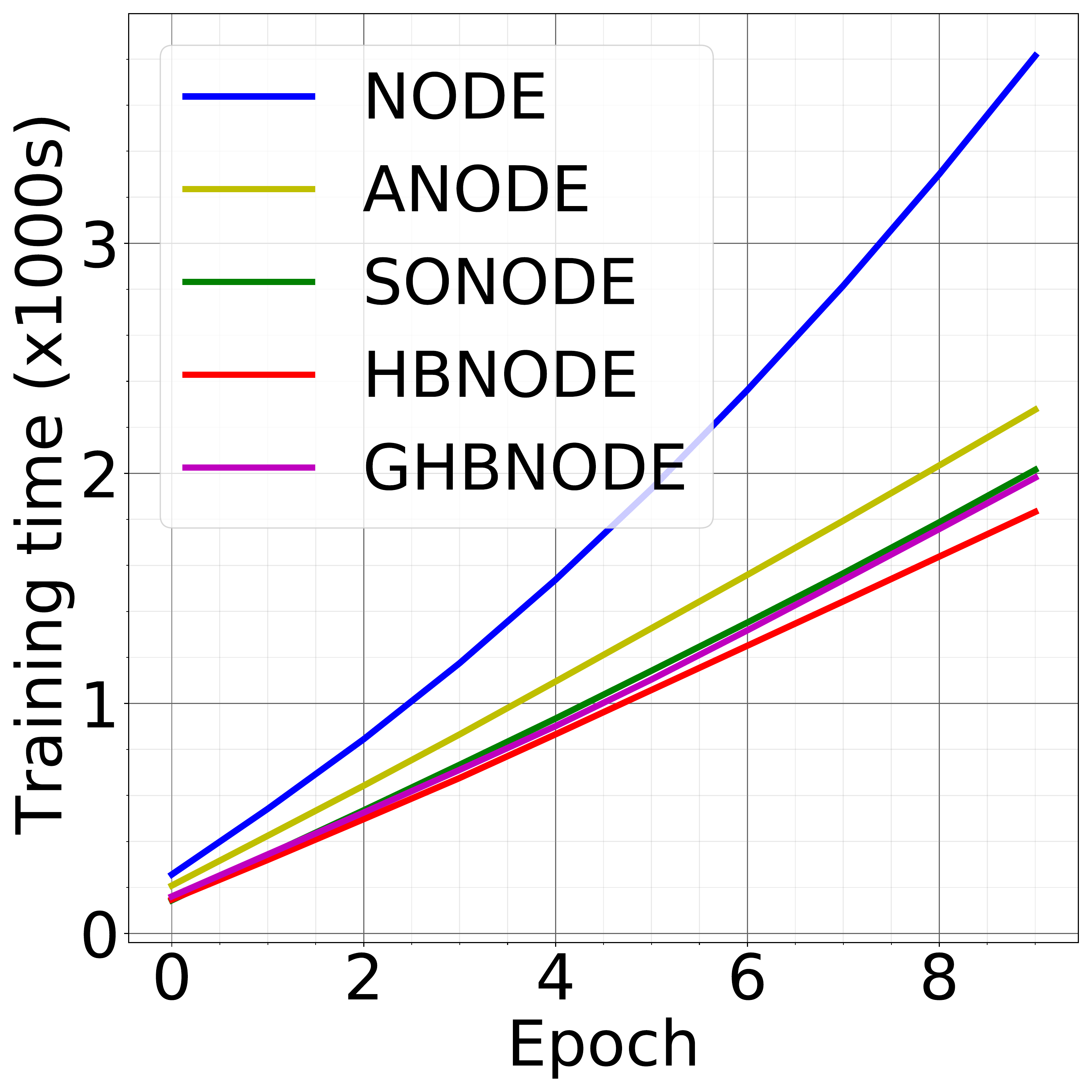} &
\hskip -0.3cm
\includegraphics[width=0.195\columnwidth]{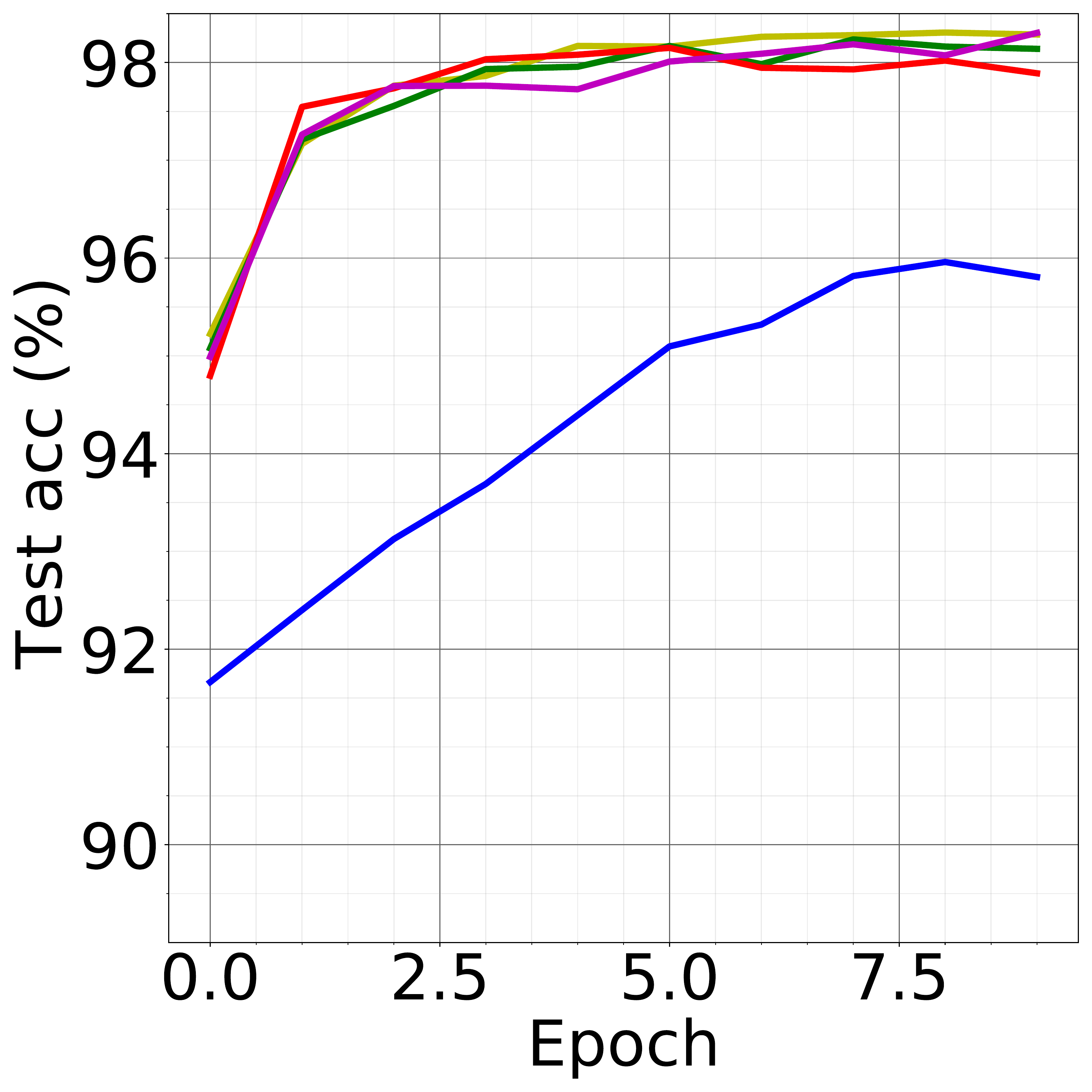}\\
  \end{tabular}
  \end{center}
  \vskip -0.2in
  \caption{ Contrasting NODE \cite{chen2018neural}, ANODE \cite{NEURIPS2019_21be9a4b}, SONODE \cite{norcliffe2020_sonode}, HBNODE, and GHBNODE for MNIST classification in NFE, training time, and test accuracy. (Tolerance: $10^{-5}$).  
  }\label{fig:node-hbnode-mnist}
\end{figure}

\paragraph{NFEs.}
As shown in Figs.~\ref{fig:node-hbnode-cifar10} and \ref{fig:node-hbnode-mnist}, the NFEs grow rapidly with training of the NODE, resulting in an increasingly 
complex model with reduced performance and the possibility 
of blow up. Input augmentation has been verified to effectively reduce the NFEs, as 
both ANODE and SONODE 
{require fewer}
forward 
{NFEs}
than NODE for the MNIST and CIFAR10 classification. However, input augmentation is less effective {in controlling} 
their backward NFEs. HBNODE and GHBNODE require much 
{fewer}
NFEs than the existing benchmarks, especially for backward NFEs. In practice, reducing NFEs {implies reducing} 
both training and inference time, as shown in Figs.~\ref{fig:node-hbnode-cifar10} and \ref{fig:node-hbnode-mnist}. 



\paragraph{Accuracy.}
We {also} compare the accuracy of different ODE-based models for MNIST and CIFAR10 classification. As shown in
{Figs.~\ref{fig:node-hbnode-cifar10} and \ref{fig:node-hbnode-mnist}}, HBNODE and GHBNODE have slightly better classification accuracy than the other three models; this resonates with the fact that less NFEs 
{lead}
to simpler models which generalize better \cite{NEURIPS2019_21be9a4b,norcliffe2020_sonode}. 

\begin{figure}[!ht]\vspace{-2mm}
\centering
\begin{tabular}{c}
\includegraphics[clip, trim=0.01cm 0.01cm 0.01cm 0.01cm, width=0.99\columnwidth]{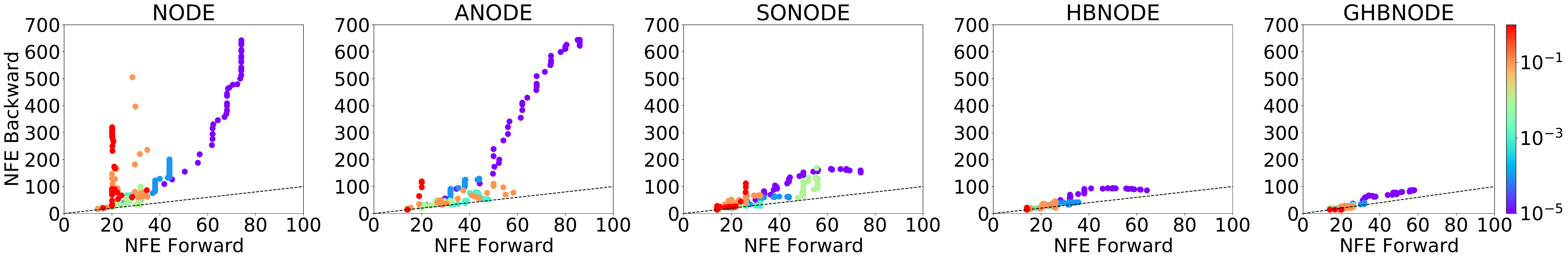}\\
\end{tabular}
\vspace{-4mm}
\caption{NFE vs. tolerance (shown in the colorbar) for training ODE-based models for CIFAR10 classification. Both forward and backward NFEs of HBNODE and GHBNODE grow much 
{more slowly}
than that of NODE, ANODE, and SONODE; especially the backward 
{NFEs}.
As the tolerance decreases, the advantage of HBNODE and GHBNODE in reducing NFEs becomes more significant.}
\label{fig:NFE-vs-tol}\vspace{-5mm}
\end{figure}

\paragraph{NFEs vs. tolerance.} We further study the NFEs for different ODE-based models under different 
{tolerances}
of the 
ODE solver using the same approach as 
in \cite{chen2018neural}. Figure~\ref{fig:NFE-vs-tol} depicts the forward and backward NFEs for different models under different tolerances. We see that (i) both forward and backward NFEs grow quickly when tolerance 
{is decreased,}
and HBNODE and GHBNODE require much 
fewer NFEs than other models; (ii) under different tolerances, the backward
{NFEs}
of NODE, ANODE, and SONODE 
{are much larger}
than the forward 
{NFEs,}
and the difference becomes 
larger when 
{the}
tolerance decreases. In contrast, the forward and backward NFEs of HBNODE and GHBNODE scale almost linearly with each other. 
{This} reflects that the advantage in NFEs of (G)HBNODE over {the} benchmarks become more significant when a smaller tolerance is used.

\subsubsection{Learning dynamical systems from irregularly-sampled time series}\label{sec:physical-systems}
In this subsection, we learn dynamical systems from experimental measurements. In particular, we use the ODE-RNN framework \cite{chen2018neural,NEURIPS2019_42a6845a}, with the recognition model being set to different ODE-based models, 
to study the vibration of an airplane dataset \cite{noel2017f}. The dataset was acquired, from time $0$ to $73627$, by attaching a shaker underneath the right wing to provide input signals, and $5$ attributes are recorded per time stamp; these attributes include voltage of input signal, force applied to aircraft, and acceleration at $3$ different spots of the airplane. We randomly take out $10\%$ of the data to make the time series irregularly-sampled. We use the first $50\%$ of data as our train set, the next $25\%$ as validation set, and the rest as test set. 
We divide each set into non-overlapping segments of consecutive $65$ time stamps of the irregularly-sampled time series, with each input instance consisting of $64$ time stamps of the irregularly-sampled time series, and we aim to forecast $8$ consecutive time stamps starting from the last time stamp of the segment. The input is fed through the the hybrid methods in a recurrent fashion; by changing the time duration of the last step of the ODE integration, we can forecast the output in the different time stamps. The output of the hybrid method is passed to a single dense layer to generate the output time series. In our experiments,  we compare different ODE-based models hybrid with RNNs. The 
ODE of each model is parametrized by a $3$-layer network whereas the RNN is parametrized by a simple dense network; the total number of parameters for ODE-RNN, ANODE-RNN, SONODE-RNN, HBNODE-RNN, and GHBNODE-RNN with $16$, $22$, $14$, $15$, $15$ augmented dimensions are 15,986, 16,730, 16,649, 16,127, and 16,127, respectively. To avoid potential error due to the ODE solver, we use a tolerance of $10^{-7}$.

In training those hybrid models, we regularize the models by penalizing the L2 distance between the RNN output and the values of the next time stamp.
Due to the second-order natural of the underlying dynamics \cite{norcliffe2020_sonode}, ODE-RNN and ANODE-RNN learn the dynamics very poorly with much larger training and test losses than the other models even they take smaller NFEs. HBNODE-RNN and GHBNODE-RNN give better prediction than SONODE-RNN using less backward NFEs.

\begin{figure*}[!ht]
\centering
\begin{tabular}{ccccc}
\hskip -0.4cm
\includegraphics[width=0.2\linewidth]{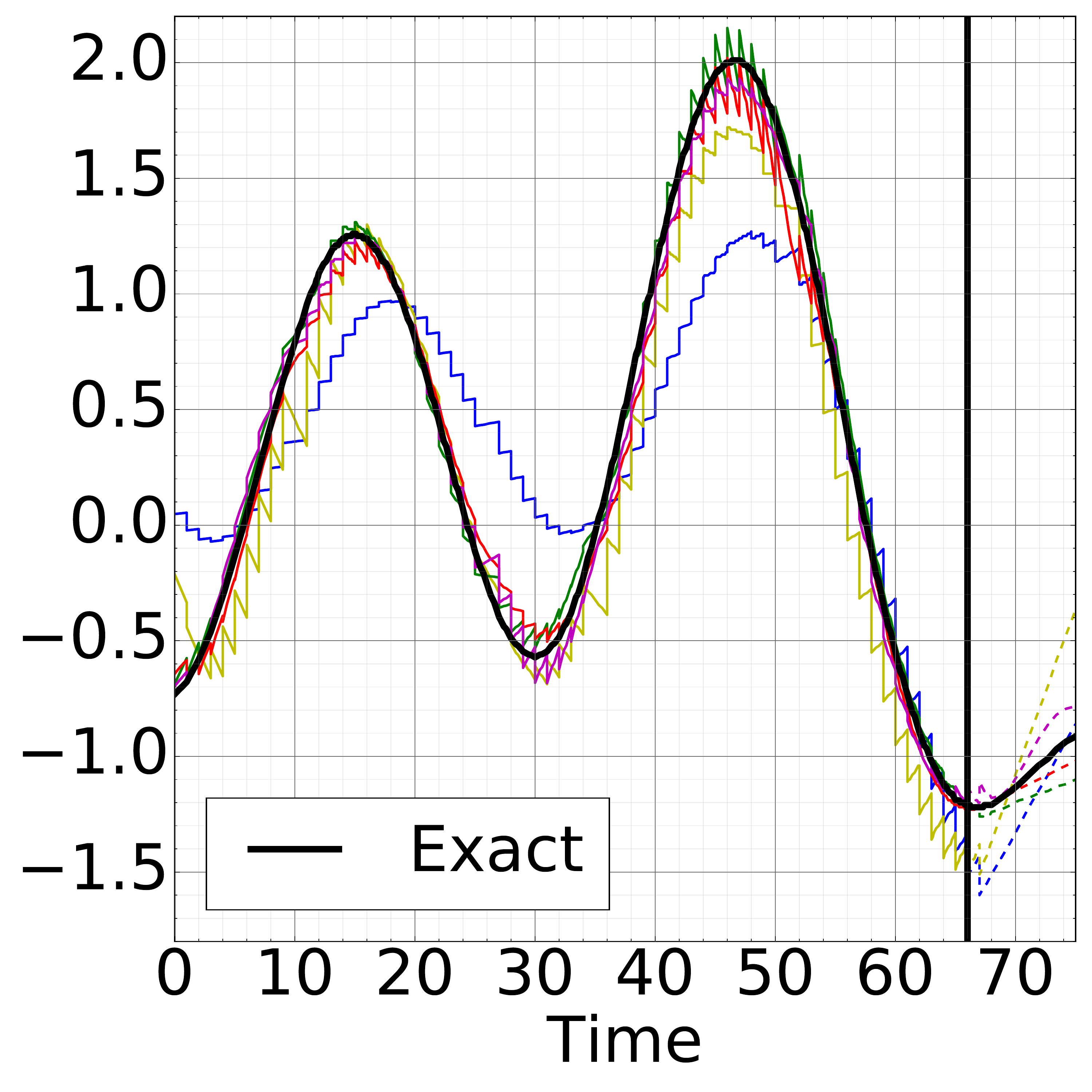}&
\hskip -0.4cm
\includegraphics[width=0.2\linewidth]{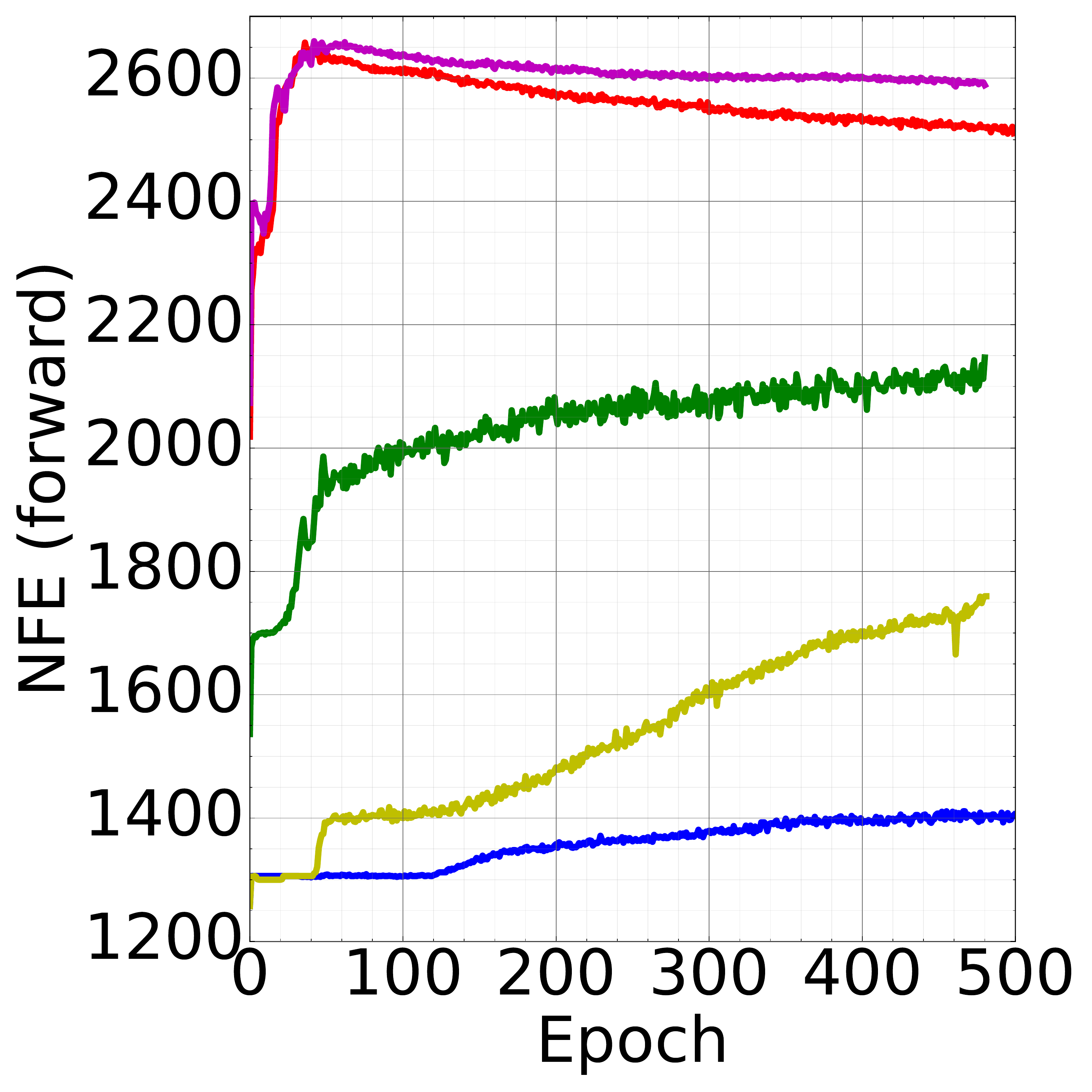}&
\hskip -0.4cm
\includegraphics[width=0.2\linewidth]{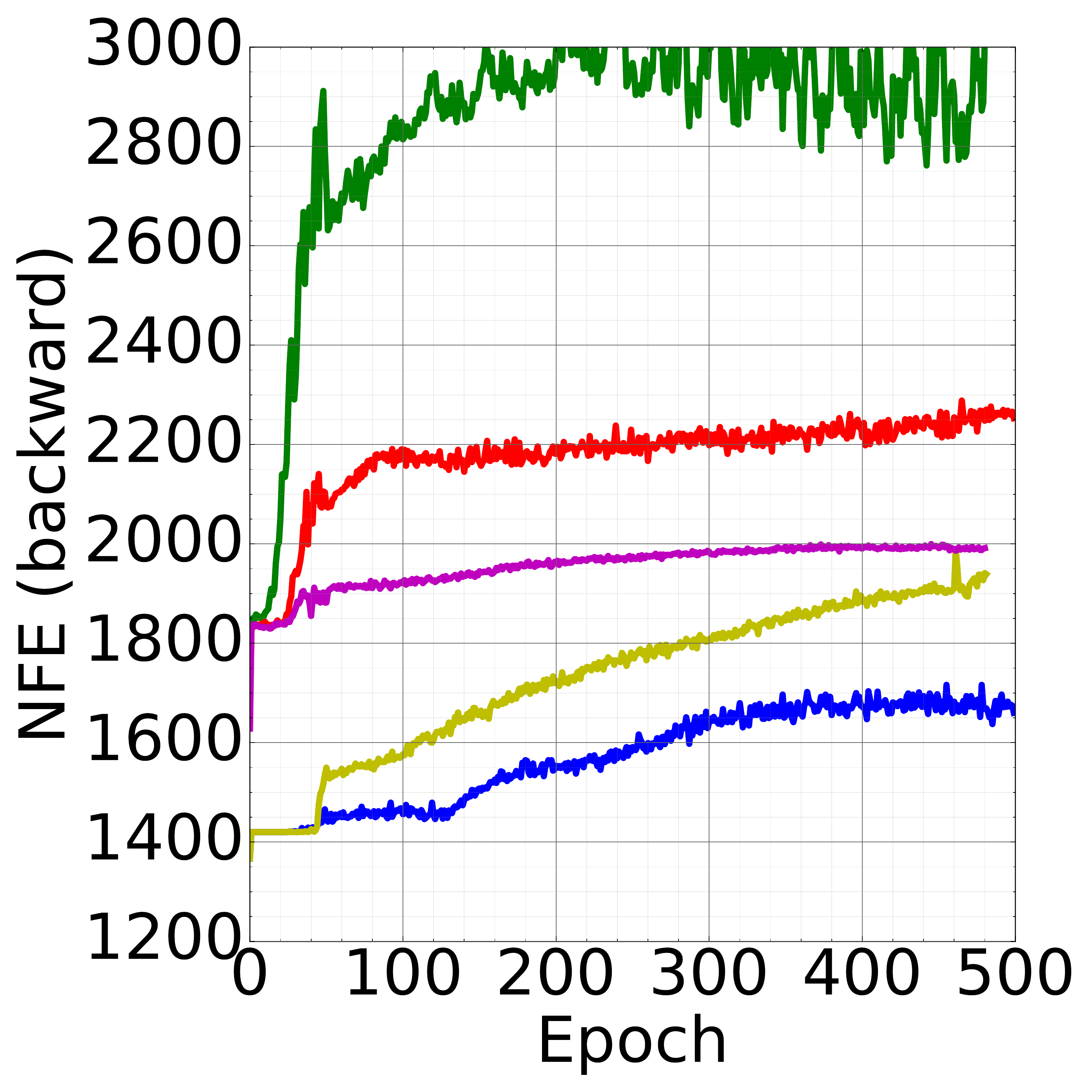}&
\hskip -0.4cm
\includegraphics[width=0.2\linewidth]{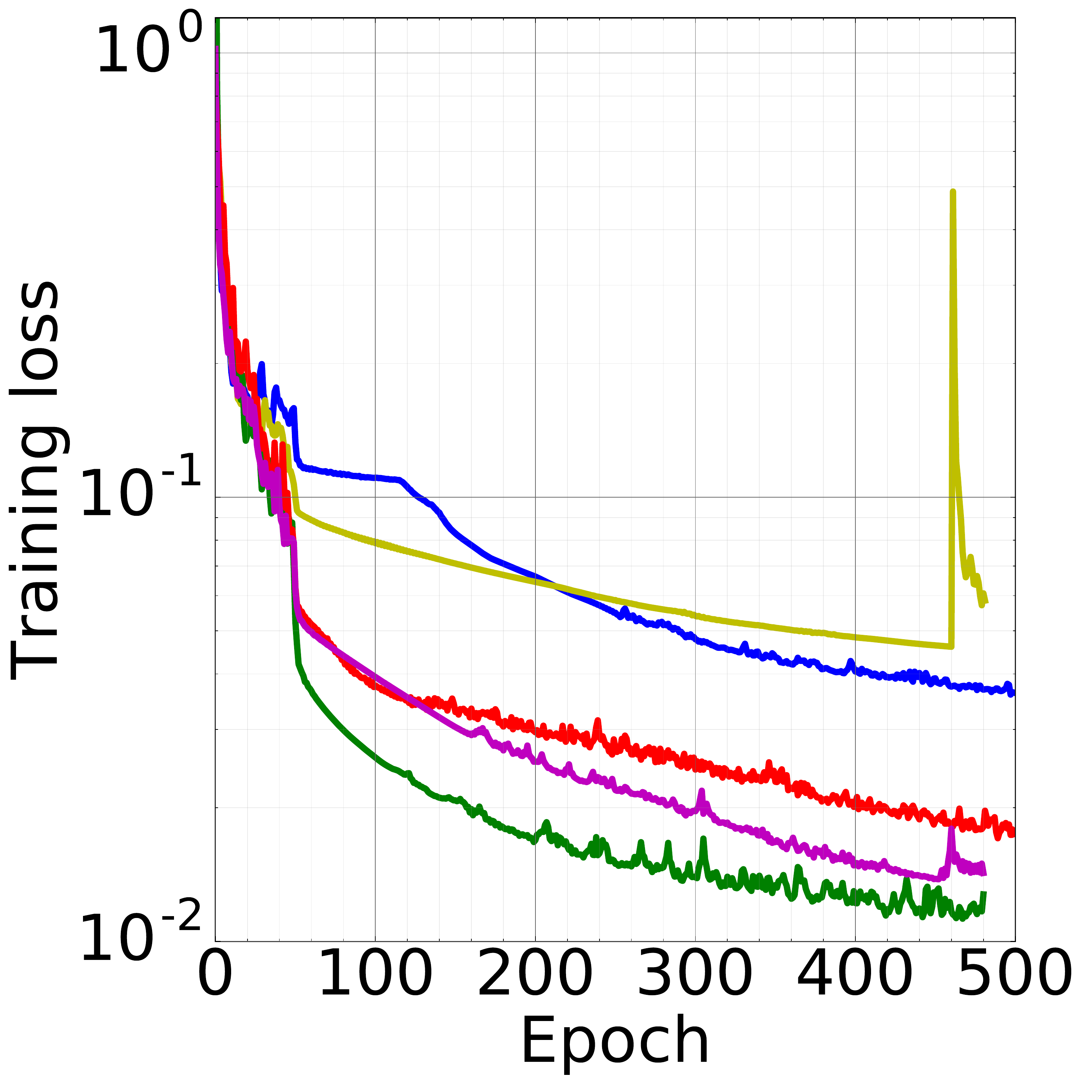}&
\hskip -0.4cm
\includegraphics[width=0.2\linewidth]{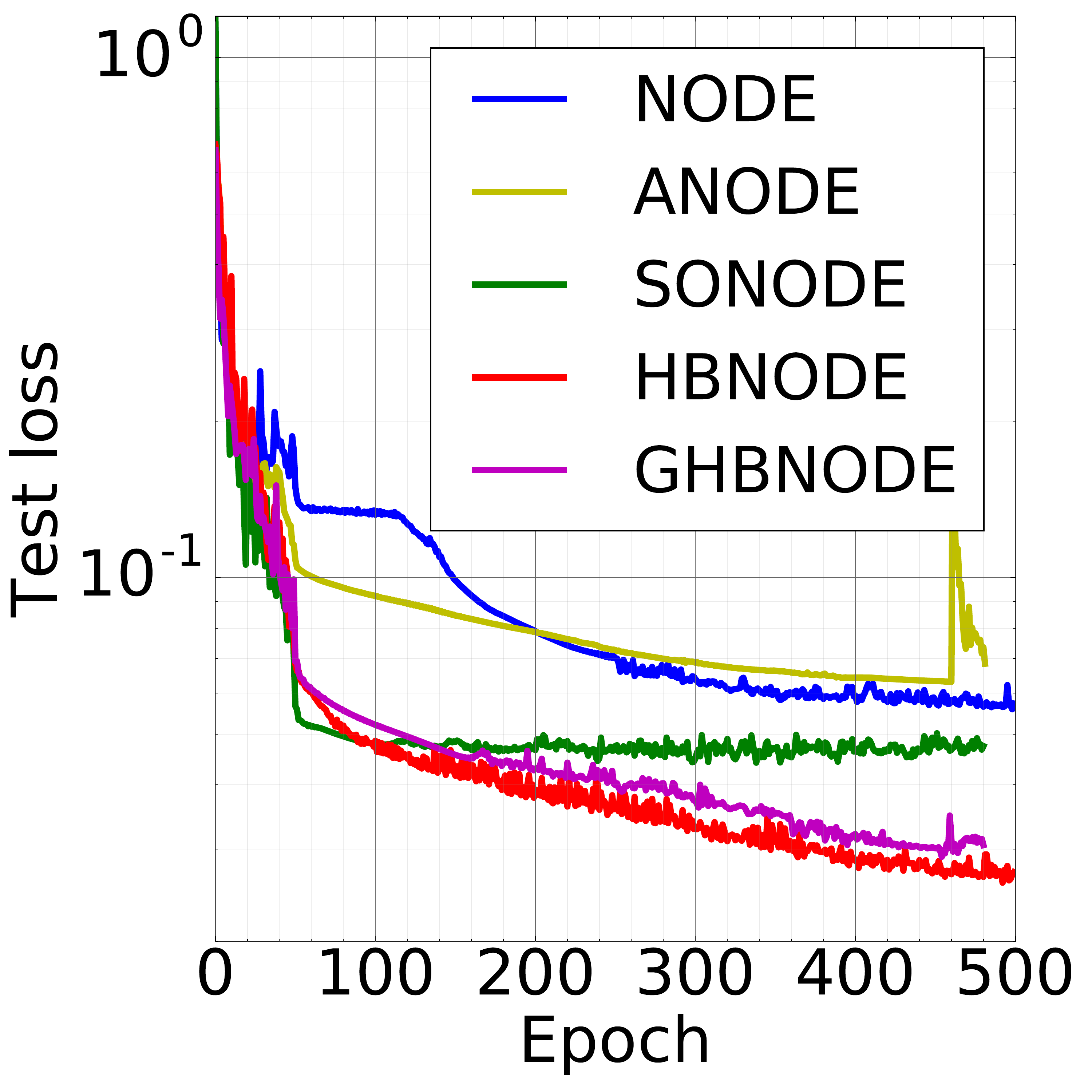}\\
\end{tabular}
\vskip -0.3cm
\caption{Contrasting ODE-RNN, ANODE-RNN, SONODE-RNN, HBNODE-RNN, and GHBNODE-RNN for learning a vibrational dynamical system. Left most: The learned curves of each model vs. the ground truth (Time: $<$66 for training, 66-75 for testing).  
}
\label{fig:Vibration-dynamics}
\end{figure*}

\subsubsection{
Walker2D kinematic simulation
}\label{sec:sequential-modeling}

We evaluate the performance of HBNODE-RNN and GHBNODE-RNN on the 
Walker2D kinematic simulation task, which requires learning long-term dependency effectively \cite{lechner2020learning}. The dataset \cite{brockman2016openai} consists of a dynamical system from kinematic simulation of a person walking from a pre-trained policy, aiming to learn the kinematic simulation of the MuJoCo physics engine \cite{6386109}. The dataset is irregularly-sampled with $10\%$ of the data removed from the simulation. Each input consists of 64 time stamps fed though the the hybrid methods in a recurrent fashion, and the output is passed to a single dense layer to generate the output time series. The goal is to provide an auto-regressive forecast so that the output time series is as close as the input sequence shifted one time stamp to the right. We compare ODE-RNN 
(with 7 augmentation), ANODE-RNN (with 7 ANODE style augmentation), HBNODE-RNN 
(with 7 augmentation), and GHBNODE-RNN (with 7 augmentation) 
The RNN is parametrized by a 3-layer network whereas the ODE is parametrized by a simple dense network. The number of parameters of the above four models are 8,729, 8,815, 8,899, and 8,899, respectively.  In Fig.~\ref{fig:walker-dynamics}, we compare the performance of the above four models on the Walker2D benchmark; HBNODE-RNN and GHBNODE-RNN not only require significantly less NFEs in both training (forward and backward) and in testing than ODE-RNN and ANODE-RNN, but also have much smaller training and test losses.

\begin{figure*}[!ht]
\centering
\begin{tabular}{ccccc}
\hskip -0.3cm
\includegraphics[width=0.2\linewidth]{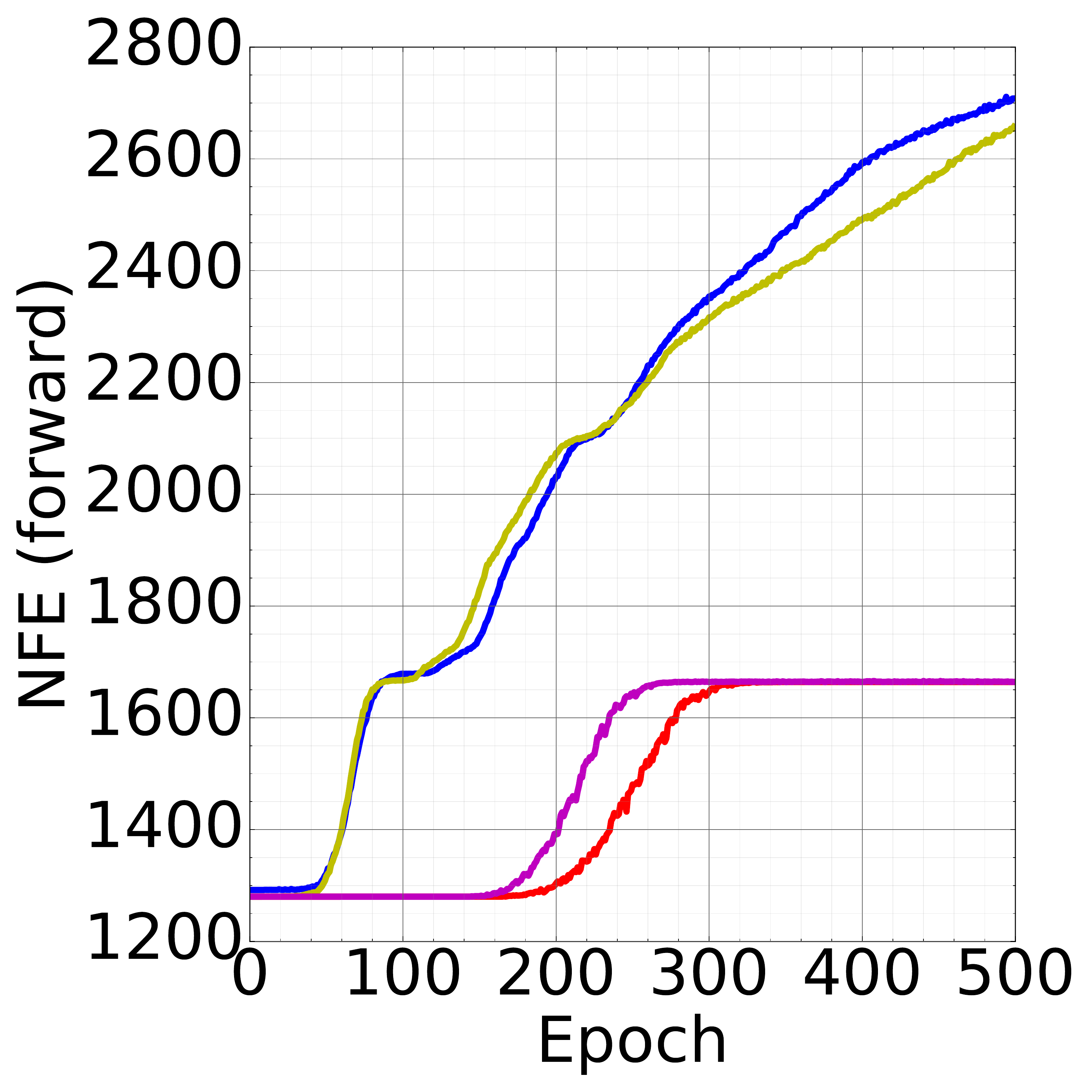}&
\hskip -0.4cm
\includegraphics[width=0.2\linewidth]{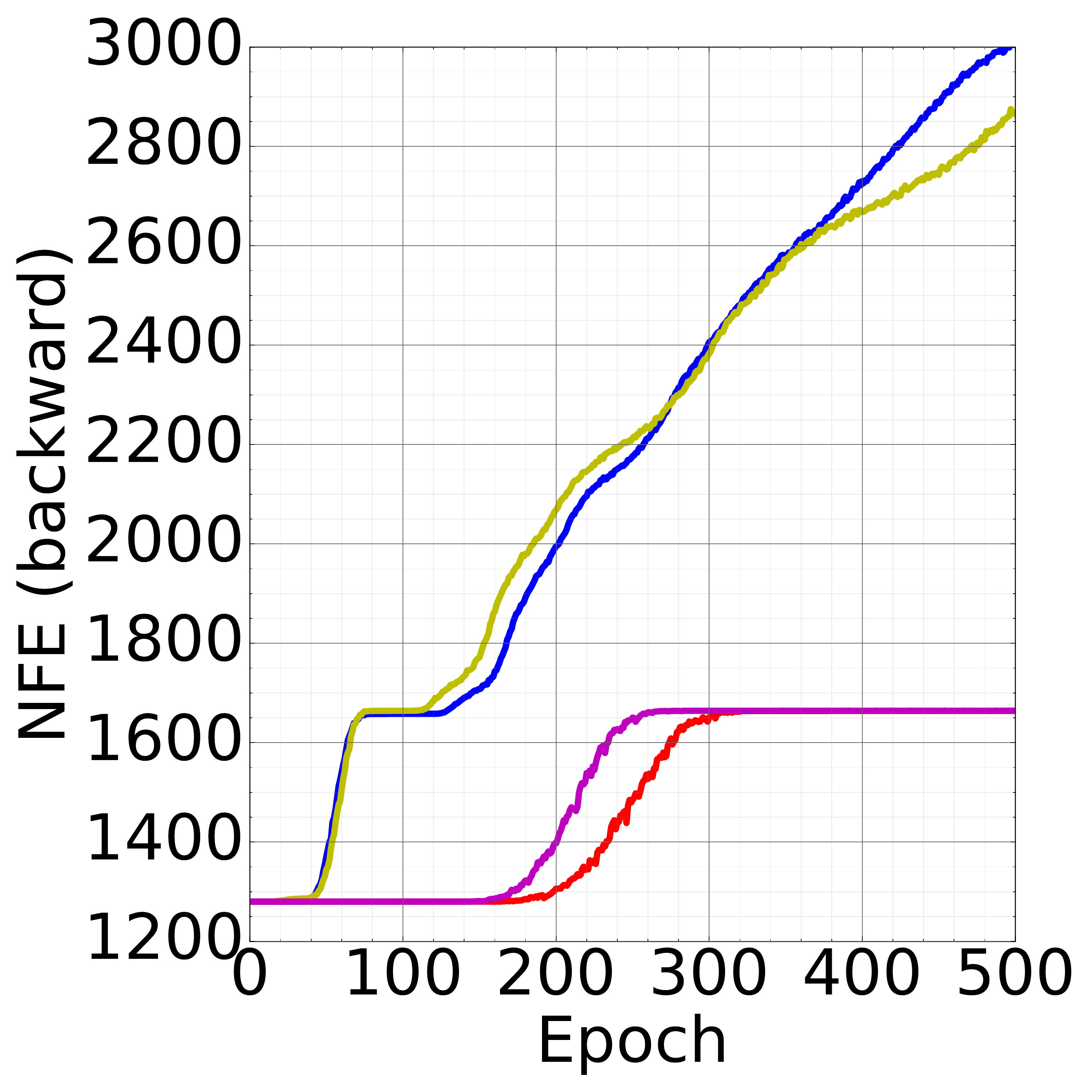}&
\hskip -0.4cm
\includegraphics[width=0.2\linewidth]{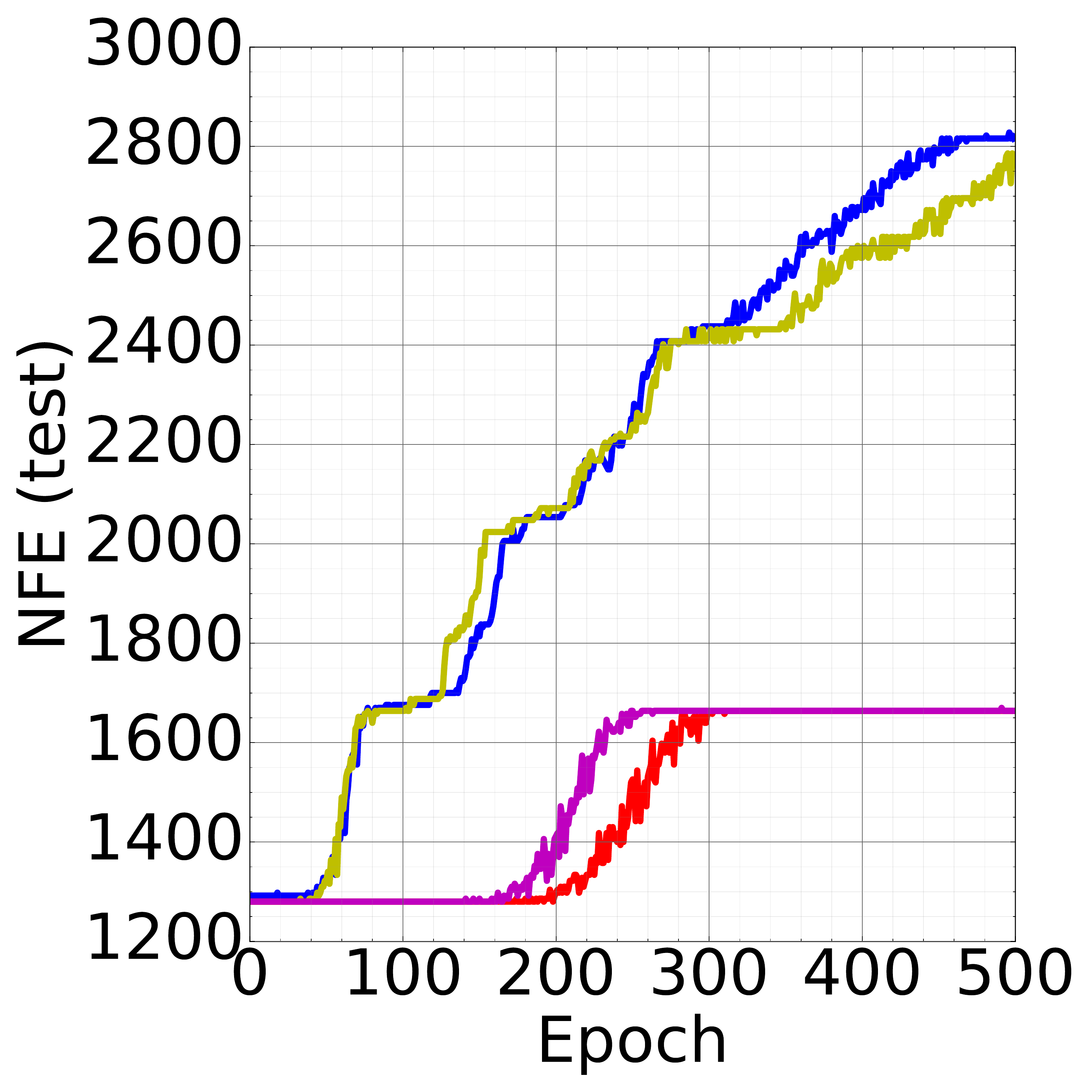}&
\hskip -0.4cm
\includegraphics[width=0.2\linewidth]{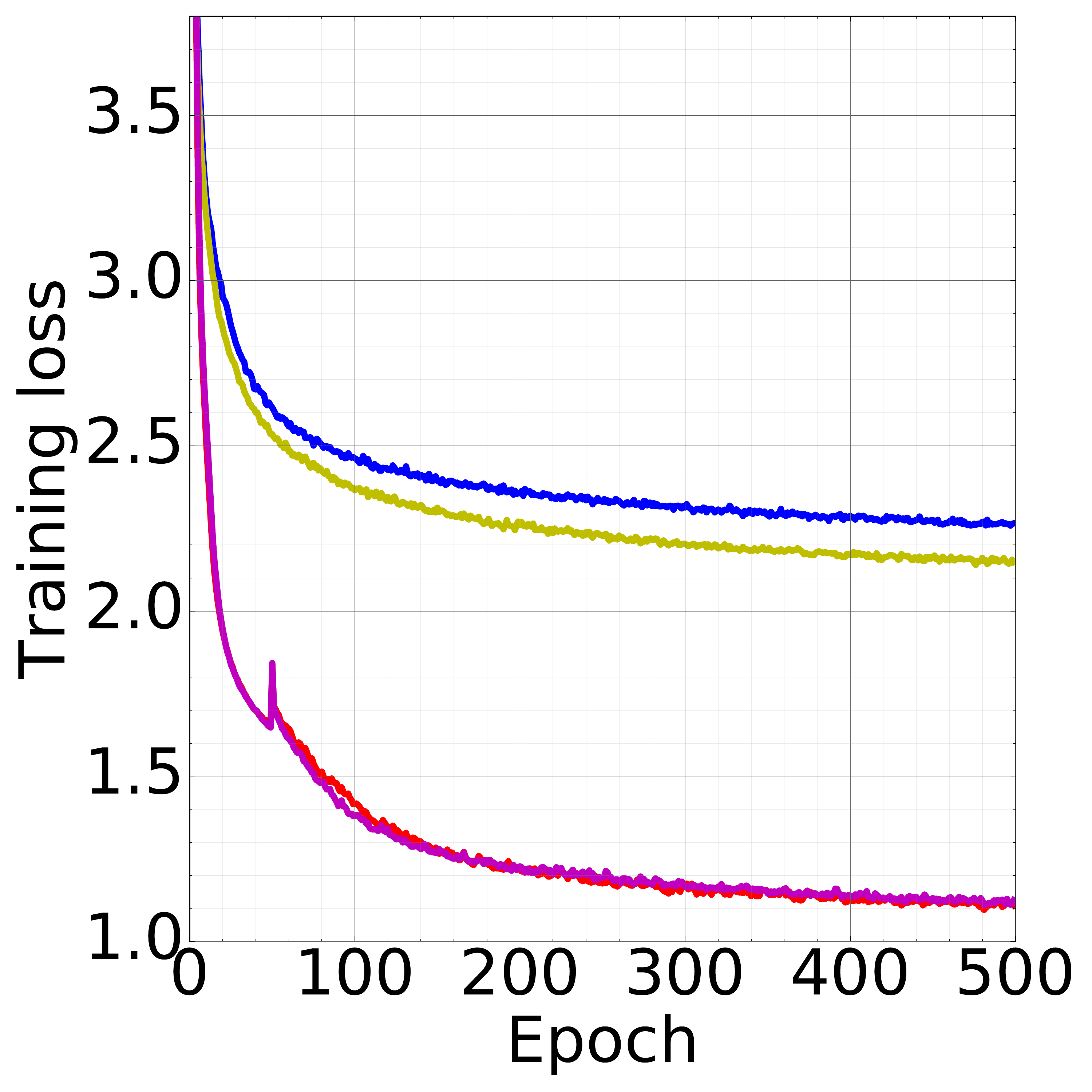}&
\hskip -0.4cm
\includegraphics[width=0.2\linewidth]{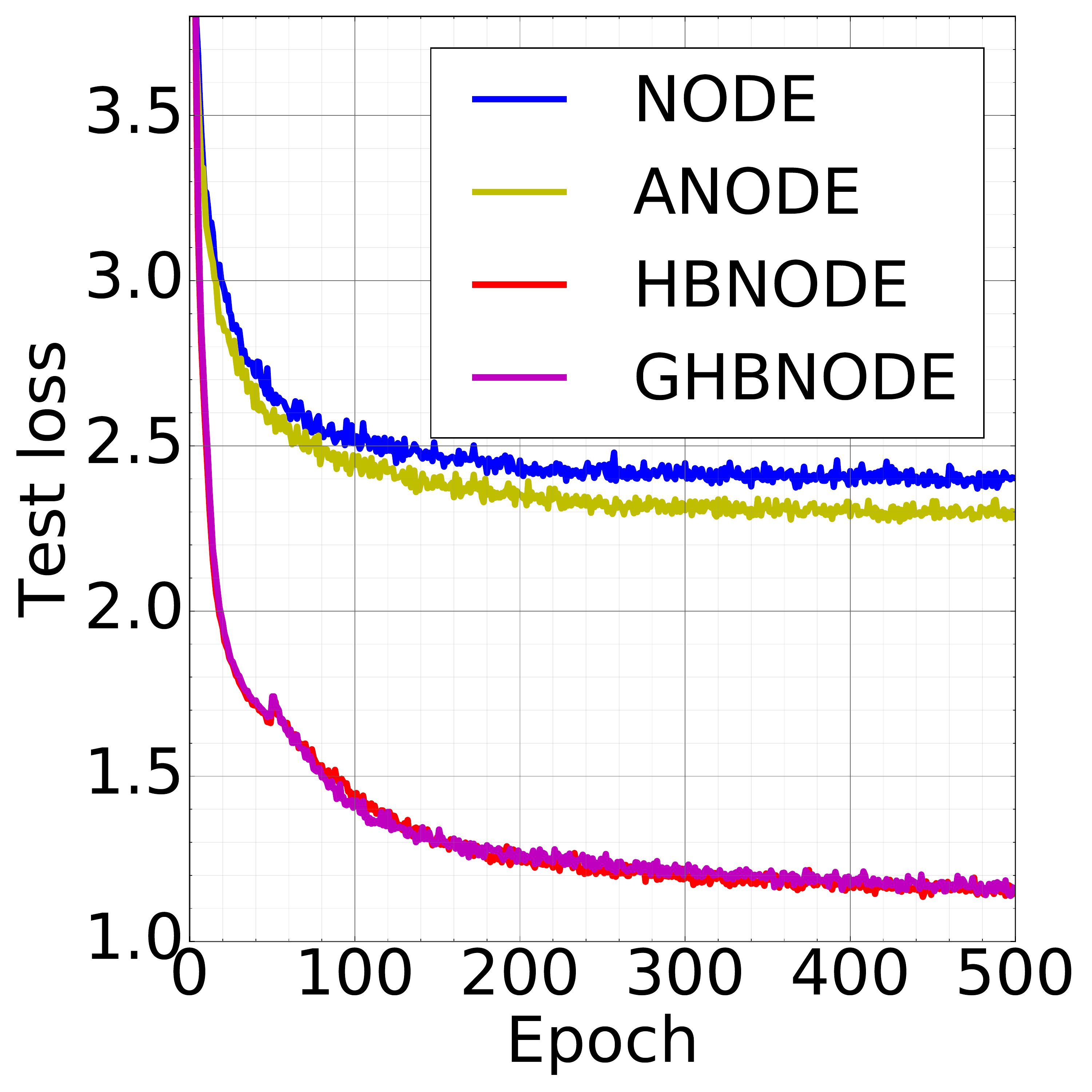}\\
\end{tabular}
\vskip -0.3cm
\caption{Contrasting ODE-RNN, ANODE-RNN, SONODE-RNN, HBNODE-RNN, and GHBNODE-RNN for the Walker-2D kinematic simulation. 
}
\label{fig:walker-dynamics}
\end{figure*}

\section{Transformers}\label{sec-transformers}
We further show that momentum can be integrated into transformers, which can significantly reduce the computational and memory costs of the standard transformer \cite{vaswani2017attention} and enhance the performance of linear transformers \cite{katharopoulos2020transformers}.

The self-attention mechanism is a fundamental building block of transformers 
\cite{vaswani2017attention,kim2017structured}. Given an input sequence ${\mX}=[\vx_1,\vx_2,\cdots,\vx_N]^\top\in \RR^{N\times D_x}$ of $N$ feature vectors, the self-attention transformers it into another sequence $\hat{\mV}=[\hat{\vv}_1,\hat{\vv}_2,\cdots,\hat{\vv}_N]^\top\in \RR^{N\times D_v}$ as follows
\begin{equation}\label{eq:attention-single-vector}
\hat{\vv}_i=\sum_{j=1}^N{\rm softmax}\Big(\frac{\vq_i^\top\vk_j}{\sqrt{D}} \Big)\vv_j,\ \mbox{for}\ i=1,\cdots,N,
\end{equation}
where the scalar ${\rm softmax}({(\vq_i^\top\vk_j)}/{\sqrt{D}})$ can be understood as the attention $\hat{\vv}_i$ pays to the input feature $\vx_j$. The vectors $\vq_i,\vk_j,$ and $\vv_j$ are called the query, key, and value vectors, respectively; these vectors are computed as follows
\begin{equation}\label{eq:query-key-value}
\begin{aligned}
[\vq_1,\vq_2,\cdots,\vq_N]^\top &:= {\mQ} = {\mX}{\mW}_Q^\top \in \RR^{N\times D}, \\
[\vk_1,\vk_2,\cdots,\vk_N]^\top &:={\mK} = {\mX}{\mW}_K^\top \in \RR^{N \times D}, \\
[\vv_1,\vv_2,\cdots,\vv_N]^\top &:= {\mV} = {\mX}{\mW}_V^\top \in \RR^{N\times D_v}, 
\end{aligned}
\end{equation}
where ${\mW}_Q, {\mW}_K\in \RR^{D\times D_x}$, and ${\mW}_V\in \RR^{D_v\times D_x}$ are the weight matrices.
We can further write \eqref{eq:attention-single-vector} into the following compact form
\begin{equation}\label{eq:attention-vector-form}
\hat{\mV}={\rm softmax}\Big(\frac{{\mQ}{\mK}^\top}{\sqrt{D}} \Big){\mV},
\end{equation}
where the softmax function is applied to each row of 
${\small {(\mQ\mK^\top)}/{\sqrt{D}}}$. Equation \eqref{eq:attention-vector-form} is also called the ``scaled dot-product attention'' or ``softmax attention''. Each transformer layer $T_{\ell}(\cdot)$ is defined via the following residual connection,
\begin{align}
    T_{\ell}({\mX}) = f_{\ell}(\hat{\mV} + {\mX}), \label{eqn:res-connect}
\end{align}
where $f_{\ell}$ is a function that transforms each feature vector independently and usually chosen to be a 
feedforward network. We call a transformer built with softmax attention standard transformer or transformer. It is easy to see that both memory and computational complexity of \eqref{eq:attention-vector-form} are 
$\mathcal{O}(N^2)$ with $N$ being the length of the input sequence. We can further introduce causal masking into \eqref{eq:attention-vector-form} for autoregressive applications \cite{vaswani2017attention}.

Transformers have become the state-of-the-art model for solving many challenging problems in natural language processing
\cite{vaswani2017attention,47866,dai2019transformer,williams-etal-2018-broad,devlin2018bert,NEURIPS2020_1457c0d6,howard-ruder-2018-universal,rajpurkar-etal-2016-squad} and computer vision \cite{dehghani2018universal,so2019evolved,dosovitskiy2020image,touvron2020deit}. Nevertheless, the quadratic memory and computational cost of computing the softmax attention \eqref{eq:attention-vector-form} is a major bottleneck for applying transformers to large-scale applications that involve very long sequences, such as those in 
\cite{liu2018generating,huang2018music,pmlr-v80-parmar18a}.
Thus, much recent research on transformers has been focusing on developing efficient transformers, aiming to reduce the memory and computational complexities of the model \cite{qiu2019blockwise,child2019generating,ho2019axial,pmlr-v80-parmar18a,beltagy2020longformer,ainslie-etal-2020-etc,wang2020linformer,tay2020synthesizer,pmlr-v119-tay20a,Kitaev2020Reformer,roy-etal-2021-efficient,vyas2020fast,zaheer2021big,wang2020linformer,katharopoulos2020transformers,choromanski2021rethinking,shen2021efficient,DBLP:journals/corr/abs-2102-11174,pmlr-v80-blanc18a,NEURIPS2019_e43739bb,song2021implicit,peng2021random,xiong2021nystromformer}. A thorough survey of recent advances in efficient transformers is available at \cite{tay2020efficient}. 
These efficient transformers have better memory and/or computational efficiency at the cost of a significant reduction in accuracy. 

\subsection{Motivation}
In \cite{katharopoulos2020transformers}, the authors have established a connection between transformers and RNNs through the kernel trick. They proposed the linear transformer, which can be considered a rank-one approximation of the softmax transformer. Linear transformers have computational advantages in training, test, and inference: the RNN formulation (see \eqref{eq:rnn:formulation} below) enjoys fast inference, especially for autoregressive tasks, and the unrolled RNN formulation (see \eqref{eq:attention-kernel-linear-matrix} below) is efficient for fast training. See subsection~\ref{sec:transformers} for a detailed review of the linear transformer and its advantages. \cite{MomentumRNN} proposes integrating momentum into RNNs to accelerate training RNNs and improve learning long-term dependencies. We notice that MomentumRNN also enjoys a closed unrolling form, which is quite unique among existing techniques for improving RNNs, enabling fast training, test, and inference; see section~\ref{sec:momentum-transformers} for details. As such, in this section we study \emph{how momentum improves linear transformers?}

\subsection{Linear transformer}
\label{sec:transformers}
Transformers learn long-term dependencies in sequences effectively and concurrently through the self-attention mechanism. Note we can write \eqref{eq:attention-single-vector} as  ${\hat{\vv}_i=({\sum_{j=1}^Nk(\vq_i,\vk_j)\vv_j})/({\sum_{j=1}^Nk(\vq_i,\vk_j)})}$, where $k(\vq_i,\vk_j)$ $:=\exp(\vq_i^\top\vk_j/\sqrt{D})$. In linear transformers \cite{wang2020linformer,katharopoulos2020transformers,choromanski2021rethinking,shen2021efficient}, the feature map $k(\vq_i,\vk_j)$ is linearized as the product of feature maps $\phi(\cdot)$ on the vectors $\vq_i$ and $\vk_j$, i.e., $k(\vq_i,\vk_j)=\phi(\vq_i)^\top\phi(\vk_j)$. The associative property of matrix multiplication is then utilized to derive the following efficient computation of the attention map
\begin{equation}
\begin{aligned}\label{eq:attention3}
\hspace{-0.3cm}{\small \hat{\vv}_i=\frac{\sum_{j=1}^Nk(\vq_i,\vk_j)\vv_j}{\sum_{j=1}^Nk(\vq_i,\vk_j)}
=\frac{\sum_{j=1}^N\phi(\vq_i)^\top\phi(\vk_j)\vv_j}{\sum_{j=1}^N\phi(\vq_i)^\top\phi(\vk_j) } 
=\frac{\phi(\vq_i)^\top\sum_{j=1}^N\phi(\vk_j)\vv_j^\top}{\phi(\vq_i)^\top\sum_{j=1}^N\phi(\vk_j)}.}
\end{aligned}
\end{equation}
In the matrix-product form, we can further write \eqref{eq:attention3} as follows
\begin{equation}\label{eq:attention-kernel-linear-matrix}
\Hat{\mV}=\frac{\phi({\mQ})(\phi({\mK})^\top{\mV} )}{\phi({\mQ})\phi({\mK})^\top}.
\end{equation}
Replacing $(\phi({\mQ})\phi({\mK}^\top)){\mV}$ with $\phi({\mQ})(\phi({\mK}^\top){\mV})$ reduces the  memory and computational cost of computing the attention map from $\mathcal{O}(N^2)$ to $\mathcal{O}(N)$, making linear transformers scalable to very long sequences. 

Causal masking can be easily implemented in the linearized attention by truncating the summation term in the last equation of \eqref{eq:attention3}, resulting in 
\begin{eqnarray}\label{eq:attention:causal-masking}
\hat{\vv}_i=\frac{\phi(\vq_i)^\top\sum_{j=1}^i\phi(\vk_j)\vv_j^\top}{\phi(\vq_i)^\top\sum_{j=1}^i\phi(\vk_j)}:=\frac{\phi(\vq_i)^\top\vs_i}{\phi(\vq_i)^\top\vz_i},
\end{eqnarray}
where $\vs_i=\sum_{j=1}^i\phi(\vk_j)\vv_j^\top$ and $\vz_i=\sum_{j=1}^i\phi(\vk_j)$. The states $\vs_i$ and $\vz_i$ can be computed recurrently.

\paragraph{Efficient inference via the RNN formulation.} Self-attention processes tokens of a sequence concurrently, enabling fast training of transformers.
However, during inference, the output for timestep $i$ is the input for timestep $i + 1$. As a result, the inference in standard transformers cannot be parallelized and is thus inefficient. Linear transformers provide an elegant approach to fixing this issue by leveraging their RNN formulation. In particular, we can further write the linear attention with causal masking in \eqref{eq:attention:causal-masking} into the following RNN form\footnote{We omit the nonlinearity (a two-layer feedforward network) compared to \cite{katharopoulos2020transformers}.}
\begin{equation}\label{eq:rnn:formulation}
\begin{aligned}
\vs_i&=\vs_{i-1} + \phi(\vk_i)\vv_i^\top;\quad\\
\vz_i&=\vz_{i-1} + \phi(\vk_i);\quad\\
\hat{\vv}_i&=\frac{\phi(\vq_i)^\top \vs_i}{\phi(\vq_i)^\top\vz_i},
\end{aligned}
\end{equation}
where $\vs_0=\mathbf{0}$ and $\vz_0=\mathbf{0}$. Note that this RNN formulation of linear transformers with causal masking contains two memory states $\vs_i$ and $\vz_i$.

\subsection{Momentum transformer}
\label{sec:momentum-transformers}
In this section, we present the \emph{momentum transformer}. We start by integrating the heavy-ball momentum into the RNN formulation of causal linear attention in \eqref{eq:rnn:formulation}, resulting in the causal momentum attention. Next, we generalize the causal momentum attention to momentum attention that can efficiently train the model. Moreover, we propose the \emph{momentum connection} to replace residual connections between the attention $\hat{\mV}$ and the input ${\mX}$ in~\eqref{eqn:res-connect} to boost the model's performance. Finally, we derive the adaptive momentum attention from the theory of optimal choice of momentum for the heavy-ball method.

\subsubsection{Momentum transformer}\label{subsec:momentum-transformer}

\paragraph{Integrating momentum into causal linear attention.} Now we consider integrating momentum into causal linear attention. We integrate momentum into the state $\vs_i$  in \eqref{eq:rnn:formulation} only since the denominator in causal linear attention is simply a normalizing scalar. If we regard $-\phi(\vk_i)\vv_i^\top$ as the gradient vector in \eqref{eq:GD}, then we can add momentum into the state $\vs_i$ by following 
the heavy-ball method in \eqref{eq:HB2}, resulting in the following RNN formulation of causal momentum attention,
\begin{equation}\label{eq:momentum:rnn:formulation}
\begin{aligned}
\vm_i &= \beta\vm_{i-1}-\phi(\vk_i)\vv_i^\top;\quad \\
\vs_i &= \vs_{i-1} - \gamma\vm_i;\quad\\
\vz_i &= \vz_{i-1} + \phi(\vk_i);\quad\\
\hat{\vv}_i &= \frac{\phi(\vq_i)^\top\vs_i}{\phi(\vq_i)^\top\vz_i},
\end{aligned}
\end{equation}
where $\vm_0=\mathbf{0}$, and $\gamma>0$ and $0\leq \beta<1$ are two hyperparameters. The RNN formulation of causal momentum attention in \eqref{eq:momentum:rnn:formulation} is efficient for autoregressive inference. For 
training, we need to rewrite \eqref{eq:momentum:rnn:formulation} into a form that is similar to 
\eqref{eq:attention:causal-masking}. To this end, we need to eliminate $\vm_i, \vs_i$, and $\vz_i$ from \eqref{eq:momentum:rnn:formulation}. Note that
\begin{equation*}
\begin{aligned}
\vs_i = \vs_{i-1}-\underbrace{\gamma\vm_i}_{:=\vp_i}
 = \underbrace{\vs_0}_{=\mathbf{0}} - \Big(\vp_i+\vp_{i-1}+\cdots+\vp_1\Big),
\end{aligned}
\end{equation*}
since $\vm_i=\beta\vm_{i-1}-\phi(\vk_i)\vv_i^\top$, we have $\vp_i=\beta\vp_{i-1}-\gamma\phi(\vk_i)\vv_i^\top$. Therefore,
\begin{equation*}
\begin{aligned}
\vs_i &= - (\vp_i+\vp_{i-1}+\cdots+\vp_1)= \gamma\phi(\vk_i)\vv_i^\top - \Big( (1+\beta)\vp_{i-1}+\vp_{i-2}+\cdots+\vp_1\Big)\\
&= \gamma\phi(\vk_i)\vv_i^\top + \gamma(1+\beta)\phi(\vk_i)\vv_i^\top  - \Big((1+\beta)^2\vp_{i-2}+\cdots+\vp_1\Big)\\
&=\cdots\\
&=\gamma \sum_{j=1}^i\frac{1-\beta^{i-j+1}}{1-\beta}\phi(\vk_j)\vv_j^\top\ \mbox{for}\ i\geq 1.
\end{aligned}
\end{equation*}
We can then formulate the causal momentum attention as follows
\begin{equation}\label{eq:momentum:causal:attention}
\hat{\vv}_i=\frac{\gamma\phi(\vq_i)^\top\sum_{j=1}^i\Big(\frac{1-\beta^{i-j+1}}{1-\beta}\phi(\vk_j)\vv_j^\top\Big) }{\phi(\vq_i)^\top\vz_i}.
\end{equation}
Note that \eqref{eq:momentum:causal:attention} is mathematically equivalent to \eqref{eq:momentum:rnn:formulation}, but it can be trained much more efficiently in a concurrent fashion via layerwise parallelism. 

\begin{remark}
Comparing \eqref{eq:momentum:causal:attention} with \eqref{eq:attention:causal-masking} , we see that momentum plays a role in reweighting the terms $\{\phi(\vk_j)\vv_j^\top\}_{j=1}^i$. It is interesting to note that this reweighting is opposite to that used for reweighting the local attention \cite{dai2019transformer}. It has also been noticed that low-rank attention can complement local attention, resulting in improved performance \cite{nguyen2021fmmformer}.
\end{remark}

\paragraph{Integrating momentum into linear attention.} To obtain momentum attention without causal masking, we can simply take the sum from $1$ to $N$ instead of summing from $1$ to $i$. Therefore, we obtain the following momentum attention
\begin{equation}\label{eq:momentum:attention}
\hat{\vv}_i=\frac{\gamma\phi(\vq_i)^\top\sum_{j=1}^N\Big(\frac{1-\beta^{N-j+1}}{1-\beta}\phi(\vk_j)\vv_j^\top\Big) }{\phi(\vq_i)^\top\sum_{j=1}^N\phi(\vk_j)}.
\end{equation}

\paragraph{Memory and computational complexity.} 
Training momentum transformers have the same memory and computational complexities of $\mathcal{O}(N)$ as the training of linear transformers. For test and inference, momentum transformers also have the same memory and computational complexities as linear transformers. However, in the RNN form, momentum transformers require slightly more memory than linear transformers to store the extra momentum state $\vm_i$.

\subsubsection{Momentum connection}\label{subsec:momentum-connection}
Each transformer layer has a residual connection between the self-attention output and the input as shown in~\eqref{eqn:res-connect}. We further integrate momentum into~\eqref{eqn:res-connect} and derive the momentum connection as follows

\begin{equation}\label{eq:momentum-transformer-layer}
T_\ell({\mX}) = f_\ell\big(\hat{\mV}+{\mX} + \Tilde{\beta}({\mX}-T_{\ell-1}({\mX}))\big),\ \ 0\leq \Tilde{\beta}<1.
\end{equation}

\paragraph{Adaptive momentum.}\label{subsubsec:adaptive-momentum}\ \ 
Our momentum transformer introduces additional hyperparameters $\gamma$ and $\beta$, as well as $\Tilde{\beta}$, compared to the linear transformer. 
Often $\gamma$ can be simply set to 1. 
However, tuning 
$\beta$ and $\Tilde{\beta}$
can introduce extra computational cost for training transformers. Moreover, using a constant momentum may not give us optimal performance. In this part, we will introduce an adaptive momentum formula for computing the momentum hyperparameter in momentum connection and thus eliminating the computational overhead for tuning $\Tilde{\beta}$. Here, the adaptive momentum does not apply to $\beta$ since it will break the closed unrolling form in \eqref{eq:momentum:causal:attention}. Adaptive momentum has been used in optimization, see, e.g., \cite{wang2020stochastic}, \cite{sun2021training}; here, we use the later one for its simplicity.

\begin{figure*}[t!]
\centering
\includegraphics[width=0.9\linewidth]{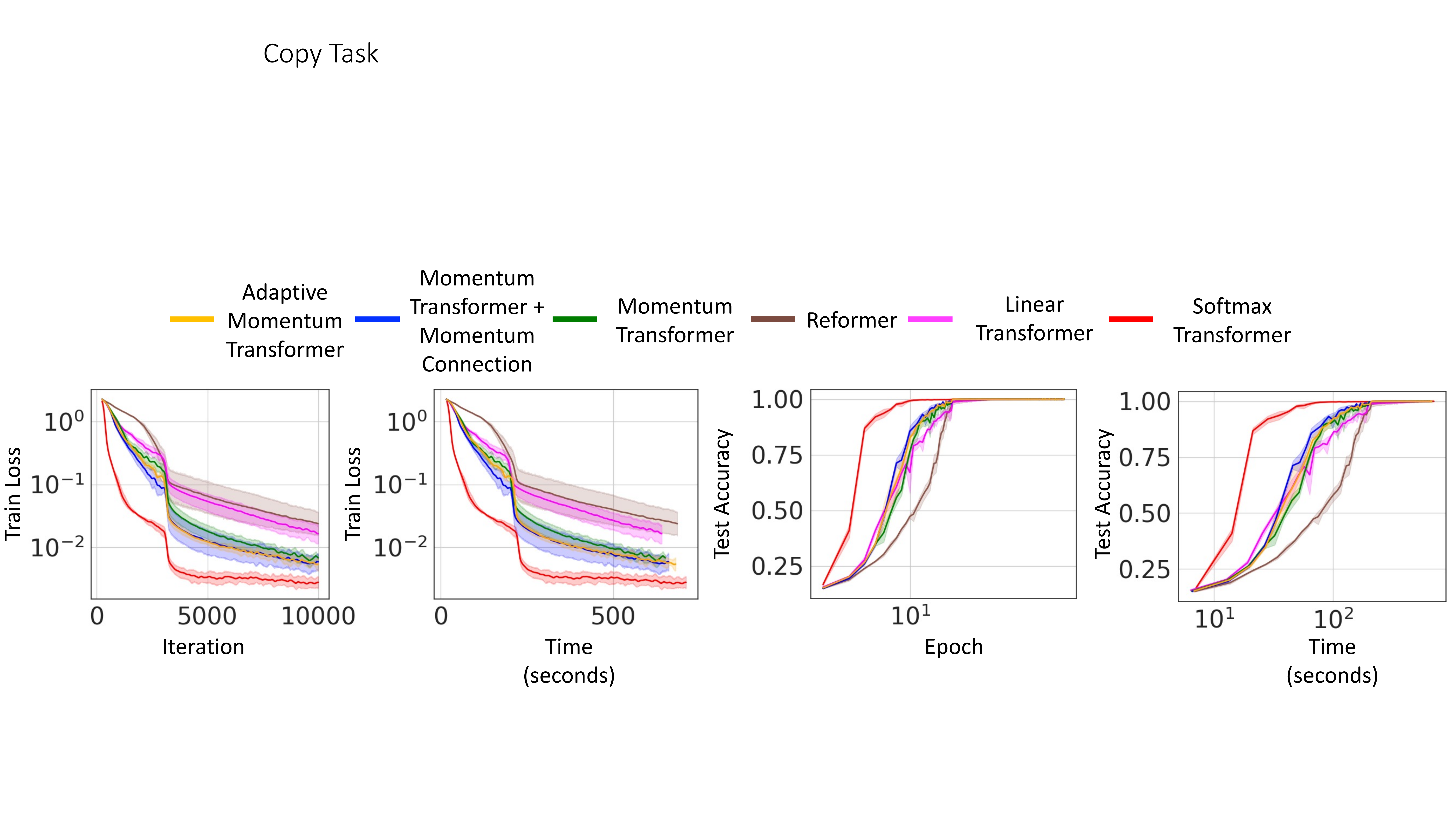}\vspace{-0.3cm}
\caption{Convergence comparison of adaptive momentum, momentum, reformer, linear, and softmax transformer on the sequence copy task. Momentum and adaptive momentum transformers converge faster and achieve better training loss than both linear transformer and reformer. Softmax transformer converges the fastest but suffers from quadratic memory and computational complexity. Adaptive momentum transformer performs as well as momentum transformer without intensively searching for momentum values. 
}
\label{fig:copy-convergence-analysis}
\end{figure*}

\subsection{Experimental results}

We evaluate the benefits of our momentum transformers in terms of convergence speed, efficiency, and accuracy. We compare the performance of momentum and adaptive momentum transformers with the baseline standard softmax transformer and several other efficient transformers in the following tasks: 1) the synthetic copy task, 2) the MNIST and CIFAR image generation task, 3) Long-Range Arena~\cite{tay2021long}, and 4) the non-autoregressive machine translation task. These tasks are among standard benchmarks for measuring the performance of transformers and their efficiency. The tasks we choose also cover different data modalities - text and image - and a variety of model sizes. Our experimental results confirm that momentum and adaptive momentum transformers outperform many existing efficient transformers, including linear transformers and reformers, in accuracy and converge faster. Furthermore, adaptive momentum transformer improves over momentum transformer without the need of searching for momentum hyperparameter. 

\subsubsection{Copy task}

We train momentum transformers and baseline models on a synthetic copy task to analyze their convergence speed. In this task, the model has to duplicate a sequence of symbols. 
Each training and test sample has the form $0w0w$ where $w$ is a sequence of symbols collected from the set $\{1,\dots,N\}$. 

In our experiments, we follow the same experimental setting as that used in \cite{katharopoulos2020transformers}. In particular, we use a sequence of maximum length 128 with 10 different symbols separated by a separator symbol. The baseline architecture for all methods is a 4-layer transformer with 8 attention heads and $D=32$. The models are trained with the RAdam optimizer 
using a batch size of 64 and a learning rate of $10^{-3}$ which is reduced to $10^{-4}$ after 3000 iterations. Figure~\ref{fig:copy-convergence-analysis} shows the training loss and the test accuracy over epochs and over GPU time. Both the momentum and the adaptive momentum transformers converge much faster and achieve better training loss than the linear transformer. Notice that while the standard 
transformer converges the fastest, it has quadratic complexity. Adaptive momentum transformer has similar performance as the momentum transformer without the need of tuning for the momentum value. 

\subsubsection{Image generation}
Transformers have shown great promise in autoregressive generation applications \cite{radford2019language,child2019generating}, such as autoregressive image generation \cite{ramesh2020dalle}. However, the training and sampling procedure using transformers are quite slow for these tasks due to the quadratic computational time complexity and the memory scaling with respect to the sequence length. In this section, we train our momentum-based transformers and the baselines with causal masking to predict images pixel by pixel and compare their performance. In particular, we demonstrate that, like linear transformers, both momentum and adaptive momentum transformers are able to generate images much faster than the standard softmax transformer. Furthermore, we show that momentum-based transformers converge much faster than linear transformers while achieving better bits per dimension (bits/dim). 
Momentum and adaptive momentum transformers also generate images with constant memory per image like linear transformers.

\begin{figure}[t!]
\centering
\includegraphics[width=0.7\linewidth]{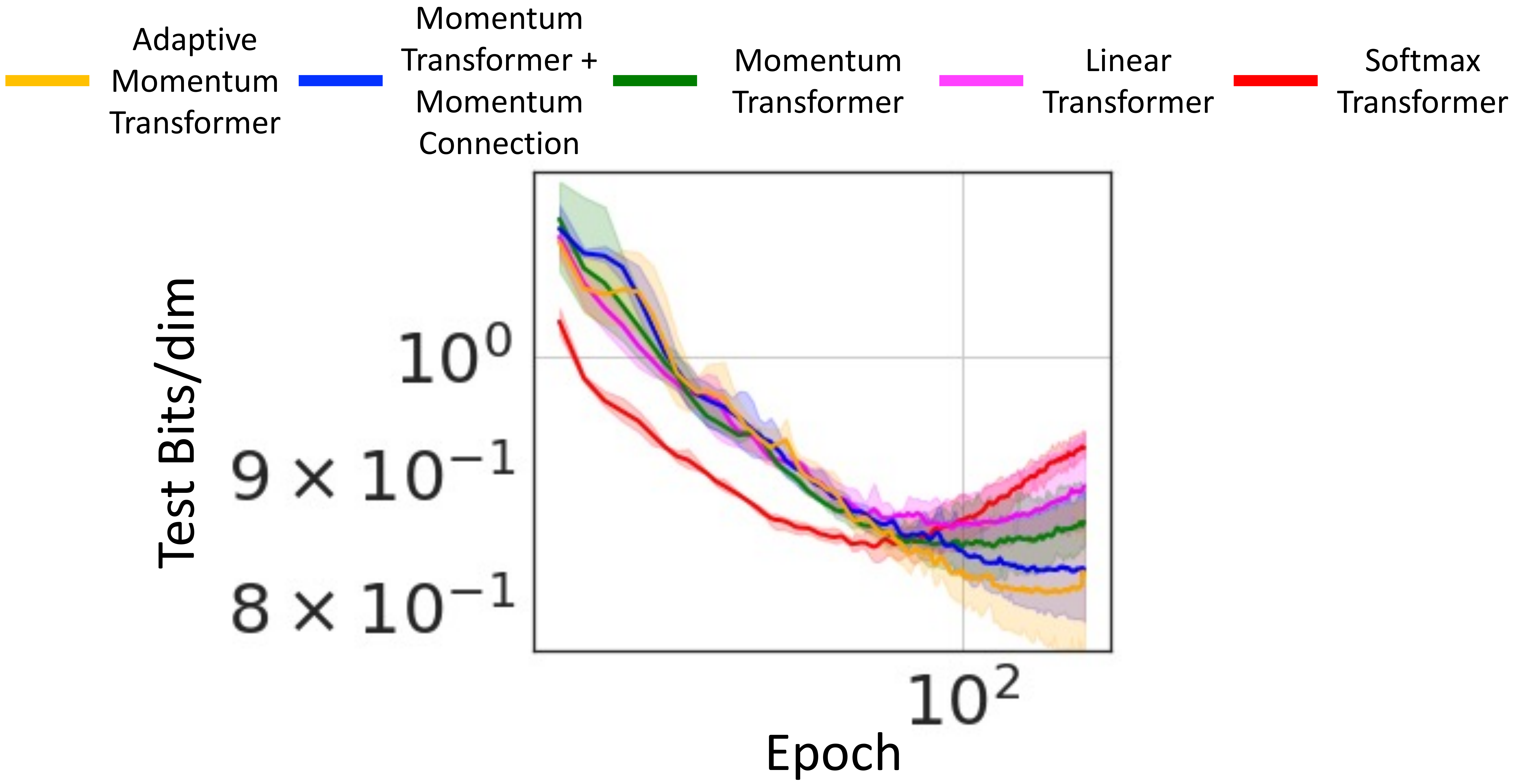}
\vskip -0.2cm
\caption{Momentum transformers outperform linear transformers on the MNIST image generation task. Adaptive momentum transformer achieves the best test bits/dim.}
\label{fig:mnist-convergence-analysis}
\end{figure}

\bgroup
\renewcommand{\arraystretch}{1.0}
\begin{table*}[t!]
    \begin{center}
    \begin{tabular}{lcrl}
        Method & Bits/dim & \multicolumn{2}{c}{Images/sec} \\
        \hline
        Standard softmax transformer & 0.84 & 0.45 & (1$\times$) \\
        Linear transformer & 0.85 & 142.8 & (317$\times$)  \\
        \hline
        Momentum transformer & 0.84 & 139.7 & (310$\times$) \\
        Momentum transformer + momentum connection & 0.82 & 135.5 & (301$\times$) \\
        Adaptive momentum transformer & 0.80 & 134.9 & (300$\times$)  \\
    \end{tabular}
    \end{center}\vspace{-0.3cm}
    \caption{Momentum transformers achieve better test bits/dim than both softmax and linear transformers on MNIST generation. 
    }
    \label{tab:mnist2}
\end{table*}
\egroup

\paragraph{MNIST.}
We first examine our momentum-based transformers on the MNIST image generation task. 
For all methods, we train a 8-layer transformer with 8 attention heads and the embedding size of 256, which corresponds to 32 dimensions per head. The feedforward dimensions are 4 times larger than the embedding size. A mixture of 10 logistics is used to model the output as in \cite{salimans2017pixelcnn}. For training, we use the RAdam optimizer with a learning rate of $10^{-4}$ and train all models for 250 epochs except for the adaptive momentum transformer. 

We report the bits/dim and image generation throughput in Table~\ref{tab:mnist2}. Compared to the linear transformer, all momentum-based transformers not only attain better bits/dim but also have comparable image generation throughput, justifying the linear complexity of our models. In addition, we demonstrate that the adaptive momentum transformer converges much faster than the baseline models in Figure~\ref{fig:mnist-convergence-analysis}. Momentum-based transformers even outperform softmax transformers in this task.

\paragraph{CIFAR10.}
Next, we investigate the advantages of our momentum-based transformers when the sequence length and the number of layers in the model increase. We consider the CIFAR-10 image generation task, in which we train 16-layer transformers to generate CIFAR-10 images. The configuration for each layer is the same as in the MNIST experiment. For the linear transformer and our momentum-based transformer, we use a batch size of 4 while using a batch size of 1 for the standard softmax transformer due to the memory limit of the largest GPU available to us, i.e., NVIDIA V100. This is similar to the setting in \cite{katharopoulos2020transformers}. Like in the MNIST image generation task, our momentum-based transformers outperform the linear transformer in terms of bits/dim while maintaining comparable image generation throughput. This is a very expensive task, limiting us to perform a thorough hyperparameter search; we believe better results can be obtained with a more thorough hyperparameter search.

\bgroup
\renewcommand{\arraystretch}{1.0}
\begin{table*}[t!]
    \begin{center}
    \begin{tabular}{lcrl}
        Method & Bits/dim & \multicolumn{2}{c}{Images/sec} \\
        \hline
        Standard softmax transformer & 3.20 & 0.004 & (1$\times$) \\
        Linear transformer & 3.44 & 17.85 & (4462$\times$)  \\
        \hline
        Momentum transformer & 3.43 & 17.52 & (4380$\times$) \\
        Momentum transformer + momentum connection & 3.41 & 17.11 & (4277$\times$) \\
        Adaptive momentum transformer & 3.38 & 17.07 & (4267$\times$)  \\
    \end{tabular}
    \end{center}\vspace{-0.4cm}
    \caption{
    Momentum-based transformers achieve better test bits/dim than 
    linear transformer on CIFAR10 image generation task.}
    \label{tab:cifar10}
\end{table*}
\egroup

\begin{table*}[!t]
\centering
\resizebox{1.0\linewidth}{!}{
\begin{tabular}{c|c|c|c|c|c|c}
\hline
Model          & ListOps (2K) & Text (4K) & Retrieval (4K) & Image (1K) &  Pathfinder (1K) & Avg \\ \hline
Softmax~\cite{vaswani2017attention} & 37.10 (37.10) & 64.17 (65.02) & 80.71 (79.35) & 39.06 (38.20) & 72.48 (74.16) & \bf 58.70 (58.77) \\
\hline
Linear~\cite{katharopoulos2020transformers}  & 18.30 & 64.22 & 81.37 & 38.29 & 71.17 & 54.67 \\
\hline
Performer~\cite{choromanski2021rethinking}  & 18.80 & 63.81 & 78.62 & 37.07 & 69.87 & 53.63 \\
\hline
Reformer~\cite{Kitaev2020Reformer}  & 19.05 & 64.88 & 78.64 & 43.29 & 69.36 & 55.04 \\
\hline
Linformer~\cite{wang2020linformer}  & \bf 37.25 & 55.91 & 79.37 & 37.84 & 67.60 & 55.59 \\
\hline
Momentum transformer & 19.56 & 64.35 & 81.95 & 39.40 & 73.12 & 55.68  \\
Adaptive momentum transformer & 20.16 & \bf 64.45 & \bf 82.07 & \bf 39.53 & \bf 74.00 & 56.04 \\
\hline
\end{tabular}}
\hspace{0.1em}
\caption{Results on the LRA tasks. We report the test classification accuracy for each task and average accuracy across all tasks. The momentum-based transformers, in particular, the adaptive momentum transformer, outperforms all other transformers except on the ListOps. 
The numbers in the parenthesis are from the paper \cite{xiong2021nystromformer}. Unit: \%. 
}\label{tab:lra}
\end{table*}

\bgroup
\renewcommand{\arraystretch}{1.0}
\begin{table*}[t!]
    \begin{center}
    \begin{tabular}{lcc}
        Method & BLEU Score & Speed (tokens/s) \\
        \hline
        Standard softmax transformer & 24.34 & 5104 \\
        Linear transformer & 21.37 & 1382  \\
        \hline
        Momentum transformer & 22.11 & 1398 \\
        Momentum transformer + momentum connection & 22.14 & 1403 \\
        Adaptive momentum transformer & 22.20 & 1410 \\
    \end{tabular}
    \end{center}\vspace{-0.3cm}
    \caption{BLEU scores and tokens per second from machine translation models trained on IWSLT
show the advantages of our momentum-based transformers.
The number of trainable parameters is almost the same for all models,
up to the small difference introduced by the momentum mechanism in our models. Momentum-based transformers outperform the linear transformer in generation quality in terms of BLEU score and obtain comparable generation efficiency in terms of tokens per second.}
    \label{tab:wikitext103}
\end{table*}
\egroup

\subsubsection{Long-Range Arena}
In this experiment, we evaluate our model on tasks that involve longer sequence lengths in the Long Range Arena (LRA) benchmark~\cite{tay2021long}. We show that the momentum-based transformer outperforms the baseline linear transformer and standard softmax transformer~\cite{vaswani2017attention}, justifying the advantage of our momentum-based transformers in capturing long-term dependency. 

{\bf Datasets and metrics.} We consider all five tasks in the LRA benchmark \cite{tay2021long}, 
including Listops, byte-level IMDb reviews text classification, byte-level document retrieval, CIFAR-10 classification on sequences of pixels, and Pathfinder. These tasks involve long sequences of length $2K$, $4K$, $4K$, $1K$, and $1K$, respectively. We follow the setup/evaluation protocol in \cite{tay2021long} and report test accuracy for each task and the average result across all tasks.  

{\bf Models and training.} All models have 2 layers, 64 embedding dimension, 128 hidden dimension, 2 attention heads. Mean pooling is applied in all models. Also, we use the nonlinear activation $elu(x) + 1$ for the linear transformer. Our implementation uses the public code in ~\cite{xiong2021nystromformer} as a starting point, and we follow their training procedures. The training setting and additional baseline model details are provided in the configuration file used in~\cite{xiong2021nystromformer}. 

{\bf Results.} We summarize our results in Table~\ref{tab:lra}. Both momentum-based transformers 
outperform linear transformers in all tasks and yield better accuracy than the standard softmax transformer in most tasks except the Listops. The adaptive momentum transformer performs the best on every task except the LipsOps, far behind the softmax transformer and Linformer.

\subsubsection{Non-autoregressive machine translation}
All of the above experiments are for auto-regressive tasks. In this last experiment, we demonstrate that the benefits of our momentum-based transformers also hold for a non-autoregressive task. We consider a machine translation task on the popular IWSLT' 16 En-De dataset. We follow the setting in~\cite{lee-etal-2018-deterministic}. In particular, we tokenize each sentence using a script from Moses~\cite{koehn-etal-2007-moses} and segment each word into subword units using BPE~\cite{sennrich-etal-2016-neural}. We also use $40K$ tokens from both source and target. Our baseline model is the small transformer-based network in~\cite{lee-etal-2018-deterministic}. This model has 5 layers, and each layer has 2 attention heads. We replace the softmax attention in this network with the linear and momentum-based attention to obtain the linear transformer baseline and the momentum-based transformer models, respectively. 

Table~\ref{tab:wikitext103} reports the results in terms of generation quality, measured by the BLEU score~\cite{Papineni02bleu:a}, and generation efficiency, measured by the number of generated tokens per second. Consistent with other experiments above, our momentum-based transformers obtain better BLEU scores than the linear transformer in this non-autoregressive setting. Furthermore, in terms of generation efficiency, momentum-based models are comparable with the linear transformer and much more efficient than the standard softmax transformer.

\section{Conclusion and Future Work}\label{sec-conclusion}

In this paper, we reviewed how to integrate momentum into neural networks to enhance their theoretical and practical performances. In particular, we showed that momentum improves learning long-term dependencies of RNNs and neural ODEs and significantly reduces their computational costs. Moreover, we showed that momentum can also be used to improve the efficiency and accuracy of transformers. There are numerous directions for future work: 
1) Can we leverage the momentum-augmented neural network component to aid the neural architecture search? 
2) Can we further improve the momentum-integrated architectures by using the numerical ODE insights \cite{haber2017stable}? 
3) Momentum has also been used in designing CNNs \cite{li2018optimization,sander2021momentum}; it is also worth further studying the benefits of momentum for CNNs.

\section*{Acknowledgement}
This material is based on research sponsored by NSF grants DMS-1924935 and DMS-1952339, DOE grant  DE-SC0021142, and ONR grant N00014-18-1-2527 and the ONR MURI grant N00014-20-1-2787.


\end{document}